\documentclass[11pt]{article}

\usepackage{arxiv}

\usepackage{natbib}

\usepackage[utf8]{inputenc} 
\usepackage[T1]{fontenc}    
\usepackage{hyperref}       
\usepackage{url}            
\usepackage{booktabs}       
\usepackage{amsfonts}       
\usepackage{nicefrac}       
\usepackage{microtype}      
\usepackage{xcolor}         
\usepackage{graphicx}
\usepackage{amsmath,amssymb,amsthm}
\usepackage{tabularx}
\usepackage{multirow}
\usepackage[dvipsnames]{xcolor}
\usepackage{colortbl}
\usepackage{array,colortbl}
\usepackage{rotating}
\usepackage{url}

\DeclareMathOperator{\clr}{CLR}

\DeclareMathOperator{\hCov}{h-Cov}
\DeclareMathOperator{\sCov}{s-Cov}
\DeclareMathOperator{\amp}{Amp}

\theoremstyle{remark}
\newtheorem{remark}{Remark}

\newtheorem{proposition}{Proposition}

\title{Geometry-Aware Bayesian Quantification via Compositional Data Analysis}

\author{
 Alejandro Moreo \\
  Istituto di Scienza e Tecnologie dell'Informazione\\
  Consiglio Nazionale delle Ricerche\\
  Pisa, Italy \\
  \texttt{alejandro.moreo@isti.cnr.it} \\
   \And
 Pablo González, Juan José del Coz \\
  Artificial Intelligence Center\\
  University of Oviedo\\
  Asturias, Spain \\
  \texttt{\{gonzalezgpablo,juanjo\}@uniovi.es} \\
}

\begin{document}

\maketitle

\begin{abstract}
  Accurately estimating the unknown target label distribution is the critical first step for adapting to label shift. This task, widely known as quantification or class prevalence estimation, has recently seen significant advances through continuous KDE-based methods which model the density of multiclass classifier posteriors. Posterior vectors might be regarded as compositional data, since they lie on the probability simplex.
  However, existing KDE-based quantifiers typically rely on Euclidean Gaussian kernels, which ignore simplex geometry and incorrectly assign probability mass outside its boundaries.
  We introduce a geometry-aware KDE model for multiclass quantification based on log-ratio representations and Aitchison geometry, together with a shrinkage regularization that improves robustness near the simplex boundary. Combined with a maximum-likelihood interpretation of KDE-based quantification, we derive both point-estimation and Bayesian inference procedures for class prevalences.
  Experiments on 42 datasets across tabular, text, and image domains show that the proposed method is competitive with state-of-the-art quantifiers, often improving over standard KDE-based baselines, while also yielding strong results among Bayesian quantification methods.
  
\end{abstract}

\section{Introduction}

Machine learning models deployed in the wild frequently encounter label shift \citep{Storkey:2009lp}. Adapting to this shift requires accurately estimating the unknown target label distribution. This problem, called quantification \citep{Forman:2005fk}
or class prevalence estimation \citep{Lipton:2018fj}, is a critical standalone task in domains where aggregate population trends matter more than individual predictions. Furthermore, it serves as the essential foundation for label shift adaptation, as accurate prevalence estimates are strictly required to compute importance weights for classifier retraining \citep{Alexandari:2020dn}.

Many prominent methods, such as Maximum Likelihood for Label Shift (MLLS) \citep{Saerens:2002uq} and Black-Box Shift Estimation (BBSE) \citep{Lipton:2018fj}, operate directly on the posterior probabilities produced by a base classifier. In multiclass quantification, modeling the continuous distribution of these posteriors using density-based methods like KDEy \citep{kdey2025} has emerged as a powerful alternative that naturally preserves the dependence structure across classes. Consequently, KDEy is increasingly adopted as a robust solution in a variety of application domains, such as fairness monitoring in IR \cite{jaenich2026quantifying}, graph-related applications \cite{damke2025distribution,micheli2025efficient}, classifier accuracy prediction \cite{volpi2025leap}, and medical imaging for healthcare \cite{godau2025navigating}.

However, modeling classifier posteriors via continuous density estimation introduces a fundamental geometric mismatch. Posterior probability vectors lie on the probability simplex and therefore constitute compositional data. Existing KDE-based quantifiers rely on Gaussian kernels in Euclidean space, which ignore this relative structure and inevitably assign probability mass outside the simplex boundaries. This mismatch becomes particularly problematic near the boundary of the simplex, where posterior probabilities heavily concentrate in many real-world settings.

In this work, we revisit multiclass quantification from the perspective of compositional data analysis (CoDA) \citep{aitchison1982statistical} to resolve this bottleneck. We introduce a geometry-aware kernel based on log-ratio transformations, enabling continuous density estimation in Aitchison geometry. While this representation is more consistent with the underlying simplex structure, naive log-ratio modeling can be unstable near the boundary. To address this issue, we propose a shrinkage-based regularization that improves robustness. Building on this representation, we formulate quantification as a mixture density estimation problem that naturally admits both point estimation and Bayesian inference over class prevalences.
Across 42 datasets spanning tabular, textual, and image domains, the proposed method achieves performance that is competitive with state-of-the-art approaches. It frequently improves over standard KDE-based quantifiers, while providing a principled geometric formulation of continuous density estimation on the simplex.

\section{Related Work}

The literature on label shift adaptation and quantification primarily focuses on the setting of label shift (aka prior probability shift) \citep{Storkey:2009lp}. Existing approaches include classify-and-count variants (e.g., BBSE \citep{Lipton:2018fj}, also known in the quantification literature as Adjusted Classify and Count (ACC) \citep{Forman:2005fk} or the confusion matrix approach in \citep{Saerens:2002uq}), distribution matching methods \cite{Gonzalez-Castro:2013fk,Castano:2022oq,Dussap:2023gr},
and maximum likelihood approaches via expectation maximization (e.g., MLLS, often referred to as EMQ in quantification literature) \citep{Saerens:2002uq,Azizzadenesheli:2019qf,Alexandari:2020dn}.
Among these, KDE-based methods such as KDEy have recently emerged as strong performers in the multiclass setting \citep{kdey2025}.

Bootstrap-based approaches have been explored for deriving uncertainty estimates in quantification \citep{Hopkins:2010fk,Daughton:2019ca,tasche2019confidence}. 
Analytical approaches to confidence interval estimation have also been developed for specific estimators \citep{Vaz:2019eu,denham2021gain}. 
More recently, Bayesian quantification methods have been proposed for standard quantifiers \citep{fiksel2022generalized}; examples include BayesianACC\footnote{Introduced as BayesianCC by Ziegler and Czyż \citep{bayesianCC.2024}. We use the term Bayesian\underline{A}CC to maintain consistent terminology and reflect its mathematical equivalence to ACC (BBSE) rather than to CC (which often denotes plain ``Classify and Count'').} for BBSE/ACC,  MAPLS \citep{ye2024label,hu2025bayesian} for MLLS/EMQ, and Precise Quantifier (PQ) \citep{igiraneza2025estimating} for HDy \cite{Gonzalez-Castro:2013fk}.
However, the multinomial model over binned probabilities used by PQ (an analogue of HDy’s histogram representation) inherits the limitations of univariate, discretized density models and is thus restricted to binary problems. Such limitations are overcome by continuous multivariate density models based on KDE  \citep{kdey2025}. Nevertheless, no prior work has developed a Bayesian formulation for KDE-based quantifiers; we fill this gap. 

We also note that KDE models based on Gaussian kernels may be inappropriate for compositional data.
Simplex-supported alternatives to Euclidean KDE do exist in the compositional-data literature, most notably Dirichlet-kernel and logistic-normal-kernel estimators \citep{aitchison1985kernel}. Subsequent work has further developed the log-ratio route, in particular through Gaussian kernels defined in isometric log-ratio (ILR) coordinates together with modern bandwidth-selection methods, suggesting that this line is especially natural when the goal is to respect the geometry of compositional data \citep{chacon2011gaussian,martin2006updating}. At the same time, Dirichlet kernels remain theoretically attractive because they are natively supported on the simplex and exhibit favorable boundary behavior \citep{ouimet2022asymptotic}. 
In practice, applying standard compositional kernels directly to classifier outputs is difficult, as confident predictions cause (presumably well-calibrated) posteriors to concentrate heavily at the simplex boundaries. Furthermore, unlike standard Euclidean KDEs, these simplex-aware estimators lack drop-in implementations for ML pipelines. By introducing a shrinkage-regularized log-ratio representation, we bridge these gaps, providing a mathematically principled yet computationally practical solution.




\section{Problem Setup}
\label{sec:setup}

Let 
\(\mathcal{X}\) 
denote the input space and \(\mathcal{Y} = \{1, \dots, K\}\) the set of class labels. We assume a standard prior probability shift setting \citep{Storkey:2009lp}, where training and test data are drawn from distributions \(P_s\) and \(P_t\), respectively, such that
\(P_s(Y) \neq P_t(Y)\), and \(P_s(X \mid Y) = P_t(X \mid Y)\).

We assume access to a probabilistic classifier \(s: \mathcal{X} \to \Delta^{K-1}\), which maps each input \(x \in \mathcal{X}\) to a posterior probability vector \(p = s(x) = (p_1, \dots, p_K)\),
where \(p_k = \hat{P}_s(Y = k \mid x)\). Here, \(\Delta^{K-1}\) denotes the probability simplex \(\Delta^{K-1}=\{ p \in \mathbb{R}^K : p_k \ge 0,\ \sum_{k=1}^K p_k = 1 \}\).

Since \(p = s(X)\) is a deterministic, measurable transformation of \(X\), the class-conditional distributions of posterior probabilities are also invariant across domains, i.e., \(P_s(p \mid Y) = P_t(p \mid Y)\); see Lemma 1 in \cite{Lipton:2018fj}.
This allows us to perform quantification directly in the space of posterior probabilities.

Let \(\mathcal{U} = \{x_1, \dots, x_n\}\) denote an unlabeled sample drawn from \(P_t\), and let \(\mathcal{P} = \{p^{(1)}, \dots, p^{(n)}\}\), with \(p^{(i)} = s(x_i)\), be the corresponding set of posterior probability vectors.
The goal of quantification is to estimate the class prevalence vector
\(\pi = (\pi_1, \dots, \pi_K) \in \Delta^{K-1}\), where \(\pi_k = P_t(Y = k)\).

\begin{remark}
Let \(f_1,\ldots,f_K\) denote class-conditional densities, and consider the mixture model
\(
m_\pi(p)=\sum_{k=1}^K \pi_k f_k(p)\), with
\(\pi \in \Delta^{K-1}.\)
Let \(\mathcal{P}=\{p^{(1)},\ldots,p^{(n)}\}\) be the observed posterior probability vectors in the target sample. 
At the population level, quantification can be viewed as a distribution matching problem in which the target density \(q\) is approximated by the mixture \(m_\pi\). When matching is formalized via the Kullback--Leibler divergence, one obtains
\[
\arg\min_{\pi\in\Delta^{K-1}} D_{\mathrm{KL}}(q \,\|\, m_\pi)
=
\arg\max_{\pi\in\Delta^{K-1}} \mathbb{E}_{q}\!\left[\log m_\pi(p)\right].
\]
Replacing the expectation by its empirical counterpart yields the maximum likelihood estimator \citep{Garg:2020jt}
\begin{equation}
\label{eq:maxlike}
\hat{\pi}
=
\arg\max_{\pi\in\Delta^{K-1}}
\sum_{i=1}^n \log m_\pi\!\left(p^{(i)}\right).
\end{equation}
\end{remark}

Following \citep{kdey2025}, we model the densities \(f_k\) using class-conditional kernel density estimators.

\section{Geometric Mismatch in KDE-Based Quantification}
\label{sec:geometry}

The KDE-based formulation introduced above operates on classifier posterior probabilities \(p \in \Delta^{K-1}\). This is already a substantial simplification with respect to working directly in the original input space \(\mathcal{X}\), and has proven highly effective in multiclass quantification \citep{kdey2025}. However, posterior probability vectors are not ordinary Euclidean data: they are \emph{compositions}, i.e., non-negative vectors constrained to sum to one. As such, they live on the probability simplex and are more naturally studied within the framework of compositional data analysis (CoDA) \citep{aitchison1982statistical}.

The constraint by which posterior vectors \(p = (p_1,\dots,p_K)\) lie in the simplex
induces a non-Euclidean geometry: the components of \(p\) are not independent coordinates, and meaningful comparisons are inherently relative. In particular, compositions are more naturally compared through log-ratios than through Euclidean differences in the ambient space \(\mathbb{R}^K\).
%
Standard KDE-based quantifiers, including KDEy, estimate class-conditional densities of posterior vectors using Gaussian kernels in Euclidean space. This introduces a geometric mismatch for two related reasons. First, Gaussian kernels on the ambient space \(\mathbb{R}^K\) assign non-zero probability mass outside the simplex. Second, Euclidean distance does not respect the relative structure of compositions, and may therefore distort similarity relationships between posterior vectors.
This mismatch is not merely aesthetic. Since the quantifier ultimately estimates prevalences by fitting a mixture model to the empirical distribution of posterior vectors, misspecifying the geometry of the data may lead to distorted density estimates and, consequently, biased or unstable prevalence predictions.

\begin{figure}[bh]
\centering
\includegraphics[width=\textwidth]{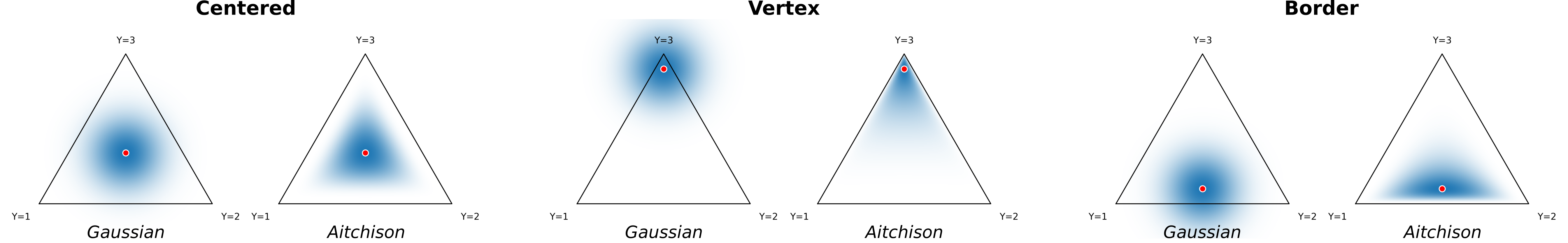}
\caption{Comparison between a standard Gaussian kernel in Euclidean space and a geometry-aware kernel induced through a log-ratio representation. Near the boundary of the simplex, the Gaussian kernel assigns probability mass outside the support, whereas the geometry-aware alternative better respects the simplex structure.}
\label{fig:kernel-mismatch}
\end{figure}

The geometric mismatch is particularly severe near the boundary of the simplex. In many realistic classification problems, posterior vectors tend to concentrate near vertices or edges, corresponding to confident predictions or highly imbalanced uncertainty across classes. This is especially true for neural classifiers, which are known to be overconfident, in which case posterior calibration is required \citep{guo2017calibration}.
In such regions, Euclidean Gaussian kernels place a non-negligible fraction of their mass outside the simplex, precisely where the support constraint is most restrictive.
At the same time, the natural tools from CoDA (namely, log-ratio transformations such as the centered log-ratio (CLR) and isometric log-ratio (ILR) maps) become increasingly sensitive near the boundary, where some coordinates approach zero. This creates a tension: while log-ratio geometry is more appropriate for simplex-valued data, a naive application of CLR/ILR may itself become unstable in the regime where geometric mismatch matters most. Figure~\ref{fig:kernel-mismatch} illustrates this effect.


These observations suggest that density-based quantification should be performed in a geometry-aware representation of the simplex, rather than in the ambient Euclidean space. In the next section, we develop such a representation using log-ratio coordinates. We further show that, although CLR/ILR provide the correct geometric framework, additional regularization is needed to obtain stable density estimates near the simplex boundary.

\section{Shrinkage-Regularized Geometry-Aware KDE}
\label{sec:shrinkage}

Let \(p \in \Delta^{K-1,\circ}\) be a posterior probability vector. We consider log-ratio transformations
\(\phi : \Delta^{K-1,\circ} \to \mathbb{R}^{K-1}\) or to the zero-sum subspace of \(\mathbb{R}^K\), in particular the isometric log-ratio (ILR) and the centered log-ratio (CLR)  maps from CoDA, respectively. These transformations provide Euclidean coordinates consistent with Aitchison geometry, thereby enabling density estimation in a representation that respects the relative structure of compositions.
In practice, both CLR and ILR lead to equivalent distance-based behavior, since ILR is an isometric reparameterization of CLR restricted to the zero-sum subspace. We therefore use CLR in our implementation for computational convenience, while the same construction applies to ILR. From this point onward, we let \(\phi\) denote the CLR transformation, defined as
\[
\phi(p)
=
\left(
\log \frac{p_1}{g(p)},
\ldots,
\log \frac{p_K}{g(p)}
\right),
\qquad
g(p) = \left(\prod_{k=1}^K p_k\right)^{1/K},
\]
where \(g(p)\) denotes the geometric mean of the components of \(p\). 

Despite their geometric appeal, log-ratio coordinates are delicate near the boundary of the simplex. Whenever one or more components of \(p\) approach zero, the corresponding log-ratios become large in magnitude, which may induce unstable density estimates and highly distorted local geometries. This is particularly problematic in quantification, where posterior probabilities may legitimately attain zero or near-zero values.
To mitigate this issue, we introduce a shrinkage transformation that moves each composition away from the boundary and towards the simplex barycenter \(u = (\tfrac{1}{K}, \dots, \tfrac{1}{K})\):
\[
T_\lambda(p) = (1-\lambda)p + \lambda u,
\qquad
\lambda \in [0,1).
\]
When \(\lambda = 0\), the original composition is recovered. As \(\lambda\) increases, \(T_\lambda(p)\) contracts the simplex towards its center, ensuring that all transformed compositions remain bounded away from zero.

Shrinkage has two complementary effects. First, it prevents numerical instability in the log-ratio map. Second, it regularizes the geometry of the transformed space near the simplex boundary, where the nonlinear effects of \(\phi\) are most severe. Importantly, this does not simply recover the original Euclidean KDE. Rather, shrinkage moves the problem into a regime where the log-ratio representation becomes locally well behaved, while preserving the compositional structure of the data. The goal is not to discard simplex geometry, but to regularize its most unstable regime.

Because shrinkage contracts the effective simplex towards the barycenter, the bandwidth must be adjusted accordingly to preserve a comparable amount of smoothing across values of \(\lambda\). Under the affine shrinkage map \(T_\lambda\), distances in the simplex hyperplane are uniformly rescaled by the factor \(1-\lambda\). Therefore, if \(h\) denotes the bandwidth before shrinkage, the natural effective bandwidth becomes
\(
h_{\mathrm{eff}} = (1-\lambda)h
\). 
More generally, since the intrinsic dimension of the simplex is \(K-1\), the corresponding density height rescales by the Jacobian factor \((1-\lambda)^{-(K-1)}\). A full derivation is provided in Appendix~\ref{app:bandwidth-scaling}.

\begin{figure}[tb]
    \centering
    \includegraphics[width=\linewidth]{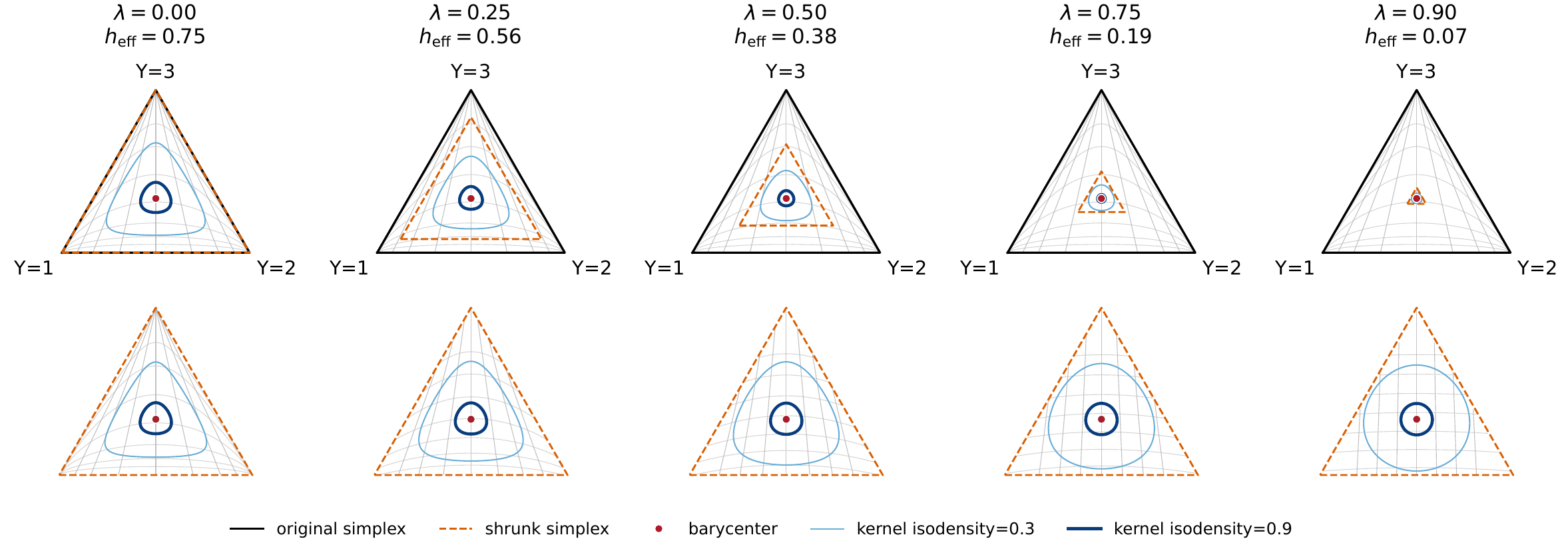}
    \caption{Effect of shrinkage on a geometry-aware kernel centered in the simplex for $h=0.75$. As \(\lambda\) increases, the effective simplex contracts towards the barycenter. The first row show the shrinkage from the point of view of the original simplex, while the second row is a zoom in the effective shrunk region. By adjusting $h_\mathrm{eff}$, the induced kernel becomes locally Euclidean for high values of $\lambda$.
    }
    \label{fig:shrinkage}
\end{figure}

Figure~\ref{fig:shrinkage} illustrates the combined effect of shrinkage and bandwidth rescaling on an Aitchison-centered kernel. 
The regularizing effect of shrinkage can be formalized by studying the behavior of \(\phi\) near the barycenter.

\begin{proposition}
\label{prop:clr-linearization}
Let \(u = (\tfrac{1}{K},\dots,\tfrac{1}{K})\), let \(\phi\) denote the CLR map, and write \(T_\lambda(p) = u + (1-\lambda)(p-u)\). Then, for any fixed \(p \in \Delta^{K-1,\circ}\),
\[
\phi(T_\lambda(p))
=
K(1-\lambda)(p-u)
+
O((1-\lambda)^2)
\qquad \text{as } \lambda \to 1.
\]
\end{proposition}
This shows that shrinkage attenuates the nonlinear effects of the CLR transformation: the leading term is linear in \(1-\lambda\), while the nonlinear remainder is of second order. Hence, shrinkage regularizes the geometry precisely in the regime where log-ratio representations would otherwise be most unstable. A full proof is provided in Appendix~\ref{app:clr-linearization}.

This local linearization also has an estimator-level implication. When combined with the bandwidth scaling \(h_{\mathrm{eff}}=(1-\lambda)h\), the induced Gaussian kernel in transformed coordinates is locally well approximated by a Euclidean Gaussian kernel on the simplex hyperplane, up to a multiplicative constant and higher-order local terms. In this sense, shrinkage places the geometry-aware estimator in a locally Euclidean regime near the barycenter, without collapsing it into an ordinary Euclidean KDE. A heuristic derivation is provided in Appendix~\ref{app:estimator}.

\subsection{Geometry-aware KDE}

Given the transformed representation
\(
z = \phi(T_\lambda(p)),
\)
we estimate class-conditional densities in the transformed space by standard kernel density estimation. For each class \(k \in \{1,\dots,K\}\), let
\(
\mathcal{Z}_k = \{ z_i : y_i = k \}
\)
denote the transformed training posteriors for class \(k\). We then define
\begin{equation}
    \label{eq:kde}
    f^h_k(z)
=
\frac{1}{|\mathcal{Z}_k|}
\sum_{z' \in \mathcal{Z}_k}
\mathcal{K}_{h_{\mathrm{eff}}}(z-z'),
\end{equation}
where \(\mathcal{K}_{h_{\mathrm{eff}}}\) is a Gaussian kernel with effective bandwidth \(h_{\mathrm{eff}}=(1-\lambda)h\) (refer to Appendix~\ref{app:bandwidth-scaling} for the mathematical derivation).
This yields a shrinkage-regularized, geometry-aware instantiation of the class-conditional densities. 


We now turn to showing that this model naturally supports both Bayesian posterior inference and point estimation over the prevalence vector \(\pi\).
Let \(z = \phi(T_\lambda(p))\), where \(\phi\) denotes the fixed CLR transformation and \(T_\lambda\) the fixed shrinkage map. For each class \(k \in \{1,\dots,K\}\), let \(f^h_k(z)\) denote the corresponding class-conditional density estimate in the transformed space (Equation~\ref{eq:kde}). We redefine the prevalence-dependent mixture:
\[
m_\pi(z) = \sum_{k=1}^K \pi_k f^h_k(z),
\qquad
\pi \in \Delta^{K-1}.
\]
%
%
Although density estimation is carried out in transformed coordinates, the observed posterior vectors \(p^{(1)},\dots,p^{(n)}\) live on the simplex. Therefore, the exact density induced on the simplex must account for the change of variables. For any \(p \in \Delta^{K-1,\circ}\), the corresponding exact density is
\[
q_\pi(p)
=
J_{\phi \circ T_\lambda}(p)\, m_\pi(\phi(T_\lambda(p))),
\]
where \(J_{\phi \circ T_\lambda}(p)\) denotes the Jacobian term induced by the transformation \(\phi \circ T_\lambda\).
Since the Jacobian does not depend on \(\pi\), the exact likelihood for a test sample
\(\mathcal{P}=\{p^{(1)},\dots,p^{(n)}\}\) satisfies
\[
\mathcal{L}_{\mathrm{exact}}(\pi; \mathcal{P})
=
\prod_{i=1}^n
J_{\phi \circ T_\lambda}(p^{(i)})
\left(
\sum_{k=1}^K \pi_k f^h_k(z^{(i)})
\right)
\propto
\prod_{i=1}^n
\left(
\sum_{k=1}^K \pi_k f^h_k(z^{(i)})
\right),
\]
where \(z^{(i)}=\phi(T_\lambda(p^{(i)}))\). Thus, the Jacobian term does not need to be evaluated explicitly. 

\subsection{Bayesian posterior inference}

The previous result immediately yields a Bayesian formulation of geometry-aware KDE quantification. Given a prior \(p(\pi)\) on the prevalence vector, the posterior distribution is
\begin{equation}
\label{eq:bayes}
p(\pi \mid \mathcal{P})
\propto
p(\pi)\, \mathcal{L}_{\mathrm{exact}}(\pi; \mathcal{P})
\propto
p(\pi)
\prod_{i=1}^n
\left(
\sum_{k=1}^K \pi_k f^h_k(z^{(i)})
\right).
\end{equation}
Thus, although the exact simplex density formally includes a Jacobian term, posterior inference can be carried out directly in transformed coordinates using the unnormalized posterior above. This is the key observation that makes Bayesian inference practical for KDE-based quantification.
Note that the same cancellation argument applies to point estimation, since multiplicative constants independent of \(\pi\) do not affect the argmax in Equation~\ref{eq:maxlike} under the change of variables.


To control posterior dispersion under misspecification, we consider a tempered posterior of the form
\[
p_T(\pi \mid \mathcal{P})
\propto
p(\pi)\,
\left[
\prod_{i=1}^n
\left(
\sum_{k=1}^K \pi_k f^h_k(z^{(i)})
\right)
\right]^{1/T},
\qquad
T > 0.
\]
The temperature parameter \(T\) leaves the density model unchanged and primarily controls posterior concentration. In practice, \(T\) is selected on a held-out calibration split; see Section~\ref{sec:experiments}.

\section{Experiments}
\label{sec:experiments}

\textbf{Datasets.}
We evaluate on 42 datasets spanning text, tabular, and image domains, with numbers of classes ranging from 2 to 100, used in previous related research. The text suite includes the standard LeQua benchmarks \cite{esuli2022lequa,esuli2024overview} and 11 Twitter sentiment datasets \cite{gao2016classification}; the tabular suite contains 22 UCI ML datasets \cite{ucimlrepo}; the image suite includes \texttt{CIFAR10}/\texttt{CIFAR100} \citep{krizhevsky2009learning}, \texttt{MNIST} \cite{lecun1998mnist}, \texttt{FashionMNIST} \citep{xiao2017fashion}, and \texttt{SVHN} \citep{netzer2011reading}. Full dataset statistics are given in Appendix~\ref{app:datasets}.

\textbf{Base classifiers and baselines.}
For text and tabular data, we use logistic regression (LR) on fixed features (neural embeddings provided by the LeQua creators\footnote{LeQua2022: \url{https://doi.org/10.5281/zenodo.5734464} LeQua2024: \url{https://doi.org/10.5281/zenodo.10654474}}, and TF-IDF weighted vectors for Twitter data available online\footnote{\url{https://zenodo.org/record/4255764/files/tweet_sentiment_quantification_snam.zip}}) since LR is the standard choice in the quantification literature as it natively provides well calibrated posterior probabilities \cite{schumacher2025comparative}. For image data, we use ResNet18 (CIFAR/SVHN) or a lightweight CNN (MNIST/FashionMNIST), with calibrated posteriors obtained via BCTS following \citep{Alexandari:2020dn}. 
We consider a representative set of baselines covering classical point estimators, density-based methods, and available Bayesian quantifiers. Specifically, we compare against a naive ``classify and count'' (CC) \citep{Forman:2005fk}, its probabilistic counterpart (PCC) \citep{Forman:2005fk}, BBSE/ACC \citep{Lipton:2018fj}, MLLS/EMQ \citep{Saerens:2002uq,Alexandari:2020dn}, and KDEy-ML (hereafter KDEy(Gau), since it uses Gaussian kernels) \citep{kdey2025}, together with our proposed geometry-aware point estimator KDEy(Ait-$\lambda$). For uncertainty quantification, we include Bayes-ACC (the Bayesian adaptation of BBSE/ACC by \citep{bayesianCC.2024}), Bayes-EMQ (aka MAPLS \citep{ye2024label})\footnote{We adapted the original code \url{https://github.com/ChangkunYe/MAPLS/blob/main/label_shift/mapls.py}}, Bayes-KDEy(Gau) (our Bayesian adaptation of KDEy), and Bayes-KDEy(Ait-$\lambda$) (our proposed Bayesian geometry-aware KDE), with Bootstrap variants of CC (Boots-CC) and PCC (Boots-PCC) as naive baselines. We exclude Precise Quantifier \cite{igiraneza2025estimating}, since it is binary-only. To the best of our knowledge, no additional Bayesian implementations are available for multiclass quantification. We use the implementations of baseline methods available in QuaPy \cite{moreo2021quapy}. 
The code that reproduces our experiments is available online\footnote{\url{https://github.com/AlexMoreo/BayesKDEy}}. 

\textbf{Protocol.}
Evaluation follows a Dirichlet-sampling protocol \citep{Alexandari:2020dn}: for each dataset, we extract 500 test bags satisfying prevalence vectors sampled from a Dirichlet($\alpha$) on the simplex, thereby simulating label shift. 
Following \citep{Alexandari:2020dn}, we use $\alpha=1$ to generate a uniform coverage of the simplex,\footnote{This is equivalent to the so-called Artificial Prevalence Protocol (APP) in the quantification literature \cite{Forman:2005fk}.} $\alpha=0.1$ to generate prevalence samples from the border, and $\alpha=10$ to generate samples concentrated on the center of the simplex.
Bag size depends on the dataset and ranges from 100 (Twitter data) to 2000 (\texttt{CIFAR100}); full details are given in Appendix~\ref{app:datasets}.

\textbf{Evaluation metrics.} Point-estimation performance is evaluated in terms of absolute error $\mathrm{AE}(\pi,\hat{\pi})=K^{-1}||\pi-\hat{\pi}||_1$,
and the squared weight-ratio error 
$W(\pi,\hat{\pi})=K^{-1}||w-\hat{w}||^2_2$, with $w_i=\pi_i/\pi^{\mathrm{train}}_i$, and $\hat{w}_i=
\hat{\pi}_i/\pi^{\mathrm{train}}_i$. 
%
AE is arguably the most directly interpretable measure of prevalence estimation error \citep{sebastiani2020evaluation}, while \(W\) is particularly informative when prevalence estimates are used to construct class-wise importance weights for domain adaptation via importance-weighted empirical risk minimization \citep{Lipton:2018fj,Azizzadenesheli:2019qf,Alexandari:2020dn}.
For Bayesian methods, posterior inference is performed with NUTS \cite{hoffman2014no} 
using 500 warmup steps, and point-estimation metrics are computed from the posterior mean of the 1000 sampled prevalence vectors. Unless explicitly stated, we consider no pre-experimental knowledge, therefore employing a uniform prior. 
Uncertainty is evaluated in terms of coverage and amplitude over the empirical classwise confidence intervals \(\mathcal{C}=\{[\ell_k,u_k]\}_{k=1}^K\) constructed at 95\% level. Concretely, we report standard (hard) coverage ($\hCov$) as the proportion of experiments in which the true prevalence values are simultaneously contained in their respective Bonferroni-corrected confidence intervals; soft coverage ($\sCov$) as the proportion of prevalence values included in their corresponding (uncorrected) confidence intervals; and amplitude ($\amp$) as the mean percentage of the simplex volume (as estimated via Monte Carlo sampling) covered by the hyper-rectangle induced by the classwise intervals. 
For coverage ($\hCov$, $\sCov$), we report the gap with respect to the nominal value \(95\%\), so that lower is better for all the metrics we report.

\textbf{Model selection and calibration.}
Classifier hyperparameters are treated as quantifier hyperparameters and selected by minimizing absolute error (AE) on 100 bags generated using the Dirichlet($\alpha=1$) sampling protocol from a held-out stratified validation split (40\% of the training data). Quantifier-specific hyperparameters include kernel bandwidths and the shrinkage parameter \(\lambda\). For Bayesian methods, posterior temperature is then \textit{post hoc} calibrated on the same validation bags by minimizing the mean classwise Winkler's interval score \citep{gneiting2007strictly}. 
Full grids are reported in Appendix~\ref{app:hyperparameters}.

\subsection{Results}
\textbf{Point estimation.}
Table~\ref{tab:main-point} summarizes point-estimation performance using average ranks for AE and \(W\), reported by modality and overall, for the Dirichlet($\alpha=1$) sampling protocol. Average ranks are preferable here because raw error scales differ substantially across datasets; full tables with raw values are given in Appendix~\ref{app:fulltables:point}.
The geometry-aware KDE is fully competitive in the point-estimation regime and typically improves over Gaussian KDE. Overall, KDEy(Ait-\(\lambda\)) achieves the best average rank in both AE and \(W\), with especially strong results on tabular and image datasets. 


\begin{table*}[t]
\centering
\caption{Point-estimation summary by modality in terms of average rank (lower is better).}
\label{tab:main-point}
\resizebox{.65\textwidth}{!}{%
\begin{tabular}{lcccc cccc}
\toprule
& \multicolumn{4}{c}{AE rank} & \multicolumn{4}{c}{\(W\) rank} \\
\cmidrule(lr){2-5}\cmidrule(lr){6-9}
Method & tabular & text & image & avg & tabular & text & image & avg \\
\midrule
CC \citep{Forman:2005fk} & \phantom{0}4.4 & \phantom{0}4.9 & \phantom{0}4.8 & \phantom{0}4.6 & \phantom{0}4.5 & \phantom{0}4.4 & \phantom{0}5.2 & \phantom{0}4.5 \\
PCC \citep{Forman:2005fk} &  \phantom{0}5.5 & \phantom{0}5.9 & \phantom{0}6.0 & \phantom{0}5.7 & \phantom{0}5.4 & \phantom{0}5.6 & \phantom{0}5.8 & \phantom{0}5.5 \\
BBSE/ACC \citep{Lipton:2018fj} & \phantom{0}4.1 & \phantom{0}4.1 & \phantom{0}4.2 & \phantom{0}4.1 & \phantom{0}3.9 & \phantom{0}4.3 & \phantom{0}4.0 & \phantom{0}4.0 \\
MLLS/EMQ \citep{Saerens:2002uq,Alexandari:2020dn} & \phantom{0}2.8 & \phantom{0}2.3 & \phantom{0}2.0 & \phantom{0}2.5 & \phantom{0}2.8 & \textbf{\phantom{0}2.1} & \phantom{0}2.2 & \phantom{0}2.5 \\
KDEy(Gau) \citep{kdey2025} & \phantom{0}2.4 & \textbf{\phantom{0}1.9} & \phantom{0}2.4 & \phantom{0}2.2 & \phantom{0}2.4 & \textbf{\phantom{0}2.1} & \phantom{0}2.4 & \phantom{0}2.3 \\
KDEy(Ait-$\lambda$) \textbf{(ours)} & \textbf{\phantom{0}1.9} & \textbf{\phantom{0}1.9} & \textbf{\phantom{0}1.6} & \textbf{\phantom{0}1.9} & \textbf{\phantom{0}2.1} & \phantom{0}2.5 & \textbf{\phantom{0}1.4} & \textbf{\phantom{0}2.2} \\
\bottomrule
\end{tabular}
}
\end{table*}

\textbf{Uncertainty evaluation.}
Table~\ref{tab:main-bayes-combined} summarizes Bayesian results for uncertainty-aware quantification; full tables with raw values are given in Appendix~\ref{app:fulltables:bayes}. In our experiments, temperature calibration mainly affects uncertainty rather than the posterior mean, and does not change the overall ranking. For this reason, we report AE and \(W\) ranks for the untempered posterior (\(T=1\)), while coverage gaps and amplitude are shown for both \(T=1\) and the temperature-calibrated setting. For coverage, we report the absolute deviation from the nominal \(95\%\) level, so lower values are better.
In the untempered regime, Bayes-KDEy(Ait-\(\lambda\)) achieves the best aggregate AE and \(W\) ranks, closely followed by Bayes-KDEy(Gau). Bootstrap baselines yield very small amplitudes but poor coverage, making them unreliable as uncertainty quantifiers. Bayes-ACC attains the best raw coverage gaps at \(T=1\), but at the cost of substantially larger regions and clearly weaker point-estimation performance. Bayes-EMQ offers a stronger trade-off, yet both KDE-based Bayesian models remain better overall in aggregate ranking.
As expected, temperature calibration reduces coverage gaps at the cost of amplitude for most methods (but Bayes-ACC). After calibration, the KDE-based Bayesian methods offer the best coverage, while Bayes-EMQ has the smallest amplitude.

\begin{table*}[t]
\centering
\caption{Bayesian uncertainty-aware quantification summary. AE and \(W\) ranks are reported for the \(T=1\) setting only. Coverage gaps and amplitudes are shown as mean \(\pm\) standard deviation across datasets, both for \(T=1\) and for temperature-calibrated models.}
\label{tab:main-bayes-combined}
\resizebox{\textwidth}{!}{%
\begin{tabular}{lccccc ccc}
\toprule
& & & \multicolumn{3}{c}{Temperature \(=1\)} & \multicolumn{3}{c}{Calibrated} \\
\cmidrule(lr){4-6}\cmidrule(lr){7-9}
Method & AE rank & \(W\) rank & $\hCov$ gap & $\sCov$ gap & amp & $\hCov$ gap & $\sCov$ gap & amp \\
\midrule
Boots-CC (baseline) & \phantom{0}4.6 & \phantom{0}4.5 & 60.5 $\pm$ 30.8 & 39.3 $\pm$ 28.1 & \phantom{0}1.7 $\pm$ \phantom{0}2.8 & --- & --- & --- \\
Boots-PCC (baseline) & \phantom{0}5.6 & \phantom{0}5.5 & 74.5 $\pm$ 24.0 & 52.4 $\pm$ 30.3 & \textbf{\phantom{0}1.2 $\pm$ \phantom{0}2.1} & --- & --- & --- \\
Bayes-ACC \citep{bayesianCC.2024} & \phantom{0}3.9 & \phantom{0}3.9 & \textbf{13.3 $\pm$ 21.7} & \textbf{\phantom{0}4.4 $\pm$ \phantom{0}8.3} & \phantom{0}5.4 $\pm$ \phantom{0}7.3 & 18.5 $\pm$ 23.4 & \phantom{0}7.1 $\pm$ 11.5 & \phantom{0}6.6 $\pm$ 10.6 \\
Bayes-EMQ (MAPLS) \citep{ye2024label} & \phantom{0}2.9 & \phantom{0}2.6 & 25.3 $\pm$ 30.2 & 10.9 $\pm$ 14.9 & \phantom{0}3.0 $\pm$ \phantom{0}4.2 & 19.5 $\pm$ 22.8 & \phantom{0}6.0 $\pm$ \phantom{0}7.5 & \textbf{\phantom{0}4.5 $\pm$ \phantom{0}7.9} \\
Bayes-KDEy(Gau) \textbf{(new)} & \phantom{0}2.1 & \phantom{0}2.5 & 20.3 $\pm$ 30.1 & \phantom{0}7.3 $\pm$ 14.4 & \phantom{0}3.6 $\pm$ \phantom{0}5.4 & \textbf{12.9 $\pm$ 18.0} & \textbf{\phantom{0}3.1 $\pm$ \phantom{0}5.6} & \phantom{0}5.2 $\pm$ \phantom{0}7.6 \\
Bayes-KDEy(Ait-\(\lambda\)) \textbf{(ours)} & \textbf{\phantom{0}1.9} & \textbf{\phantom{0}2.1} & 21.7 $\pm$ 31.1 & \phantom{0}7.9 $\pm$ 14.1 & \phantom{0}3.6 $\pm$ \phantom{0}5.4 & \textbf{12.9 $\pm$ 18.6} & \phantom{0}3.3 $\pm$ \phantom{0}6.1 & \phantom{0}6.4 $\pm$ 11.4 \\
\bottomrule
\end{tabular}
}%
\end{table*}


\textbf{Levels of shift.}
Figure~\ref{fig:rankbars} shows performance variation as a function of the severity of label shift. In this experiment, shift magnitude is measured in terms of $||\pi-\pi^{\mathrm{train}}||_1$, and the tests are partitioned into three quantile-based groups, corresponding to low-, medium-, and high-shift regimes with approximately the same number of experiments. AE ranks and standard deviations are computed at the micro-level within each group. This plot reveals our method tends to dominate across levels of shift and data modalities, with the sole exception of the text modality in the low-shift and medium-shift regimes, where Bayes-EMQ and Bayes-KDEy(Gau), respectively, show superior performance.

%
\begin{figure}
    \centering
    \includegraphics[height=1.1in,trim={.7cm 0 4.3cm 0},clip]{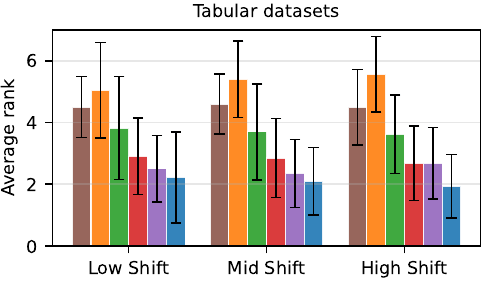}
    \includegraphics[height=1.1in,trim={1cm 0 4.3cm 0},clip]{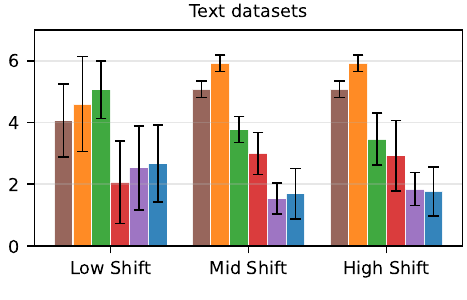}
    \includegraphics[height=1.1in,trim={1cm 0 .28cm 0},clip]{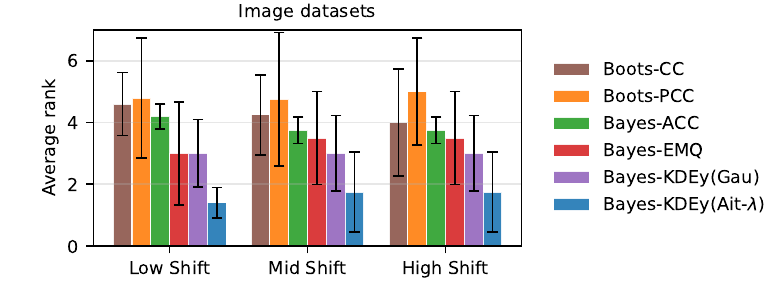}
    \caption{Average AE ranks in low-, medium-, and high-shift regimes for the three data modalities.}
    \label{fig:rankbars}
\end{figure}

\textbf{Target prevalence distributions.}
Following the experiments of \cite{Alexandari:2020dn}, we analyze quantification performance under differently distributed target prevalences as generated from symmetric Dirichlet distributions with \(\alpha_{\mathrm{test}} \in \{0.1, 10\}\); see Figure~\ref{fig:violin}. These settings favor prevalences near the simplex boundary 
and prevalences concentrated near the barycenter, respectively.\footnote{The LeQua datasets are not included in this experiment, since the test bags are fixed and do not contain  per-instance labels.} Unlike \cite{Alexandari:2020dn}, for Bayesian methods we additionally use informative but conservative symmetric Dirichlet priors, setting \(\alpha_{\mathrm{prior}}=0.3\) for \(\alpha_{\mathrm{test}}=0.1\) and \(\alpha_{\mathrm{prior}}=3\) for \(\alpha_{\mathrm{test}}=10\), effectively showcasing one of the main advantages of Bayesian inference. 
Results are by and large similar across methods for \(\alpha_{\mathrm{test}}=10\), including bootstrap variants that cannot benefit from prior information, suggesting that bags concentrated near the simplex center are comparatively easier to quantify. The most informative case is \(\alpha_{\mathrm{test}}=0.1\), where target prevalences lie near the simplex boundary and the task becomes more challenging. In this regime, bootstrap baselines degrade substantially, while Bayesian methods benefit from prior information. The shrinkage-regularized Aitchison KDE outperforms the Euclidean Gaussian KDE in this setting, where the target prevalence vector lies near the simplex boundary and geometric mismatch is expected to be most consequential. 
By respecting the simplex geometry, our method eliminates probability mass leakage outside valid bounds and provides higher resolution near the simplex boundaries, allowing for more precise estimation in this regime.

\begin{figure}
    \centering
    \includegraphics[width=\textwidth]{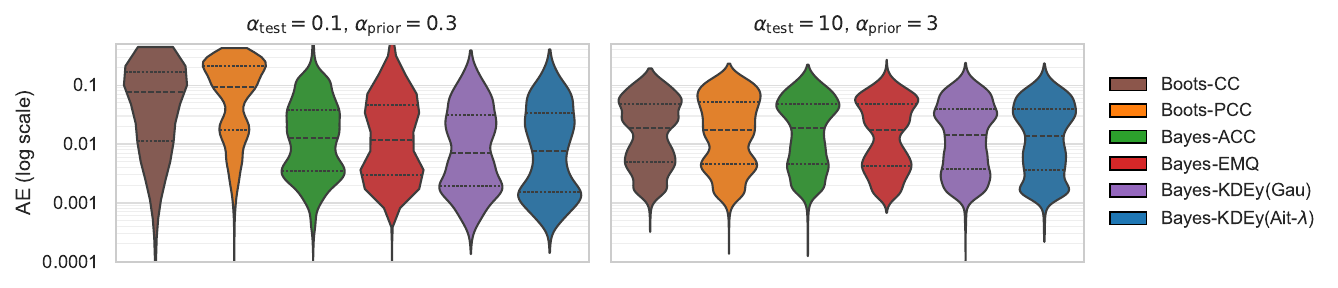}
    \caption{Violin plots for AE score distributions (log scale) for different target distributions.}
    \label{fig:violin}
\end{figure}

\textbf{Ablation experiments.} A naive Aitchison variant without shrinkage is reported in Appendix~\ref{app:ablation}. While competitive on many datasets, it exhibits sharp failures on a small but important set of benchmarks, namely \texttt{LeQua2022-T1B}, \texttt{LeQua2024-T2}, \texttt{isolet}, \texttt{mhr}, and \texttt{wine-quality}. In these cases, the pure Aitchison kernel (akin to $\lambda=0$) yields substantially larger AE/$W$ errors than the Gaussian variant. 
Importantly, these failures are largely removed by shrinkage, which brings the regularized Aitchison variant back to the range of the strongest competitors and, in most of these cases, above Gaussian KDE. More generally, the shrinkage values chosen during model selection are far from uniform across the 42 datasets: a near-pure Aitchison regime (\(\lambda=0.001\)) is the most frequently selected setting, occurring in \(31.0\%\) of the cases, while a near-Gaussian regime (\(\lambda=0.999\)) is the second most frequent, occurring in \(16.7\%\) of the cases. This indicates that shrinkage is not merely acting as a numerical safeguard, but as an adaptive mechanism that selects the geometry most compatible with the posterior structure of each dataset.

\textbf{Computational cost.} The proposed method requires the bandwidth and shrinkage to be optimized, which increases the computational cost. However, once set, the cost at inference time is not significantly different from that of the competitors.
All experiments were run on a standard desktop workstation, without requiring specialized high-performance computing infrastructure.

\textbf{Statistical significance.}
Pairwise significance was assessed at the dataset level using paired Wilcoxon signed-rank tests with Holm--Bonferroni correction after averaging results across bags within each dataset; see Appendix~\ref{app:fulltables:point} and~\ref{app:fulltables:bayes}. The proposed method consistently belongs to the top statistical group. In the point-estimation regime, it significantly improves over CC, PCC, and BBSE/ACC (\(p<0.001\)) and over MLLS/EMQ (\(p<0.05\)), while remaining statistically indistinguishable from KDEy(Gau). In the Bayesian regime at \(T=1\), KDEy(Ait-\(\lambda\)) significantly improves over the bootstrap baselines, Bayes-ACC, and Bayes-EMQ (MAPLS) (\(p<0.001\)), while remaining statistically tied with its Gaussian KDE analogue within the Bayesian formulation developed in this paper. The same pattern holds after temperature calibration.



\section{Conclusions}


We have presented a geometry-aware, uncertainty-aware approach to multiclass quantification under label shift. The method is based on two ideas. First, we revisit KDE-based quantification through compositional data analysis, replacing Euclidean density estimation on posterior probabilities with a shrinkage-regularized log-ratio representation better aligned with simplex geometry. Second, we show that this density model admits a Bayesian formulation, enabling posterior inference over class prevalences and principled uncertainty quantification. Empirically, the proposed geometry-aware KDE often outperforms standard Gaussian KDE and other baselines in both point estimation and uncertainty quantification. These results suggest that respecting simplex geometry is not only well motivated, but also practically beneficial for multiclass quantification. Our experiments also show that geometry-awareness alone is not sufficient: shrinkage is central to stabilizing log-ratio modeling near the simplex boundary.

\textbf{Limitations and open directions: } Our investigation is limited to the case in which the shrinkage and bandwidth values are selected by validation, while a fuller Bayesian treatment of these parameters may be worth investigating. 
Our method strongly relies on the distributional assumptions characteristics of label shift. Studying variants for different types of shift such as covariate shift \citep{tasche2022class} or sparse join shift \citep{chen2022estimating}, appears to be a promising direction for future work in quantification.

\bibliographystyle{plain}
\bibliography{references}


\newpage
\section*{Appendix}
\appendix

\section{Proof of Proposition~\ref{prop:clr-linearization}: linearization of CLR through shrinkage}
\label{app:clr-linearization}

In this appendix we prove the local linearization result stated in the main text. Throughout the proof, we let \(\phi=\clr\).

\begin{proposition}
\label{prop:clr-linearization-app}
Let
\(
u=\left(\frac1K,\ldots,\frac1K\right)
\)
be the barycenter of the simplex $\Delta^{K-1}$, and define the shrinkage operator
\[
T_\lambda(p)=(1-\lambda)p+\lambda u,
\qquad \lambda\in[0,1).
\]

For any fixed \(p\in\Delta^{K-1,\circ}\), it holds that
\[
\phi(T_\lambda(p))
=
K(1-\lambda)(p-u)
+
O\!\left((1-\lambda)^2\right)
\qquad \text{as } \lambda\to 1.
\]
\end{proposition}

\begin{proof}
The shrinkage operator can be written as
\[
T_\lambda(p)=u+\varepsilon\delta,
\]
with 
$\delta := p-u$, $\varepsilon := 1-\lambda$.
Since both \(p\) and \(u\) belong to the simplex, it is implied that
\(
\sum_{i=1}^K \delta_i = 0.
\)
Note that each component of \(T_\lambda(p)\) can be written as
\[
[T_\lambda(p)]_i = \frac1K + \varepsilon \delta_i.
\]

Because \(p\in\Delta^{K-1,\circ}\), all its coordinates are strictly positive. Hence, for \(\lambda\) sufficiently close to \(1\) (equivalently, \(\varepsilon\) sufficiently small), all coordinates of \(T_\lambda(p)\) remain strictly positive as well.

By definition of the CLR map,
\begin{equation}
\label{eq:app:clr}
[\phi(T_\lambda(p))]_i
=
\log\bigl([T_\lambda(p)]_i\bigr)
-
\frac1K\sum_{j=1}^K \log\bigl([T_\lambda(p)]_j\bigr).
\end{equation}

We start with the first term of Equation~\ref{eq:app:clr}, and expand the logarithm around the simplex barycenter. For each \(i\in\{1,\dots,K\}\),
\[
\log\bigl([T_\lambda(p)]_i\bigr)
=
\log\left(\frac1K+\varepsilon\delta_i\right)
=
\log\frac1K + \log\bigl(1+K\varepsilon\delta_i\bigr).
\]
Using the second-order Taylor expansion
\(
\log(1+x)=x-\frac{x^2}{2}+O(x^3),
\)
for $x=K\varepsilon\delta_i$ in the right-hand term, we obtain
\begin{equation}
\label{eq:app:clrcoord}
\log([T_\lambda(p)]_i)
=
\log\frac1K
+
K\varepsilon\delta_i
-
\frac{K^2\varepsilon^2}{2}\delta_i^2
+
O(\varepsilon^3\|\delta\|^3).
\end{equation}

For the second term of Equation~\ref{eq:app:clr}, we need to average over the log coordinates in Equation~\ref{eq:app:clrcoord}. This yields
\[
\frac1K\sum_{j=1}^K \log\bigl([T_\lambda(p)]_j\bigr)
=
\log\frac1K
+
K\varepsilon\frac1K\sum_{j=1}^K\delta_j
-
\frac{K^2\varepsilon^2}{2}\frac1K\sum_{j=1}^K\delta_j^2
+
O(\varepsilon^3\|\delta\|^3).
\]
Since \(\sum_{j=1}^K\delta_j=0\), this simplifies to
\begin{equation}
\label{eq:app:avecoor}   
\frac1K\sum_{j=1}^K \log\bigl([T_\lambda(p)]_j\bigr)
=
\log\frac1K
-
\frac{K\varepsilon^2}{2}\sum_{j=1}^K\delta_j^2
+
O(\varepsilon^3\|\delta\|^3).
\end{equation}

For the complete derivation of the CLR map (Equation~\ref{eq:app:clr}), we need to subtract the two expansions above (Equations~\ref{eq:app:clrcoord}~and~\ref{eq:app:avecoor}); this gives
\begin{align*}
[\phi(T_\lambda(p))]_i
&=
K\varepsilon\delta_i
-
\frac{K^2\varepsilon^2}{2}\delta_i^2
+
\frac{K\varepsilon^2}{2}\sum_{j=1}^K\delta_j^2
+
O(\varepsilon^3\|\delta\|^3) \\
&=
K\varepsilon\delta_i
+
O(\varepsilon^2\|\delta\|^2).
\end{align*}
Since this holds componentwise, we can write in vector form
\[
\phi(T_\lambda(p))
=
K\varepsilon\delta
+
O(\varepsilon^2\|\delta\|^2).
\]
Recalling that \(\varepsilon=1-\lambda\) and \(\delta=p-u\), we obtain
\[
\phi(T_\lambda(p))
=
K(1-\lambda)(p-u)
+
O\!\left((1-\lambda)^2\|p-u\|^2\right),
\]
which proves the first claim. The weaker bound
\[
\phi(T_\lambda(p))
=
K(1-\lambda)(p-u)+O\!\left((1-\lambda)^2\right)
\]
follows immediately for fixed \(p\).
\end{proof}

\begin{remark}
The proposition shows that shrinkage does not remove the compositional structure of the problem. Rather, it moves posterior vectors towards a regime where the log-ratio map is locally well approximated by a linear transformation. In this sense, shrinkage regularizes the most nonlinear part of the simplex geometry, namely the neighborhood of the boundary.
\end{remark}

\section{Bandwidth scaling under shrinkage}
\label{app:bandwidth-scaling}

The regularization of the log-ratio coordinates in our model is achieved by a shrinkage operator that effectively reduces the operational region of the density model in the simplex. In order to preserve coherence with a set of reference bandwidth values (e.g., a grid of values to be explored under model selection), the kernel bandwidths have to be transformed consequentially.
In this appendix we justify the bandwidth rescaling used in the main text. Recall the shrinkage map
\[
T_\lambda(p) = (1-\lambda)p + \lambda u,
\qquad \lambda\in[0,1),
\]
where \(u=\left(\frac1K,\ldots,\frac1K\right)\) is the barycenter of the simplex \(\Delta^{K-1}\). Let \(d=K-1\) denote the intrinsic dimension of the simplex.

\begin{proposition}
\label{prop:bandwidth-scale}
Under the shrinkage map \(T_\lambda\), the natural scaling of a reference bandwidth $h$ is
\[
h_{\mathrm{eff}} = (1-\lambda)h.
\]
Moreover, if a density is normalized with respect to the simplex volume measure, then its height rescales by the Jacobian factor
\(
(1-\lambda)^{-d}.
\)
\end{proposition}

\begin{proof}
The map \(T_\lambda\) is affine, and for any two simplex points \(p,q\in\Delta^{K-1}\),
\[
T_\lambda(p)-T_\lambda(q)
=
(1-\lambda)(p-q).
\]
Hence all Euclidean distances in the affine hyperplane containing the simplex are uniformly contracted by the factor
\(
a = 1-\lambda.
\)

Now consider an isotropic Gaussian kernel centered at some point \(c\) in the simplex hyperplane,
\[
\mathcal{K}_h(p-c)
\propto
\exp\!\left(
-\frac{\|p-c\|^2}{2h^2}
\right).
\]
After shrinkage, the corresponding point is
\(p' = T_\lambda(p),\)
and similarly the center is mapped to
\(
c'=T_\lambda(c).
\)
Since \(T_\lambda\) is affine,
\[
p' - c'
=
(1-\lambda)(p-c),
\]
so that
\[
\|p' - c'\| = (1-\lambda)\|p-c\|.
\]

If we want the kernel to preserve the same relative spread with respect to the contracted simplex, then the normalized distance should remain unchanged, i.e.,
\[
\frac{\|p' - c'\|}{h_{\mathrm{eff}}}
=
\frac{\|p-c\|}{h}.
\]
Substituting \(\|p' - c'\|=(1-\lambda)\|p-c\|\), we obtain
\[
\frac{(1-\lambda)\|p-c\|}{h_{\mathrm{eff}}}
=
\frac{\|p-c\|}{h},
\]
which implies
\[
h_{\mathrm{eff}} = (1-\lambda)h.
\]

To derive the density scaling, note that \(T_\lambda\) acts on a \(d\)-dimensional affine subspace, with \(d=K-1\). Since the map is a uniform contraction by factor \(1-\lambda\), its Jacobian determinant on this subspace is
\(
(1-\lambda)^d.
\)
Therefore, if a density is normalized with respect to the simplex volume measure, then under change of variables its height must be multiplied by
\(
(1-\lambda)^{-d}
\)
in order to preserve total probability mass.
\end{proof}


\begin{remark}[Discarding Jacobian scaling]
For fixed \(\lambda\), the Jacobian factor \((1-\lambda)^{-d}\) is a constant multiplicative term shared by all kernels and all class-conditional densities. It therefore factors out of the mixture in Equation~\ref{eq:kde} and is independent of \(\pi\). Consequently, it only adds a constant to the log-likelihood in Equation~\ref{eq:maxlike} and only changes the posterior normalization in Equation~\ref{eq:bayes}. For fixed $\lambda$ it may thus be omitted for inference with respect to \(\pi\).
\end{remark}

\begin{remark}[Connection with the CLR geometry]
The previous proposition is stated in the simplex hyperplane. In the main text, however, density estimation is performed after the log-ratio transformation \(\phi\). The local linearization result of Proposition~\ref{prop:clr-linearization} shows that, near the simplex barycenter,
\[
\phi(T_\lambda(p))
=
K(1-\lambda)(p-u)
+
O\!\left((1-\lambda)^2\|p-u\|^2\right).
\]
Thus, to first order, shrinkage also contracts distances in CLR coordinates by the same factor \(1-\lambda\). This justifies using the same bandwidth rescaling rule in the transformed space:
\[
h_{\mathrm{eff}} = (1-\lambda)h.
\]
In other words, shrinkage not only regularizes the geometry, but also induces a natural linear rescaling of the kernel bandwidth.
\end{remark}

\section{Estimator-level effects of shrinkage}
\label{app:estimator}

In this appendix we provide a heuristic interpretation of the effect of shrinkage at the level of the KDE estimator. The goal is not to establish a full estimator equivalence, but to explain why shrinkage moves the geometry-aware estimator into a locally Euclidean regime near the simplex barycenter.

Let \(\Delta^{K-1}\) be the probability simplex, let \(u\) denote its barycenter, and let
\[
T_\lambda(p) = (1-\lambda)p + \lambda u
\]
be the shrinkage map. Let \(\phi=\clr\), and let
\[
\mathcal P^\star = \{p^{(1)},\dots,p^{(n)}\}\subset \Delta^{K-1}
\]
be a set of reference posterior vectors. For \(d=K-1\), define the Gaussian kernel on the simplex hyperplane by
\[
\mathcal K_b(r)
=
(2\pi b^2)^{-d/2}
\exp\!\left(-\frac{\|r\|^2}{2b^2}\right).
\]
Consider the transformed-space KDE
\[
\widehat f^{\mathrm{tr}}_{\lambda,h}(p)
=
\frac1n
\sum_{i=1}^n
\mathcal K_h\!\left(
\phi(T_\lambda(p))-\phi(T_\lambda(p^{(i)}))
\right).
\]

\begin{remark}[Heuristic local Euclidean regime induced by shrinkage]
Assume that \(p\) and all reference points \(p^{(i)}\) lie in a sufficiently small neighborhood of the barycenter. By the local linearization result of Appendix~\ref{app:clr-linearization},
\[
\phi(T_\lambda(p))-\phi(T_\lambda(p^{(i)}))
=
K(1-\lambda)\bigl(p-p^{(i)}\bigr) + r_i(p),
\]
where the remainder \(r_i(p)\) is of higher order in the local displacement from \(u\).

Now set
\[
h_{\mathrm{eff}} = (1-\lambda)h.
\]
Substituting the linear term into the Gaussian kernel yields
\[
\mathcal K_{h_{\mathrm{eff}}}\!\left(K(1-\lambda)(p-p^{(i)})\right)
=
C_{\lambda,K}\,
\mathcal K_{h/K}\!\left(p-p^{(i)}\right),
\]
where
\[
C_{\lambda,K} = \bigl(K(1-\lambda)\bigr)^{-(K-1)}
\]
is a multiplicative constant independent of \(p\). A first-order Taylor expansion of \(\mathcal K_{h_{\mathrm{eff}}}\) around the linear term therefore gives
\[
\mathcal K_{h_{\mathrm{eff}}}\!\left(
\phi(T_\lambda(p))-\phi(T_\lambda(p^{(i)}))
\right)
=
C_{\lambda,K}\,
\mathcal K_{h/K}\!\left(p-p^{(i)}\right)
+
\text{higher-order local terms}.
\]
Averaging over \(i\), one obtains the heuristic approximation
\[
\widehat f^{\mathrm{tr}}_{\lambda,h_{\mathrm{eff}}}(p)
\approx
C_{\lambda,K}\,
\widehat f^{\mathrm{Euc}}_{h/K}(p),
\]
where \(\widehat f^{\mathrm{Euc}}_{h/K}\) denotes the Euclidean Gaussian KDE on the simplex hyperplane with bandwidth \(h/K\).

Thus, shrinkage does not remove the compositional structure of the estimator, but places it in a regime where the induced smoothing is locally well approximated by Euclidean Gaussian smoothing. Since the factor \(C_{\lambda,K}\) is independent of \(p\), it contributes only an additive constant to the corresponding log-likelihood. Consequently, it does not affect maximum likelihood estimation or Bayesian inference with respect to the prevalence vector \(\pi\).
\end{remark}

The class-conditional estimators used in the main text are obtained by applying the same construction separately to the class-specific reference sets.

\section{Additional Experimental Details}
\label{app:details}

\subsection{Datasets}
\label{app:datasets}

Dataset details are provided in Table~\ref{tab:dataset-details}. 
%
%
For the LeQua datasets, validation and test pool sizes are not reported because the data are provided directly as predefined bags rather than as instance-level pools from which bags are sampled. In our experiments, we use the first 100 validation bags and the first 500 test bags.
For image datasets, the training and validation sizes reported in this table refer to the effective feature-level splits used for quantification, not to the original raw dataset splits. The representation model is first trained on 70\% of the original training set and monitored on the remaining 30\%. After training, features are extracted from that 30\% split and from the original test set; the training/validation/test sizes reported here are defined on these extracted-feature splits.

\begin{table}[h]
\centering
\caption{Dataset details. }
\label{tab:dataset-details}
\begin{tabular}{llrrrrrr}
\toprule
Dataset & Modality & Train & Val & Test & Bag size & Features & Classes \\
\midrule
\texttt{abalone} & tabular & 2,846 & 1,139 & 1,220 & 1,000 & 9 & 15 \\
\texttt{academic-success} & tabular & 3,096 & 1,239 & 1,328 & 1,000 & 36 & 3 \\
\texttt{chess} & tabular & 19,198 & 7,680 & 8,228 & 1,000 & 20 & 13 \\
\texttt{cmc} & tabular & 1,031 & 413 & 442 & 1,000 & 9 & 3 \\
\texttt{connect-4} & tabular & 25,000 & 10,000 & 42,557 & 1,000 & 84 & 3 \\
\texttt{digits} & tabular & 3,933 & 1,574 & 1,687 & 1,000 & 64 & 10 \\
\texttt{dry-bean} & tabular & 9,527 & 3,811 & 4,084 & 1,000 & 16 & 7 \\
\texttt{hand\_digits} & tabular & 7,694 & 3,078 & 3,298 & 1,000 & 16 & 10 \\
\texttt{image\_seg} & tabular & 1,617 & 647 & 693 & 1,000 & 19 & 7 \\
\texttt{isolet} & tabular & 5,457 & 2,183 & 2,340 & 1,000 & 617 & 26 \\
\texttt{letter} & tabular & 14,000 & 5,600 & 6,000 & 1,000 & 16 & 26 \\
\texttt{mhr} & tabular & 709 & 284 & 305 & 1,000 & 6 & 3 \\
\texttt{molecular} & tabular & 2,233 & 894 & 957 & 1,000 & 227 & 3 \\
\texttt{nursery} & tabular & 9,070 & 3,628 & 3,888 & 1,000 & 19 & 4 \\
\texttt{obesity} & tabular & 1,477 & 591 & 634 & 1,000 & 23 & 7 \\
\texttt{page\_block} & tabular & 3,811 & 1,525 & 1,634 & 1,000 & 10 & 4 \\
\texttt{phishing} & tabular & 947 & 379 & 406 & 1,000 & 9 & 3 \\
\texttt{satellite} & tabular & 4,504 & 1,802 & 1,931 & 1,000 & 36 & 6 \\
\texttt{shuttle} & tabular & 25,000 & 10,000 & 32,756 & 1,000 & 7 & 3 \\
\texttt{waveform-v1} & tabular & 3,500 & 1,400 & 1,500 & 1,000 & 21 & 3 \\
\texttt{wine-quality} & tabular & 4,523 & 1,810 & 1,939 & 1,000 & 11 & 5 \\
\texttt{yeast} & tabular & 1,035 & 414 & 444 & 1,000 & 8 & 9 \\
\texttt{gasp} & text & 8,788 & 1,256 & 3,765 & 100 & 78,244 & 3 \\
\texttt{hcr} & text & 1,594 & 797 & 798 & 100 & 22,061 & 3 \\
\texttt{omd} & text & 1,839 & 263 & 787 & 100 & 19,385 & 3 \\
\texttt{sanders} & text & 2,155 & 308 & 923 & 100 & 23,980 & 3 \\
\texttt{semeval13} & text & 11,338 & 1,654 & 3,813 & 100 & 107,984 & 3 \\
\texttt{semeval14} & text & 11,338 & 1,654 & 1,853 & 100 & 107,984 & 3 \\
\texttt{semeval15} & text & 11,338 & 1,654 & 2,390 & 100 & 107,984 & 3 \\
\texttt{semeval16} & text & 8,000 & 2,000 & 2,000 & 100 & 76,387 & 3 \\
\texttt{sst} & text & 2,971 & 425 & 1,271 & 100 & 29,664 & 3 \\
\texttt{wa} & text & 2,184 & 312 & 936 & 100 & 23,745 & 3 \\
\texttt{wb} & text & 4,259 & 609 & 1,823 & 100 & 40,507 & 3 \\
\texttt{LeQua2022-T1A} & text & 5,000 & --- & --- & 250 & 300 & 2 \\
\texttt{LeQua2022-T1B} & text & 20,000 & --- & --- & 1,000 & 300 & 28 \\
\texttt{LeQua2024-T1} & text & 5,000 & --- & --- & 250 & 256 & 2 \\
\texttt{LeQua2024-T2} & text & 20,000 & --- & --- & 1,000 & 256 & 28 \\
\texttt{CIFAR100} & image & 9,000 & 6,000 & 10,000 & 2,000 & 100 & 100 \\
\texttt{CIFAR10} & image & 9,000 & 6,000 & 10,000 & 500 & 10 & 10 \\
\texttt{MNIST} & image & 10,800 & 7,200 & 10,000 & 500 & 10 & 10 \\
\texttt{FashionMNIST} & image & 10,800 & 7,200 & 10,000 & 500 & 10 & 10 \\
\texttt{SVHN} & image & 13,186 & 8,791 & 26,032 & 500 & 10 & 10 \\
\bottomrule
\end{tabular}
\end{table}

\subsection{Hyperparameter Grids and Model Selection}
\label{app:hyperparameters}

Classifier hyperparameters are treated as quantifier hyperparameters and are optimized using quantification-oriented validation loss. The hyperparameters we explore are specified in Table~\ref{tab:hyperparameter-grids}.

\begin{table}[h]
\centering
\caption{Hyperparameter grids used in model selection.}
\label{tab:hyperparameter-grids}
\begin{tabular}{lll}
\toprule
Component & Hyperparameter & Grid \\
\midrule
Logistic Regression & \(C\) & \texttt{logspace(-4,4,9)} \\
Logistic Regression & \texttt{class\_weight} & \{\texttt{balanced}, \texttt{None}\} \\
Gaussian KDE & \(h\) & \texttt{logspace(-2,0,10)} \\
Aitchison KDE & \(h\) & \texttt{logspace(-1,1,10)} \\
Shrinkage & \(\lambda\) & \(\{0.001, 0.25, 0.5, 0.75, 0.9, 0.999\}\) \\
\bottomrule
\end{tabular}
\end{table}

Hyperparameters are selected by minimizing the absolute error (AE) of the corresponding point estimator on 100 validation bags generated under a Dirichlet\((1)\) sampling protocol. For Bayesian methods, we additionally perform \textit{post hoc} temperature calibration after model selection. The temperature is searched over the grid 
\(
T \in \left\{\tfrac{1}{2}, 1, 1.5, 2, 5, 10, 100, 1000\right\}
\), and selected by minimizing the mean Winkler score (aka interval score) \citep{gneiting2007strictly} across classes on the same validation bags. Given classwise confidence intervals \(\mathcal{C}=\{[\ell_k,u_k]\}_{k=1}^K\) constructed at significance level \(\alpha\), the mean Winkler score is defined as
\begin{equation}
\label{eq:winkler}
    \mathcal{W}(\pi,\mathcal C;\alpha)
    =
    \frac{1}{K}\sum_{k=1}^K
    \left[
    (u_k-\ell_k)
    +
    \frac{2}{\alpha}
    \left(
    \max(0,\ell_k-\pi_k)
    +
    \max(0,\pi_k-u_k)
    \right)
    \right].
\end{equation}

\clearpage

\subsection{Full Tables: Performance of point-estimators.}
\label{app:fulltables:point}


Table~\ref{tab:point} reports the full AE and $W$ results for point-estimation performance. The table uses the following conventions: For each dataset and metric, the best result is highlighted in bold. Results that are \textit{not} statistically significantly different from the best one according to a paired Wilcoxon signed-rank test at significance level 0.05 are marked with  $\dag$ symbol. In addition, cells are color-coded for ease of visualization, with more intense cells indicating lower (better) errors and lighter cells indicating higher (worse) errors.
Statistical significance comparisons are reported in Table~\ref{tab:ae-pairwise-symbols}. Symbols denote pairwise Wilcoxon-Holm comparisons across datasets on AE: $\ll$ / $<$ mean the row method is significantly better than the column method ($p$ < 0.001 / $p$ < 0.05), $\gg$ / $>$ mean significantly worse, $\approx$ means not significantly different, and = marks the diagonal.


\begin{table}[h]
    \caption{Point-estimation performance in terms of AE and $W$}
    \label{tab:point}
    \centering    
    \resizebox{\textwidth}{!}{%
\begin{tabular}{ll|cccccc|cccccc|}
\toprule
\multicolumn{2}{c}{} & \multicolumn{1}{c}{\begin{sideways}CC\;\end{sideways}} & \multicolumn{1}{c}{\begin{sideways}PCC\;\end{sideways}} & \multicolumn{1}{c}{\begin{sideways}BBSE/ACC\;\end{sideways}} & \multicolumn{1}{c}{\begin{sideways}MLLS/EMQ\;\end{sideways}} & \multicolumn{1}{c}{\begin{sideways}KDEy(Gau)\;\end{sideways}} & \multicolumn{1}{c}{\begin{sideways}KDEy(Ait-$\lambda$)\;\end{sideways}} & \multicolumn{1}{c}{\begin{sideways}CC\;\end{sideways}} & \multicolumn{1}{c}{\begin{sideways}PCC\;\end{sideways}} & \multicolumn{1}{c}{\begin{sideways}BBSE/ACC\;\end{sideways}} & \multicolumn{1}{c}{\begin{sideways}MLLS/EMQ\;\end{sideways}} & \multicolumn{1}{c}{\begin{sideways}KDEy(Gau)\;\end{sideways}} & \multicolumn{1}{c}{\begin{sideways}KDEy(Ait-$\lambda$)\;\end{sideways}} \\
\cmidrule(lr){3-8}\cmidrule(lr){9-14}
\multicolumn{2}{c}{} & \multicolumn{6}{c}{AE} & \multicolumn{6}{c}{W} \\
\midrule
\multirow{22}{*}{\begin{sideways}tabular\end{sideways}} & \texttt{abalone} & $0.0450$\cellcolor{olive!28} & $\textbf{0.0424}$\cellcolor{olive!30} & $0.0966$\cellcolor{olive!0} & $0.0655$\cellcolor{olive!17} & $0.0496$\cellcolor{olive!26} & $0.0505$\cellcolor{olive!25} & $7.0645$\cellcolor{olive!29} & $\textbf{6.2734}$\cellcolor{olive!30} & $96.8502$\cellcolor{olive!0} & $15.7999$\cellcolor{olive!26} & $9.7935$\cellcolor{olive!28} & $9.8538$\cellcolor{olive!28} \\
 & \texttt{academic-success} & $0.0953$\cellcolor{olive!7} & $0.1211$\cellcolor{olive!0} & $0.0491$\cellcolor{olive!22} & $0.0260$\cellcolor{olive!29} & $0.0251$\cellcolor{olive!29} & $\textbf{0.0236}$\cellcolor{olive!30} & $0.2541$\cellcolor{olive!12} & $0.4139$\cellcolor{olive!0} & $0.0796$\cellcolor{olive!25} & $0.0231$\cellcolor{olive!29} & $0.0210$\cellcolor{olive!29} & $\textbf{0.0188}$\cellcolor{olive!30} \\
 & \texttt{chess} & $0.0413$\cellcolor{olive!3} & $0.0450$\cellcolor{olive!0} & $0.0409$\cellcolor{olive!4} & $0.0398$\cellcolor{olive!5} & $^\dag0.0165$\cellcolor{olive!29} & $\textbf{0.0163}$\cellcolor{olive!30} & $1.9618$\cellcolor{olive!5} & $2.3372$\cellcolor{olive!0} & $1.1529$\cellcolor{olive!17} & $1.3512$\cellcolor{olive!14} & $\textbf{0.3125}$\cellcolor{olive!30} & $0.3878$\cellcolor{olive!28} \\
 & \texttt{cmc} & $0.1560$\cellcolor{olive!7} & $0.1778$\cellcolor{olive!0} & $0.1065$\cellcolor{olive!25} & $0.0973$\cellcolor{olive!28} & $^\dag0.0937$\cellcolor{olive!29} & $\textbf{0.0934}$\cellcolor{olive!30} & $0.3903$\cellcolor{olive!10} & $0.4937$\cellcolor{olive!0} & $0.2123$\cellcolor{olive!27} & $\textbf{0.1884}$\cellcolor{olive!30} & $^\dag0.1913$\cellcolor{olive!29} & $^\dag0.1925$\cellcolor{olive!29} \\
 & \texttt{connect-4} & $0.1138$\cellcolor{olive!7} & $0.1468$\cellcolor{olive!0} & $0.0270$\cellcolor{olive!28} & $0.0301$\cellcolor{olive!27} & $^\dag0.0205$\cellcolor{olive!29} & $\textbf{0.0201}$\cellcolor{olive!30} & $1.0574$\cellcolor{olive!11} & $1.6642$\cellcolor{olive!0} & $0.0742$\cellcolor{olive!29} & $0.1020$\cellcolor{olive!28} & $0.0412$\cellcolor{olive!29} & $\textbf{0.0395}$\cellcolor{olive!30} \\
 & \texttt{digits} & $0.0035$\cellcolor{olive!10} & $0.0045$\cellcolor{olive!0} & $0.0030$\cellcolor{olive!15} & $0.0023$\cellcolor{olive!22} & $0.0024$\cellcolor{olive!21} & $\textbf{0.0016}$\cellcolor{olive!30} & $0.0034$\cellcolor{olive!8} & $0.0045$\cellcolor{olive!0} & $0.0017$\cellcolor{olive!21} & $0.0012$\cellcolor{olive!25} & $0.0014$\cellcolor{olive!23} & $\textbf{0.0006}$\cellcolor{olive!30} \\
 & \texttt{dry-bean} & $0.0088$\cellcolor{olive!11} & $0.0122$\cellcolor{olive!0} & $0.0040$\cellcolor{olive!26} & $0.0032$\cellcolor{olive!29} & $0.0032$\cellcolor{olive!29} & $\textbf{0.0030}$\cellcolor{olive!30} & $0.0066$\cellcolor{olive!16} & $0.0131$\cellcolor{olive!0} & $0.0020$\cellcolor{olive!27} & $0.0012$\cellcolor{olive!29} & $0.0011$\cellcolor{olive!29} & $\textbf{0.0010}$\cellcolor{olive!30} \\
 & \texttt{hand\_digits} & $0.0048$\cellcolor{olive!9} & $0.0063$\cellcolor{olive!0} & $0.0031$\cellcolor{olive!20} & $0.0023$\cellcolor{olive!25} & $0.0023$\cellcolor{olive!25} & $\textbf{0.0016}$\cellcolor{olive!30} & $0.0054$\cellcolor{olive!13} & $0.0092$\cellcolor{olive!0} & $0.0018$\cellcolor{olive!25} & $0.0010$\cellcolor{olive!28} & $0.0011$\cellcolor{olive!27} & $\textbf{0.0006}$\cellcolor{olive!30} \\
 & \texttt{image\_seg} & $0.0063$\cellcolor{olive!16} & $0.0083$\cellcolor{olive!4} & $0.0090$\cellcolor{olive!0} & $0.0051$\cellcolor{olive!24} & $\textbf{0.0041}$\cellcolor{olive!30} & $0.0041$\cellcolor{olive!29} & $0.0098$\cellcolor{olive!15} & $0.0164$\cellcolor{olive!0} & $0.0114$\cellcolor{olive!11} & $0.0046$\cellcolor{olive!27} & $\textbf{0.0033}$\cellcolor{olive!30} & $0.0034$\cellcolor{olive!29} \\
 & \texttt{isolet} & $0.0021$\cellcolor{olive!6} & $0.0023$\cellcolor{olive!0} & $0.0018$\cellcolor{olive!14} & $0.0014$\cellcolor{olive!26} & $0.0014$\cellcolor{olive!28} & $\textbf{0.0013}$\cellcolor{olive!30} & $0.0077$\cellcolor{olive!6} & $0.0092$\cellcolor{olive!0} & $0.0051$\cellcolor{olive!19} & $0.0034$\cellcolor{olive!26} & $0.0029$\cellcolor{olive!29} & $\textbf{0.0027}$\cellcolor{olive!30} \\
 & \texttt{letter} & $0.0077$\cellcolor{olive!9} & $0.0108$\cellcolor{olive!0} & $0.0048$\cellcolor{olive!19} & $0.0043$\cellcolor{olive!20} & $0.0025$\cellcolor{olive!26} & $\textbf{0.0014}$\cellcolor{olive!30} & $0.0856$\cellcolor{olive!14} & $0.1631$\cellcolor{olive!0} & $0.0300$\cellcolor{olive!24} & $0.0253$\cellcolor{olive!25} & $0.0078$\cellcolor{olive!29} & $\textbf{0.0025}$\cellcolor{olive!30} \\
 & \texttt{mhr} & $0.1216$\cellcolor{olive!5} & $0.1444$\cellcolor{olive!0} & $0.1113$\cellcolor{olive!8} & $0.0549$\cellcolor{olive!22} & $\textbf{0.0257}$\cellcolor{olive!30} & $0.0340$\cellcolor{olive!27} & $0.2154$\cellcolor{olive!6} & $0.2776$\cellcolor{olive!0} & $0.1621$\cellcolor{olive!12} & $0.0443$\cellcolor{olive!26} & $\textbf{0.0104}$\cellcolor{olive!30} & $0.0164$\cellcolor{olive!29} \\
 & \texttt{molecular} & $0.0188$\cellcolor{olive!18} & $0.0388$\cellcolor{olive!0} & $0.0089$\cellcolor{olive!28} & $0.0071$\cellcolor{olive!29} & $0.0072$\cellcolor{olive!29} & $\textbf{0.0068}$\cellcolor{olive!30} & $0.0066$\cellcolor{olive!24} & $0.0294$\cellcolor{olive!0} & $0.0014$\cellcolor{olive!29} & $0.0013$\cellcolor{olive!29} & $0.0011$\cellcolor{olive!29} & $\textbf{0.0011}$\cellcolor{olive!30} \\
 & \texttt{nursery} & $0.0179$\cellcolor{olive!12} & $0.0265$\cellcolor{olive!0} & $\textbf{0.0050}$\cellcolor{olive!30} & $^\dag0.0051$\cellcolor{olive!29} & $0.0061$\cellcolor{olive!28} & $0.0069$\cellcolor{olive!27} & $0.1206$\cellcolor{olive!18} & $0.2875$\cellcolor{olive!0} & $\textbf{0.0139}$\cellcolor{olive!30} & $0.0256$\cellcolor{olive!28} & $0.0703$\cellcolor{olive!23} & $0.0474$\cellcolor{olive!26} \\
 & \texttt{obesity} & $0.0070$\cellcolor{olive!16} & $0.0072$\cellcolor{olive!15} & $0.0095$\cellcolor{olive!0} & $\textbf{0.0050}$\cellcolor{olive!30} & $0.0084$\cellcolor{olive!6} & $0.0085$\cellcolor{olive!6} & $0.0059$\cellcolor{olive!17} & $0.0065$\cellcolor{olive!15} & $0.0102$\cellcolor{olive!0} & $\textbf{0.0029}$\cellcolor{olive!30} & $0.0092$\cellcolor{olive!3} & $0.0093$\cellcolor{olive!3} \\
 & \texttt{page\_block} & $\textbf{0.0141}$\cellcolor{olive!30} & $0.0271$\cellcolor{olive!0} & $0.0217$\cellcolor{olive!12} & $0.0170$\cellcolor{olive!23} & $0.0194$\cellcolor{olive!17} & $0.0215$\cellcolor{olive!12} & $\textbf{0.3355}$\cellcolor{olive!30} & $1.4168$\cellcolor{olive!2} & $1.4289$\cellcolor{olive!2} & $0.5930$\cellcolor{olive!23} & $1.1913$\cellcolor{olive!8} & $1.5181$\cellcolor{olive!0} \\
 & \texttt{phishing} & $0.0953$\cellcolor{olive!3} & $0.1062$\cellcolor{olive!0} & $0.0553$\cellcolor{olive!18} & $0.0281$\cellcolor{olive!27} & $\textbf{0.0221}$\cellcolor{olive!30} & $0.0243$\cellcolor{olive!29} & $1.3672$\cellcolor{olive!1} & $1.4128$\cellcolor{olive!0} & $0.4169$\cellcolor{olive!22} & $0.1092$\cellcolor{olive!29} & $\textbf{0.0717}$\cellcolor{olive!30} & $0.0872$\cellcolor{olive!29} \\
 & \texttt{satellite} & $0.0231$\cellcolor{olive!10} & $0.0312$\cellcolor{olive!0} & $0.0111$\cellcolor{olive!24} & $0.0076$\cellcolor{olive!29} & $\textbf{0.0070}$\cellcolor{olive!30} & $0.0078$\cellcolor{olive!29} & $0.0767$\cellcolor{olive!14} & $0.1463$\cellcolor{olive!0} & $0.0194$\cellcolor{olive!27} & $0.0077$\cellcolor{olive!29} & $\textbf{0.0059}$\cellcolor{olive!30} & $0.0078$\cellcolor{olive!29} \\
 & \texttt{shuttle} & $0.0073$\cellcolor{olive!8} & $0.0093$\cellcolor{olive!0} & $0.0021$\cellcolor{olive!28} & $0.0054$\cellcolor{olive!15} & $0.0019$\cellcolor{olive!29} & $\textbf{0.0018}$\cellcolor{olive!30} & $0.0028$\cellcolor{olive!11} & $0.0043$\cellcolor{olive!0} & $0.0004$\cellcolor{olive!29} & $0.0014$\cellcolor{olive!21} & $0.0003$\cellcolor{olive!29} & $\textbf{0.0002}$\cellcolor{olive!30} \\
 & \texttt{waveform-v1} & $0.0416$\cellcolor{olive!9} & $0.0582$\cellcolor{olive!0} & $0.0250$\cellcolor{olive!19} & $\textbf{0.0070}$\cellcolor{olive!30} & $0.0102$\cellcolor{olive!28} & $^\dag0.0073$\cellcolor{olive!29} & $0.0234$\cellcolor{olive!14} & $0.0436$\cellcolor{olive!0} & $0.0094$\cellcolor{olive!23} & $\textbf{0.0007}$\cellcolor{olive!30} & $0.0015$\cellcolor{olive!29} & $0.0008$\cellcolor{olive!29} \\
 & \texttt{wine-quality} & $0.0998$\cellcolor{olive!16} & $0.1144$\cellcolor{olive!9} & $0.0879$\cellcolor{olive!21} & $\textbf{0.0714}$\cellcolor{olive!30} & $0.1328$\cellcolor{olive!0} & $^\dag0.0729$\cellcolor{olive!29} & $5.4520$\cellcolor{olive!16} & $7.6107$\cellcolor{olive!8} & $2.9600$\cellcolor{olive!25} & $\textbf{1.6965}$\cellcolor{olive!30} & $9.9969$\cellcolor{olive!0} & $3.8726$\cellcolor{olive!22} \\
 & \texttt{yeast} & $0.0539$\cellcolor{olive!15} & $0.0557$\cellcolor{olive!12} & $0.0561$\cellcolor{olive!12} & $0.0647$\cellcolor{olive!0} & $0.0449$\cellcolor{olive!28} & $\textbf{0.0437}$\cellcolor{olive!30} & $7.5748$\cellcolor{olive!17} & $6.3074$\cellcolor{olive!23} & $6.7092$\cellcolor{olive!21} & $11.1195$\cellcolor{olive!0} & $5.8082$\cellcolor{olive!26} & $\textbf{5.0089}$\cellcolor{olive!30} \\
\midrule
\multirow{15}{*}{\begin{sideways}text\end{sideways}} & \texttt{gasp} & $0.0899$\cellcolor{olive!11} & $0.1199$\cellcolor{olive!0} & $0.0518$\cellcolor{olive!25} & $0.0415$\cellcolor{olive!29} & $0.0409$\cellcolor{olive!29} & $\textbf{0.0397}$\cellcolor{olive!30} & $0.6271$\cellcolor{olive!16} & $1.2881$\cellcolor{olive!0} & $0.1704$\cellcolor{olive!28} & $\textbf{0.1001}$\cellcolor{olive!30} & $0.1149$\cellcolor{olive!29} & $^\dag0.1011$\cellcolor{olive!29} \\
 & \texttt{hcr} & $0.1265$\cellcolor{olive!10} & $0.1524$\cellcolor{olive!0} & $0.0962$\cellcolor{olive!23} & $^\dag0.0811$\cellcolor{olive!29} & $0.0814$\cellcolor{olive!29} & $\textbf{0.0793}$\cellcolor{olive!30} & $0.3777$\cellcolor{olive!9} & $0.4836$\cellcolor{olive!0} & $0.2078$\cellcolor{olive!25} & $\textbf{0.1610}$\cellcolor{olive!30} & $^\dag0.1737$\cellcolor{olive!28} & $^\dag0.1635$\cellcolor{olive!29} \\
 & \texttt{omd} & $0.1062$\cellcolor{olive!6} & $0.1182$\cellcolor{olive!0} & $0.0901$\cellcolor{olive!14} & $0.0688$\cellcolor{olive!26} & $^\dag0.0620$\cellcolor{olive!29} & $\textbf{0.0618}$\cellcolor{olive!30} & $0.1844$\cellcolor{olive!9} & $0.2356$\cellcolor{olive!0} & $0.1412$\cellcolor{olive!17} & $0.0876$\cellcolor{olive!27} & $\textbf{0.0715}$\cellcolor{olive!30} & $^\dag0.0745$\cellcolor{olive!29} \\
 & \texttt{sanders} & $0.1114$\cellcolor{olive!7} & $0.1341$\cellcolor{olive!0} & $0.0587$\cellcolor{olive!25} & $^\dag0.0470$\cellcolor{olive!29} & $\textbf{0.0459}$\cellcolor{olive!30} & $^\dag0.0463$\cellcolor{olive!29} & $0.5178$\cellcolor{olive!7} & $0.6633$\cellcolor{olive!0} & $0.1318$\cellcolor{olive!27} & $0.0930$\cellcolor{olive!29} & $\textbf{0.0839}$\cellcolor{olive!30} & $0.0892$\cellcolor{olive!29} \\
 & \texttt{semeval13} & $0.1115$\cellcolor{olive!9} & $0.1365$\cellcolor{olive!0} & $0.0758$\cellcolor{olive!23} & $0.0739$\cellcolor{olive!23} & $\textbf{0.0580}$\cellcolor{olive!30} & $^\dag0.0589$\cellcolor{olive!29} & $0.2692$\cellcolor{olive!16} & $0.5073$\cellcolor{olive!0} & $0.1111$\cellcolor{olive!26} & $0.1040$\cellcolor{olive!27} & $\textbf{0.0642}$\cellcolor{olive!30} & $^\dag0.0658$\cellcolor{olive!29} \\
 & \texttt{semeval14} & $0.1038$\cellcolor{olive!11} & $0.1332$\cellcolor{olive!0} & $0.0668$\cellcolor{olive!25} & $0.0656$\cellcolor{olive!26} & $\textbf{0.0560}$\cellcolor{olive!30} & $^\dag0.0568$\cellcolor{olive!29} & $0.2100$\cellcolor{olive!19} & $0.4723$\cellcolor{olive!0} & $0.0967$\cellcolor{olive!28} & $0.1070$\cellcolor{olive!27} & $\textbf{0.0776}$\cellcolor{olive!30} & $0.0818$\cellcolor{olive!29} \\
 & \texttt{semeval15} & $0.1253$\cellcolor{olive!9} & $0.1425$\cellcolor{olive!0} & $0.0958$\cellcolor{olive!24} & $^\dag0.0865$\cellcolor{olive!29} & $\textbf{0.0863}$\cellcolor{olive!30} & $0.0913$\cellcolor{olive!27} & $0.2880$\cellcolor{olive!14} & $0.4533$\cellcolor{olive!0} & $0.1585$\cellcolor{olive!25} & $\textbf{0.1124}$\cellcolor{olive!30} & $0.1520$\cellcolor{olive!26} & $0.1715$\cellcolor{olive!24} \\
 & \texttt{semeval16} & $0.1389$\cellcolor{olive!19} & $0.1612$\cellcolor{olive!10} & $0.1891$\cellcolor{olive!0} & $\textbf{0.1110}$\cellcolor{olive!30} & $0.1370$\cellcolor{olive!20} & $0.1390$\cellcolor{olive!19} & $0.5124$\cellcolor{olive!26} & $0.6856$\cellcolor{olive!19} & $1.1685$\cellcolor{olive!0} & $\textbf{0.4215}$\cellcolor{olive!30} & $0.6520$\cellcolor{olive!20} & $0.6749$\cellcolor{olive!19} \\
 & \texttt{sst} & $0.1051$\cellcolor{olive!11} & $0.1352$\cellcolor{olive!0} & $0.0749$\cellcolor{olive!22} & $^\dag0.0539$\cellcolor{olive!29} & $0.0541$\cellcolor{olive!29} & $\textbf{0.0538}$\cellcolor{olive!30} & $0.1687$\cellcolor{olive!13} & $0.2659$\cellcolor{olive!0} & $0.0904$\cellcolor{olive!24} & $^\dag0.0487$\cellcolor{olive!29} & $0.0490$\cellcolor{olive!29} & $\textbf{0.0484}$\cellcolor{olive!30} \\
 & \texttt{wa} & $0.0783$\cellcolor{olive!2} & $0.0819$\cellcolor{olive!0} & $0.0517$\cellcolor{olive!22} & $0.0500$\cellcolor{olive!24} & $\textbf{0.0425}$\cellcolor{olive!30} & $^\dag0.0432$\cellcolor{olive!29} & $0.0873$\cellcolor{olive!2} & $0.0938$\cellcolor{olive!0} & $0.0393$\cellcolor{olive!24} & $0.0370$\cellcolor{olive!25} & $\textbf{0.0281}$\cellcolor{olive!30} & $^\dag0.0297$\cellcolor{olive!29} \\
 & \texttt{wb} & $0.0804$\cellcolor{olive!0} & $0.0799$\cellcolor{olive!0} & $0.0439$\cellcolor{olive!23} & $\textbf{0.0333}$\cellcolor{olive!30} & $0.0343$\cellcolor{olive!29} & $0.0346$\cellcolor{olive!29} & $0.1001$\cellcolor{olive!0} & $0.0990$\cellcolor{olive!0} & $0.0316$\cellcolor{olive!25} & $\textbf{0.0183}$\cellcolor{olive!30} & $0.0190$\cellcolor{olive!29} & $0.0194$\cellcolor{olive!29} \\
 & \texttt{LeQua2022-T1A} & $0.0933$\cellcolor{olive!7} & $0.1181$\cellcolor{olive!0} & $0.0314$\cellcolor{olive!27} & $^\dag0.0240$\cellcolor{olive!29} & $0.0242$\cellcolor{olive!29} & $\textbf{0.0234}$\cellcolor{olive!30} & $0.1260$\cellcolor{olive!11} & $0.2024$\cellcolor{olive!0} & $0.0170$\cellcolor{olive!28} & $^\dag0.0099$\cellcolor{olive!29} & $^\dag0.0097$\cellcolor{olive!29} & $\textbf{0.0093}$\cellcolor{olive!30} \\
 & \texttt{LeQua2022-T1B} & $0.0140$\cellcolor{olive!16} & $0.0165$\cellcolor{olive!0} & $0.0130$\cellcolor{olive!22} & $^\dag0.0119$\cellcolor{olive!29} & $0.0118$\cellcolor{olive!29} & $\textbf{0.0118}$\cellcolor{olive!30} & $\textbf{91.1316}$\cellcolor{olive!30} & $118.7259$\cellcolor{olive!2} & $121.7791$\cellcolor{olive!0} & $92.8496$\cellcolor{olive!28} & $109.3588$\cellcolor{olive!12} & $109.3062$\cellcolor{olive!12} \\
 & \texttt{LeQua2024-T1} & $0.0862$\cellcolor{olive!9} & $0.1167$\cellcolor{olive!0} & $0.0270$\cellcolor{olive!27} & $^\dag0.0205$\cellcolor{olive!29} & $\textbf{0.0204}$\cellcolor{olive!30} & $^\dag0.0206$\cellcolor{olive!29} & $0.1114$\cellcolor{olive!13} & $0.1978$\cellcolor{olive!0} & $0.0130$\cellcolor{olive!29} & $^\dag0.0078$\cellcolor{olive!29} & $\textbf{0.0078}$\cellcolor{olive!30} & $^\dag0.0079$\cellcolor{olive!29} \\
 & \texttt{LeQua2024-T2} & $0.0171$\cellcolor{olive!11} & $0.0190$\cellcolor{olive!0} & $0.0170$\cellcolor{olive!11} & $\textbf{0.0137}$\cellcolor{olive!30} & $0.0160$\cellcolor{olive!17} & $0.0156$\cellcolor{olive!19} & $137.5911$\cellcolor{olive!21} & $163.6366$\cellcolor{olive!15} & $231.3429$\cellcolor{olive!0} & $\textbf{100.8582}$\cellcolor{olive!30} & $186.1994$\cellcolor{olive!10} & $196.3228$\cellcolor{olive!8} \\
\midrule
\multirow{5}{*}{\begin{sideways}image\end{sideways}} & \texttt{CIFAR100} & $0.0025$\cellcolor{olive!8} & $0.0028$\cellcolor{olive!0} & $0.0022$\cellcolor{olive!15} & $\textbf{0.0016}$\cellcolor{olive!30} & $0.0018$\cellcolor{olive!25} & $0.0017$\cellcolor{olive!28} & $0.1375$\cellcolor{olive!8} & $0.1695$\cellcolor{olive!0} & $0.1009$\cellcolor{olive!18} & $\textbf{0.0576}$\cellcolor{olive!30} & $0.0712$\cellcolor{olive!26} & $0.0643$\cellcolor{olive!28} \\
 & \texttt{CIFAR10} & $0.0065$\cellcolor{olive!10} & $0.0080$\cellcolor{olive!0} & $0.0044$\cellcolor{olive!24} & $\textbf{0.0036}$\cellcolor{olive!30} & $0.0038$\cellcolor{olive!28} & $^\dag0.0036$\cellcolor{olive!29} & $0.0092$\cellcolor{olive!10} & $0.0127$\cellcolor{olive!0} & $0.0037$\cellcolor{olive!26} & $^\dag0.0025$\cellcolor{olive!29} & $0.0027$\cellcolor{olive!29} & $\textbf{0.0024}$\cellcolor{olive!30} \\
 & \texttt{MNIST} & $0.0016$\cellcolor{olive!7} & $0.0018$\cellcolor{olive!0} & $0.0017$\cellcolor{olive!3} & $0.0012$\cellcolor{olive!28} & $0.0012$\cellcolor{olive!26} & $\textbf{0.0011}$\cellcolor{olive!30} & $0.0007$\cellcolor{olive!0} & $0.0006$\cellcolor{olive!4} & $0.0005$\cellcolor{olive!11} & $0.0003$\cellcolor{olive!29} & $0.0003$\cellcolor{olive!25} & $\textbf{0.0003}$\cellcolor{olive!30} \\
 & \texttt{fashionMNIST} & $0.0083$\cellcolor{olive!9} & $0.0103$\cellcolor{olive!0} & $0.0047$\cellcolor{olive!25} & $0.0041$\cellcolor{olive!28} & $\textbf{0.0038}$\cellcolor{olive!30} & $0.0038$\cellcolor{olive!29} & $0.0201$\cellcolor{olive!9} & $0.0274$\cellcolor{olive!0} & $0.0052$\cellcolor{olive!27} & $0.0041$\cellcolor{olive!29} & $\textbf{0.0036}$\cellcolor{olive!30} & $0.0036$\cellcolor{olive!29} \\
 & \texttt{SVHN} & $0.0098$\cellcolor{olive!3} & $0.0106$\cellcolor{olive!0} & $0.0061$\cellcolor{olive!22} & $0.0053$\cellcolor{olive!26} & $0.0049$\cellcolor{olive!28} & $\textbf{0.0045}$\cellcolor{olive!30} & $0.0235$\cellcolor{olive!6} & $0.0284$\cellcolor{olive!0} & $0.0091$\cellcolor{olive!24} & $0.0072$\cellcolor{olive!27} & $0.0060$\cellcolor{olive!28} & $\textbf{0.0052}$\cellcolor{olive!30} \\
\midrule
 & Rank & $4.6$\cellcolor{olive!8} & $5.7$\cellcolor{olive!0} & $4.1$\cellcolor{olive!12} & $2.5$\cellcolor{olive!25} & $2.2$\cellcolor{olive!27} & $\textbf{1.9}$\cellcolor{olive!30} & $4.5$\cellcolor{olive!8} & $5.5$\cellcolor{olive!0} & $4.0$\cellcolor{olive!13} & $2.5$\cellcolor{olive!27} & $2.3$\cellcolor{olive!29} & $\textbf{2.2}$\cellcolor{olive!30} \\
\bottomrule
\end{tabular}

    }%
\end{table}

\begin{table*}[t]
\centering
\caption{Pairwise symbolic matrix of Wilcoxon-Holm comparisons on AE.}
\label{tab:ae-pairwise-symbols}
\begin{tabular}{lcccccc}
\toprule
 & CC & PCC & BBSE/ACC & MLLS/EMQ & KDEy(Gau) & KDEy(Ait-$\lambda$) \\
\midrule
CC & $=$ & $\ll$ & $\gg$ & $\gg$ & $\gg$ & $\gg$ \\
PCC & $\gg$ & $=$ & $\gg$ & $\gg$ & $\gg$ & $\gg$ \\
BBSE/ACC & $\ll$ & $\ll$ & $=$ & $\gg$ & $\gg$ & $\gg$ \\
MLLS/EMQ & $\ll$ & $\ll$ & $\ll$ & $=$ & $\approx$ & $>$ \\
KDEy(Gau) & $\ll$ & $\ll$ & $\ll$ & $\approx$ & $=$ & $\approx$ \\
KDEy(Ait-$\lambda$) & $\ll$ & $\ll$ & $\ll$ & $<$ & $\approx$ & $=$ \\
\bottomrule
\end{tabular}
\end{table*}


\subsection{Full Tables: Performance of uncertainty-aware estimators.}
\label{app:fulltables:bayes}


Table~\ref{tab:bayest1} reports the full AE and $W$ results for the point-estimation performance of Bayesian methods and bootstrap-based baselines. Notational conventions are as in Table~\ref{tab:point}.


\begin{table}[h]
    \caption{Point-estimation performance of Bayesian methods in terms of AE and $W$}
    \label{tab:bayest1}
    \centering
    \resizebox{\textwidth}{!}{%
\begin{tabular}{ll|cccccc|cccccc|}
\toprule
\multicolumn{2}{c}{} & \multicolumn{1}{c}{\begin{sideways}Boots-CC\;\end{sideways}} & \multicolumn{1}{c}{\begin{sideways}Boots-PCC\;\end{sideways}} & \multicolumn{1}{c}{\begin{sideways}Bayes-ACC\;\end{sideways}} & \multicolumn{1}{c}{\begin{sideways}Bayes-EMQ (MAPLS)\;\end{sideways}} & \multicolumn{1}{c}{\begin{sideways}Bayes-KDEy(Gau)\;\end{sideways}} & \multicolumn{1}{c}{\begin{sideways}Bayes-KDEy(Ait-$\lambda$)\;\end{sideways}} & \multicolumn{1}{c}{\begin{sideways}Boots-CC\;\end{sideways}} & \multicolumn{1}{c}{\begin{sideways}Boots-PCC\;\end{sideways}} & \multicolumn{1}{c}{\begin{sideways}Bayes-ACC\;\end{sideways}} & \multicolumn{1}{c}{\begin{sideways}Bayes-EMQ (MAPLS)\;\end{sideways}} & \multicolumn{1}{c}{\begin{sideways}Bayes-KDEy(Gau)\;\end{sideways}} & \multicolumn{1}{c}{\begin{sideways}Bayes-KDEy(Ait-$\lambda$)\;\end{sideways}} \\
\cmidrule(lr){3-8}\cmidrule(lr){9-14}
\multicolumn{2}{c}{} & \multicolumn{6}{c}{AE} & \multicolumn{6}{c}{W} \\
\midrule
\multirow{22}{*}{\begin{sideways}tabular\end{sideways}} & \texttt{abalone} & $0.0450$\cellcolor{olive!28} & $\textbf{0.0424}$\cellcolor{olive!30} & $0.0888$\cellcolor{olive!0} & $0.0555$\cellcolor{olive!21} & $0.0489$\cellcolor{olive!25} & $0.0489$\cellcolor{olive!25} & $7.0638$\cellcolor{olive!29} & $\textbf{6.2735}$\cellcolor{olive!30} & $40.9764$\cellcolor{olive!0} & $13.5245$\cellcolor{olive!23} & $9.5851$\cellcolor{olive!27} & $9.5199$\cellcolor{olive!27} \\
 & \texttt{academic-success} & $0.0953$\cellcolor{olive!7} & $0.1211$\cellcolor{olive!0} & $0.0457$\cellcolor{olive!23} & $0.0254$\cellcolor{olive!29} & $0.0244$\cellcolor{olive!29} & $\textbf{0.0230}$\cellcolor{olive!30} & $0.2542$\cellcolor{olive!12} & $0.4139$\cellcolor{olive!0} & $0.0698$\cellcolor{olive!26} & $0.0220$\cellcolor{olive!29} & $0.0199$\cellcolor{olive!29} & $\textbf{0.0178}$\cellcolor{olive!30} \\
 & \texttt{chess} & $0.0413$\cellcolor{olive!3} & $0.0450$\cellcolor{olive!0} & $0.0280$\cellcolor{olive!17} & $0.0345$\cellcolor{olive!10} & $\textbf{0.0155}$\cellcolor{olive!30} & $0.0161$\cellcolor{olive!29} & $1.9621$\cellcolor{olive!5} & $2.3368$\cellcolor{olive!0} & $0.5852$\cellcolor{olive!25} & $1.1171$\cellcolor{olive!18} & $\textbf{0.3102}$\cellcolor{olive!30} & $0.3897$\cellcolor{olive!28} \\
 & \texttt{cmc} & $0.1560$\cellcolor{olive!7} & $0.1778$\cellcolor{olive!0} & $0.0987$\cellcolor{olive!27} & $0.0945$\cellcolor{olive!28} & $\textbf{0.0911}$\cellcolor{olive!30} & $^\dag0.0912$\cellcolor{olive!29} & $0.3902$\cellcolor{olive!9} & $0.4938$\cellcolor{olive!0} & $^\dag0.1830$\cellcolor{olive!29} & $\textbf{0.1808}$\cellcolor{olive!30} & $^\dag0.1850$\cellcolor{olive!29} & $^\dag0.1872$\cellcolor{olive!29} \\
 & \texttt{connect-4} & $0.1138$\cellcolor{olive!7} & $0.1468$\cellcolor{olive!0} & $0.0266$\cellcolor{olive!28} & $0.0284$\cellcolor{olive!28} & $\textbf{0.0201}$\cellcolor{olive!30} & $^\dag0.0201$\cellcolor{olive!29} & $1.0574$\cellcolor{olive!11} & $1.6643$\cellcolor{olive!0} & $0.0710$\cellcolor{olive!29} & $0.0896$\cellcolor{olive!29} & $^\dag0.0390$\cellcolor{olive!29} & $\textbf{0.0388}$\cellcolor{olive!30} \\
 & \texttt{digits} & $0.0036$\cellcolor{olive!10} & $0.0045$\cellcolor{olive!0} & $0.0031$\cellcolor{olive!15} & $0.0025$\cellcolor{olive!21} & $0.0026$\cellcolor{olive!21} & $\textbf{0.0018}$\cellcolor{olive!30} & $0.0034$\cellcolor{olive!8} & $0.0046$\cellcolor{olive!0} & $0.0018$\cellcolor{olive!21} & $0.0013$\cellcolor{olive!25} & $0.0015$\cellcolor{olive!23} & $\textbf{0.0007}$\cellcolor{olive!30} \\
 & \texttt{dry-bean} & $0.0088$\cellcolor{olive!11} & $0.0122$\cellcolor{olive!0} & $0.0042$\cellcolor{olive!26} & $0.0034$\cellcolor{olive!29} & $0.0034$\cellcolor{olive!29} & $\textbf{0.0032}$\cellcolor{olive!30} & $0.0066$\cellcolor{olive!16} & $0.0131$\cellcolor{olive!0} & $0.0025$\cellcolor{olive!26} & $0.0014$\cellcolor{olive!29} & $0.0013$\cellcolor{olive!29} & $\textbf{0.0012}$\cellcolor{olive!30} \\
 & \texttt{hand\_digits} & $0.0048$\cellcolor{olive!9} & $0.0063$\cellcolor{olive!0} & $0.0031$\cellcolor{olive!21} & $0.0024$\cellcolor{olive!25} & $0.0026$\cellcolor{olive!24} & $\textbf{0.0017}$\cellcolor{olive!30} & $0.0054$\cellcolor{olive!13} & $0.0092$\cellcolor{olive!0} & $0.0018$\cellcolor{olive!25} & $0.0011$\cellcolor{olive!28} & $0.0013$\cellcolor{olive!27} & $\textbf{0.0006}$\cellcolor{olive!30} \\
 & \texttt{image\_seg} & $0.0065$\cellcolor{olive!17} & $0.0084$\cellcolor{olive!6} & $0.0095$\cellcolor{olive!0} & $0.0054$\cellcolor{olive!24} & $\textbf{0.0044}$\cellcolor{olive!30} & $0.0044$\cellcolor{olive!29} & $0.0098$\cellcolor{olive!15} & $0.0164$\cellcolor{olive!0} & $0.0113$\cellcolor{olive!11} & $0.0045$\cellcolor{olive!27} & $\textbf{0.0035}$\cellcolor{olive!30} & $0.0036$\cellcolor{olive!29} \\
 & \texttt{isolet} & $0.0021$\cellcolor{olive!18} & $0.0023$\cellcolor{olive!12} & $0.0028$\cellcolor{olive!0} & $0.0018$\cellcolor{olive!25} & $^\dag0.0017$\cellcolor{olive!29} & $\textbf{0.0017}$\cellcolor{olive!30} & $0.0078$\cellcolor{olive!10} & $0.0093$\cellcolor{olive!3} & $0.0100$\cellcolor{olive!0} & $0.0044$\cellcolor{olive!26} & $^\dag0.0038$\cellcolor{olive!29} & $\textbf{0.0038}$\cellcolor{olive!30} \\
 & \texttt{letter} & $0.0077$\cellcolor{olive!10} & $0.0108$\cellcolor{olive!0} & $0.0047$\cellcolor{olive!19} & $0.0042$\cellcolor{olive!21} & $0.0028$\cellcolor{olive!26} & $\textbf{0.0016}$\cellcolor{olive!30} & $0.0857$\cellcolor{olive!14} & $0.1631$\cellcolor{olive!0} & $0.0282$\cellcolor{olive!25} & $0.0245$\cellcolor{olive!25} & $0.0098$\cellcolor{olive!28} & $\textbf{0.0031}$\cellcolor{olive!30} \\
 & \texttt{mhr} & $0.1216$\cellcolor{olive!5} & $0.1444$\cellcolor{olive!0} & $0.0951$\cellcolor{olive!12} & $0.0533$\cellcolor{olive!22} & $\textbf{0.0242}$\cellcolor{olive!30} & $0.0324$\cellcolor{olive!27} & $0.2154$\cellcolor{olive!6} & $0.2777$\cellcolor{olive!0} & $0.1269$\cellcolor{olive!16} & $0.0420$\cellcolor{olive!26} & $\textbf{0.0092}$\cellcolor{olive!30} & $0.0147$\cellcolor{olive!29} \\
 & \texttt{molecular} & $0.0188$\cellcolor{olive!18} & $0.0388$\cellcolor{olive!0} & $0.0087$\cellcolor{olive!27} & $0.0070$\cellcolor{olive!29} & $0.0069$\cellcolor{olive!29} & $\textbf{0.0066}$\cellcolor{olive!30} & $0.0066$\cellcolor{olive!24} & $0.0294$\cellcolor{olive!0} & $0.0014$\cellcolor{olive!29} & $0.0012$\cellcolor{olive!29} & $0.0011$\cellcolor{olive!29} & $\textbf{0.0010}$\cellcolor{olive!30} \\
 & \texttt{nursery} & $0.0179$\cellcolor{olive!12} & $0.0265$\cellcolor{olive!0} & $\textbf{0.0052}$\cellcolor{olive!30} & $^\dag0.0053$\cellcolor{olive!29} & $0.0061$\cellcolor{olive!28} & $0.0058$\cellcolor{olive!29} & $0.1207$\cellcolor{olive!18} & $0.2875$\cellcolor{olive!0} & $\textbf{0.0138}$\cellcolor{olive!30} & $0.0257$\cellcolor{olive!28} & $0.0340$\cellcolor{olive!27} & $0.0259$\cellcolor{olive!28} \\
 & \texttt{obesity} & $0.0070$\cellcolor{olive!19} & $0.0072$\cellcolor{olive!18} & $0.0100$\cellcolor{olive!0} & $\textbf{0.0054}$\cellcolor{olive!30} & $0.0084$\cellcolor{olive!10} & $0.0085$\cellcolor{olive!9} & $0.0059$\cellcolor{olive!19} & $0.0065$\cellcolor{olive!16} & $0.0106$\cellcolor{olive!0} & $\textbf{0.0033}$\cellcolor{olive!30} & $0.0095$\cellcolor{olive!4} & $0.0096$\cellcolor{olive!4} \\
 & \texttt{page\_block} & $\textbf{0.0141}$\cellcolor{olive!30} & $0.0272$\cellcolor{olive!0} & $0.0255$\cellcolor{olive!3} & $0.0169$\cellcolor{olive!23} & $0.0189$\cellcolor{olive!18} & $0.0209$\cellcolor{olive!14} & $\textbf{0.3351}$\cellcolor{olive!30} & $1.4162$\cellcolor{olive!11} & $2.1035$\cellcolor{olive!0} & $0.5741$\cellcolor{olive!25} & $1.1079$\cellcolor{olive!16} & $1.4118$\cellcolor{olive!11} \\
 & \texttt{phishing} & $0.0953$\cellcolor{olive!3} & $0.1062$\cellcolor{olive!0} & $0.0503$\cellcolor{olive!19} & $0.0260$\cellcolor{olive!28} & $\textbf{0.0213}$\cellcolor{olive!30} & $0.0234$\cellcolor{olive!29} & $1.3672$\cellcolor{olive!1} & $1.4132$\cellcolor{olive!0} & $0.3629$\cellcolor{olive!23} & $0.0926$\cellcolor{olive!29} & $\textbf{0.0664}$\cellcolor{olive!30} & $0.0802$\cellcolor{olive!29} \\
 & \texttt{satellite} & $0.0231$\cellcolor{olive!10} & $0.0312$\cellcolor{olive!0} & $0.0111$\cellcolor{olive!24} & $0.0075$\cellcolor{olive!29} & $\textbf{0.0070}$\cellcolor{olive!30} & $0.0079$\cellcolor{olive!28} & $0.0768$\cellcolor{olive!14} & $0.1463$\cellcolor{olive!0} & $0.0185$\cellcolor{olive!27} & $0.0073$\cellcolor{olive!29} & $\textbf{0.0058}$\cellcolor{olive!30} & $0.0080$\cellcolor{olive!29} \\
 & \texttt{shuttle} & $0.0073$\cellcolor{olive!8} & $0.0093$\cellcolor{olive!0} & $^\dag0.0022$\cellcolor{olive!29} & $0.0057$\cellcolor{olive!14} & $0.0022$\cellcolor{olive!29} & $\textbf{0.0021}$\cellcolor{olive!30} & $0.0028$\cellcolor{olive!11} & $0.0043$\cellcolor{olive!0} & $^\dag0.0004$\cellcolor{olive!29} & $0.0016$\cellcolor{olive!20} & $0.0004$\cellcolor{olive!29} & $\textbf{0.0003}$\cellcolor{olive!30} \\
 & \texttt{waveform-v1} & $0.0416$\cellcolor{olive!9} & $0.0582$\cellcolor{olive!0} & $0.0240$\cellcolor{olive!20} & $\textbf{0.0072}$\cellcolor{olive!30} & $0.0100$\cellcolor{olive!28} & $^\dag0.0073$\cellcolor{olive!29} & $0.0234$\cellcolor{olive!14} & $0.0436$\cellcolor{olive!0} & $0.0086$\cellcolor{olive!24} & $\textbf{0.0007}$\cellcolor{olive!30} & $0.0015$\cellcolor{olive!29} & $^\dag0.0008$\cellcolor{olive!29} \\
 & \texttt{wine-quality} & $0.0998$\cellcolor{olive!8} & $0.1144$\cellcolor{olive!0} & $\textbf{0.0618}$\cellcolor{olive!30} & $0.0662$\cellcolor{olive!27} & $0.1017$\cellcolor{olive!7} & $0.0714$\cellcolor{olive!24} & $5.4516$\cellcolor{olive!10} & $7.6112$\cellcolor{olive!0} & $^\dag1.6117$\cellcolor{olive!29} & $\textbf{1.5160}$\cellcolor{olive!30} & $6.4950$\cellcolor{olive!5} & $3.8283$\cellcolor{olive!18} \\
 & \texttt{yeast} & $0.0539$\cellcolor{olive!13} & $0.0557$\cellcolor{olive!10} & $0.0476$\cellcolor{olive!23} & $0.0623$\cellcolor{olive!0} & $0.0446$\cellcolor{olive!28} & $\textbf{0.0434}$\cellcolor{olive!30} & $7.5742$\cellcolor{olive!14} & $6.3076$\cellcolor{olive!22} & $5.5444$\cellcolor{olive!26} & $10.1276$\cellcolor{olive!0} & $5.7306$\cellcolor{olive!25} & $\textbf{4.9885}$\cellcolor{olive!30} \\
\midrule
\multirow{15}{*}{\begin{sideways}text\end{sideways}} & \texttt{gasp} & $0.0900$\cellcolor{olive!11} & $0.1199$\cellcolor{olive!0} & $0.0503$\cellcolor{olive!26} & $0.0446$\cellcolor{olive!28} & $^\dag0.0415$\cellcolor{olive!29} & $\textbf{0.0409}$\cellcolor{olive!30} & $0.6272$\cellcolor{olive!16} & $1.2870$\cellcolor{olive!0} & $0.1595$\cellcolor{olive!28} & $\textbf{0.1084}$\cellcolor{olive!30} & $^\dag0.1230$\cellcolor{olive!29} & $^\dag0.1134$\cellcolor{olive!29} \\
 & \texttt{hcr} & $0.1265$\cellcolor{olive!10} & $0.1524$\cellcolor{olive!0} & $0.0840$\cellcolor{olive!27} & $0.0848$\cellcolor{olive!27} & $0.0809$\cellcolor{olive!28} & $\textbf{0.0780}$\cellcolor{olive!30} & $0.3778$\cellcolor{olive!9} & $0.4836$\cellcolor{olive!0} & $\textbf{0.1616}$\cellcolor{olive!30} & $^\dag0.1729$\cellcolor{olive!28} & $^\dag0.1781$\cellcolor{olive!28} & $^\dag0.1662$\cellcolor{olive!29} \\
 & \texttt{omd} & $0.1062$\cellcolor{olive!6} & $0.1181$\cellcolor{olive!0} & $0.0825$\cellcolor{olive!18} & $0.0725$\cellcolor{olive!23} & $\textbf{0.0609}$\cellcolor{olive!30} & $^\dag0.0624$\cellcolor{olive!29} & $0.1845$\cellcolor{olive!9} & $0.2356$\cellcolor{olive!0} & $0.1176$\cellcolor{olive!20} & $0.0952$\cellcolor{olive!24} & $\textbf{0.0670}$\cellcolor{olive!30} & $^\dag0.0749$\cellcolor{olive!28} \\
 & \texttt{sanders} & $0.1114$\cellcolor{olive!7} & $0.1341$\cellcolor{olive!0} & $0.0534$\cellcolor{olive!26} & $\textbf{0.0415}$\cellcolor{olive!30} & $0.0454$\cellcolor{olive!28} & $0.0473$\cellcolor{olive!28} & $0.5176$\cellcolor{olive!7} & $0.6638$\cellcolor{olive!0} & $0.1101$\cellcolor{olive!27} & $\textbf{0.0696}$\cellcolor{olive!30} & $0.0841$\cellcolor{olive!29} & $0.0942$\cellcolor{olive!28} \\
 & \texttt{semeval13} & $0.1115$\cellcolor{olive!8} & $0.1364$\cellcolor{olive!0} & $0.0653$\cellcolor{olive!25} & $0.0722$\cellcolor{olive!23} & $^\dag0.0536$\cellcolor{olive!29} & $\textbf{0.0529}$\cellcolor{olive!30} & $0.2692$\cellcolor{olive!15} & $0.5069$\cellcolor{olive!0} & $0.0869$\cellcolor{olive!27} & $0.1059$\cellcolor{olive!26} & $0.0566$\cellcolor{olive!29} & $\textbf{0.0536}$\cellcolor{olive!30} \\
 & \texttt{semeval14} & $0.1039$\cellcolor{olive!10} & $0.1332$\cellcolor{olive!0} & $0.0587$\cellcolor{olive!27} & $0.0659$\cellcolor{olive!25} & $0.0536$\cellcolor{olive!29} & $\textbf{0.0525}$\cellcolor{olive!30} & $0.2100$\cellcolor{olive!19} & $0.4723$\cellcolor{olive!0} & $\textbf{0.0678}$\cellcolor{olive!30} & $0.1075$\cellcolor{olive!27} & $^\dag0.0727$\cellcolor{olive!29} & $^\dag0.0721$\cellcolor{olive!29} \\
 & \texttt{semeval15} & $0.1254$\cellcolor{olive!8} & $0.1425$\cellcolor{olive!0} & $0.0900$\cellcolor{olive!26} & $0.0848$\cellcolor{olive!28} & $\textbf{0.0823}$\cellcolor{olive!30} & $0.0862$\cellcolor{olive!28} & $0.2884$\cellcolor{olive!14} & $0.4532$\cellcolor{olive!0} & $0.1458$\cellcolor{olive!27} & $\textbf{0.1121}$\cellcolor{olive!30} & $0.1374$\cellcolor{olive!27} & $0.1505$\cellcolor{olive!26} \\
 & \texttt{semeval16} & $0.1389$\cellcolor{olive!12} & $0.1611$\cellcolor{olive!0} & $0.1261$\cellcolor{olive!19} & $^\dag0.1101$\cellcolor{olive!28} & $\textbf{0.1082}$\cellcolor{olive!30} & $^\dag0.1087$\cellcolor{olive!29} & $0.5126$\cellcolor{olive!18} & $0.6855$\cellcolor{olive!0} & $0.6068$\cellcolor{olive!8} & $\textbf{0.3997}$\cellcolor{olive!30} & $0.4494$\cellcolor{olive!24} & $0.4573$\cellcolor{olive!23} \\
 & \texttt{sst} & $0.1050$\cellcolor{olive!10} & $0.1352$\cellcolor{olive!0} & $0.0702$\cellcolor{olive!23} & $0.0562$\cellcolor{olive!28} & $0.0532$\cellcolor{olive!29} & $\textbf{0.0522}$\cellcolor{olive!30} & $0.1685$\cellcolor{olive!13} & $0.2659$\cellcolor{olive!0} & $0.0826$\cellcolor{olive!24} & $0.0509$\cellcolor{olive!29} & $0.0465$\cellcolor{olive!29} & $\textbf{0.0450}$\cellcolor{olive!30} \\
 & \texttt{wa} & $0.0784$\cellcolor{olive!2} & $0.0819$\cellcolor{olive!0} & $0.0500$\cellcolor{olive!24} & $0.0484$\cellcolor{olive!26} & $\textbf{0.0432}$\cellcolor{olive!30} & $^\dag0.0445$\cellcolor{olive!29} & $0.0874$\cellcolor{olive!2} & $0.0938$\cellcolor{olive!0} & $0.0368$\cellcolor{olive!26} & $0.0339$\cellcolor{olive!27} & $\textbf{0.0288}$\cellcolor{olive!30} & $^\dag0.0306$\cellcolor{olive!29} \\
 & \texttt{wb} & $0.0804$\cellcolor{olive!0} & $0.0799$\cellcolor{olive!0} & $0.0432$\cellcolor{olive!24} & $^\dag0.0356$\cellcolor{olive!29} & $\textbf{0.0356}$\cellcolor{olive!30} & $^\dag0.0359$\cellcolor{olive!29} & $0.1002$\cellcolor{olive!0} & $0.0990$\cellcolor{olive!0} & $0.0312$\cellcolor{olive!25} & $^\dag0.0204$\cellcolor{olive!29} & $\textbf{0.0202}$\cellcolor{olive!30} & $0.0205$\cellcolor{olive!29} \\
 & \texttt{LeQua2022-T1A} & $0.0933$\cellcolor{olive!7} & $0.1181$\cellcolor{olive!0} & $0.0326$\cellcolor{olive!27} & $^\dag0.0247$\cellcolor{olive!29} & $^\dag0.0243$\cellcolor{olive!29} & $\textbf{0.0242}$\cellcolor{olive!30} & $0.1261$\cellcolor{olive!11} & $0.2024$\cellcolor{olive!0} & $0.0180$\cellcolor{olive!28} & $^\dag0.0101$\cellcolor{olive!29} & $^\dag0.0098$\cellcolor{olive!29} & $\textbf{0.0097}$\cellcolor{olive!30} \\
 & \texttt{LeQua2022-T1B} & $0.0140$\cellcolor{olive!13} & $0.0165$\cellcolor{olive!0} & $0.0122$\cellcolor{olive!22} & $0.0119$\cellcolor{olive!24} & $0.0109$\cellcolor{olive!29} & $\textbf{0.0109}$\cellcolor{olive!30} & $^\dag91.1317$\cellcolor{olive!29} & $118.7730$\cellcolor{olive!22} & $199.2715$\cellcolor{olive!0} & $\textbf{89.9244}$\cellcolor{olive!30} & $99.0783$\cellcolor{olive!27} & $99.0695$\cellcolor{olive!27} \\
 & \texttt{LeQua2024-T1} & $0.0862$\cellcolor{olive!9} & $0.1167$\cellcolor{olive!0} & $0.0265$\cellcolor{olive!28} & $0.0218$\cellcolor{olive!29} & $^\dag0.0207$\cellcolor{olive!29} & $\textbf{0.0204}$\cellcolor{olive!30} & $0.1114$\cellcolor{olive!13} & $0.1978$\cellcolor{olive!0} & $0.0125$\cellcolor{olive!29} & $0.0085$\cellcolor{olive!29} & $^\dag0.0079$\cellcolor{olive!29} & $\textbf{0.0078}$\cellcolor{olive!30} \\
 & \texttt{LeQua2024-T2} & $0.0171$\cellcolor{olive!10} & $0.0190$\cellcolor{olive!0} & $0.0145$\cellcolor{olive!25} & $\textbf{0.0136}$\cellcolor{olive!30} & $^\dag0.0138$\cellcolor{olive!29} & $0.0139$\cellcolor{olive!28} & $137.5365$\cellcolor{olive!19} & $163.6220$\cellcolor{olive!12} & $206.9347$\cellcolor{olive!0} & $\textbf{101.0987}$\cellcolor{olive!30} & $152.2147$\cellcolor{olive!15} & $171.8516$\cellcolor{olive!9} \\
\midrule
\multirow{5}{*}{\begin{sideways}image\end{sideways}} & \texttt{CIFAR100} & $0.0025$\cellcolor{olive!8} & $0.0028$\cellcolor{olive!0} & $0.0028$\cellcolor{olive!1} & $\textbf{0.0017}$\cellcolor{olive!30} & $0.0018$\cellcolor{olive!27} & $0.0017$\cellcolor{olive!29} & $0.1375$\cellcolor{olive!10} & $0.1695$\cellcolor{olive!2} & $0.1797$\cellcolor{olive!0} & $\textbf{0.0602}$\cellcolor{olive!30} & $0.0692$\cellcolor{olive!27} & $0.0659$\cellcolor{olive!28} \\
 & \texttt{CIFAR10} & $0.0065$\cellcolor{olive!9} & $0.0080$\cellcolor{olive!0} & $0.0046$\cellcolor{olive!23} & $0.0041$\cellcolor{olive!26} & $0.0043$\cellcolor{olive!25} & $\textbf{0.0036}$\cellcolor{olive!30} & $0.0092$\cellcolor{olive!10} & $0.0127$\cellcolor{olive!0} & $0.0039$\cellcolor{olive!25} & $0.0031$\cellcolor{olive!28} & $0.0034$\cellcolor{olive!26} & $\textbf{0.0024}$\cellcolor{olive!30} \\
 & \texttt{MNIST} & $\textbf{0.0017}$\cellcolor{olive!30} & $0.0018$\cellcolor{olive!25} & $0.0019$\cellcolor{olive!19} & $0.0021$\cellcolor{olive!0} & $0.0020$\cellcolor{olive!5} & $0.0019$\cellcolor{olive!13} & $^\dag0.0007$\cellcolor{olive!10} & $\textbf{0.0007}$\cellcolor{olive!30} & $0.0007$\cellcolor{olive!18} & $0.0007$\cellcolor{olive!16} & $0.0008$\cellcolor{olive!0} & $0.0007$\cellcolor{olive!26} \\
 & \texttt{fashionMNIST} & $0.0083$\cellcolor{olive!10} & $0.0104$\cellcolor{olive!0} & $0.0049$\cellcolor{olive!27} & $0.0046$\cellcolor{olive!28} & $0.0044$\cellcolor{olive!29} & $\textbf{0.0044}$\cellcolor{olive!30} & $0.0201$\cellcolor{olive!9} & $0.0274$\cellcolor{olive!0} & $0.0056$\cellcolor{olive!28} & $0.0049$\cellcolor{olive!29} & $^\dag0.0042$\cellcolor{olive!29} & $\textbf{0.0042}$\cellcolor{olive!30} \\
 & \texttt{SVHN} & $0.0099$\cellcolor{olive!3} & $0.0106$\cellcolor{olive!0} & $0.0061$\cellcolor{olive!23} & $0.0056$\cellcolor{olive!25} & $0.0052$\cellcolor{olive!27} & $\textbf{0.0047}$\cellcolor{olive!30} & $0.0235$\cellcolor{olive!6} & $0.0284$\cellcolor{olive!0} & $0.0088$\cellcolor{olive!25} & $0.0073$\cellcolor{olive!27} & $0.0067$\cellcolor{olive!28} & $\textbf{0.0055}$\cellcolor{olive!30} \\
\midrule
 & Rank & $4.6$\cellcolor{olive!8} & $5.6$\cellcolor{olive!0} & $3.9$\cellcolor{olive!13} & $2.9$\cellcolor{olive!21} & $2.1$\cellcolor{olive!27} & $\textbf{1.9}$\cellcolor{olive!30} & $4.5$\cellcolor{olive!8} & $5.5$\cellcolor{olive!0} & $3.9$\cellcolor{olive!14} & $2.6$\cellcolor{olive!25} & $2.5$\cellcolor{olive!26} & $\textbf{2.1}$\cellcolor{olive!30} \\
\bottomrule
\end{tabular}

    }%
\end{table}


Tables~\ref{tab:coverage} and~\ref{tab:amplitude} report coverage (hard, \(\hCov\), and soft, \(\sCov\)) and amplitude results for 95\% posterior intervals constructed from the 1000 samples generated by each method, with and without temperature calibration. In Table~\ref{tab:coverage}, highlighted cells indicate coverage values within the range \(95\%\pm5\%\), whereas in Table~\ref{tab:amplitude}, highlighted cells indicate amplitude values covering at most \(5\%\) of the total simplex volume. As expected, temperature calibration often improves coverage at the expense of larger amplitude. Since the temperature grid includes one value below the default \(T=1\), namely \(T=\frac{1}{2}\), calibration may in a few cases (e.g., datasets \texttt{wa} and \texttt{wb}) lead to lower coverage and smaller amplitude. Bootstrap methods are unaffected by temperature calibration.

\begin{table}[h]
    \caption{Coverage before and after temperature calibration. Left: hard coverage; right: soft coverage. Highlighted cells indicate values within the range $95\%\pm5\%$.}
    \label{tab:coverage}
    \centering
    \resizebox{.48\textwidth}{!}{%
\begin{tabular}{ll|cc|cccc|cccc|}
\toprule
\multicolumn{2}{c}{} & \multicolumn{1}{c}{\begin{sideways}Boots-CC\;\end{sideways}} & \multicolumn{1}{c}{\begin{sideways}Boots-PCC\;\end{sideways}} & \multicolumn{1}{c}{\begin{sideways}Bayes-ACC\;\end{sideways}} & \multicolumn{1}{c}{\begin{sideways}Bayes-EMQ (MAPLS)\;\end{sideways}} & \multicolumn{1}{c}{\begin{sideways}Bayes-KDEy(Gau)\;\end{sideways}} & \multicolumn{1}{c}{\begin{sideways}Bayes-KDEy(Ait-$\lambda$)\;\end{sideways}} & \multicolumn{1}{c}{\begin{sideways}Bayes-ACC$^{cal}$\;\end{sideways}} & \multicolumn{1}{c}{\begin{sideways}Bayes-EMQ (MAPLS)$^{cal}$\;\end{sideways}} & \multicolumn{1}{c}{\begin{sideways}Bayes-KDEy(Gau)$^{cal}$\;\end{sideways}} & \multicolumn{1}{c}{\begin{sideways}Bayes-KDEy(Ait-$\lambda$)$^{cal}$\;\end{sideways}} \\
\cmidrule(lr){3-4}\cmidrule(lr){5-8}\cmidrule(lr){9-12}
\multicolumn{2}{c}{} & \multicolumn{2}{c}{} & \multicolumn{4}{c}{Temperature=1} & \multicolumn{4}{c}{Calibrated Temperature} \\
\midrule
\multirow{22}{*}{\begin{sideways}tabular\end{sideways}} & \texttt{abalone} & $0$ & $0$ & $2$ & $8$ & $0$ & $0$ & $86$ & $51$ & $83$ & $85$ \\
 & \texttt{academic-success} & $4$ & $1$ & $73$ & $82$ & $88$ & $89$ & $42$ & $95$\cellcolor{olive!25} & $88$ & $89$ \\
 & \texttt{chess} & $0$ & $0$ & $91$\cellcolor{olive!25} & $35$ & $60$ & $49$ & $93$\cellcolor{olive!25} & $57$ & $84$ & $78$ \\
 & \texttt{cmc} & $1$ & $0$ & $79$ & $20$ & $24$ & $27$ & $79$ & $49$ & $53$ & $50$ \\
 & \texttt{connect-4} & $3$ & $0$ & $96$\cellcolor{olive!25} & $90$\cellcolor{olive!25} & $96$\cellcolor{olive!25} & $97$\cellcolor{olive!25} & $96$\cellcolor{olive!25} & $90$\cellcolor{olive!25} & $96$\cellcolor{olive!25} & $97$\cellcolor{olive!25} \\
 & \texttt{digits} & $88$ & $79$ & $95$\cellcolor{olive!25} & $94$\cellcolor{olive!25} & $95$\cellcolor{olive!25} & $95$\cellcolor{olive!25} & $94$\cellcolor{olive!25} & $94$\cellcolor{olive!25} & $94$\cellcolor{olive!25} & $95$\cellcolor{olive!25} \\
 & \texttt{dry-bean} & $54$ & $20$ & $98$\cellcolor{olive!25} & $95$\cellcolor{olive!25} & $98$\cellcolor{olive!25} & $98$\cellcolor{olive!25} & $96$\cellcolor{olive!25} & $96$\cellcolor{olive!25} & $97$\cellcolor{olive!25} & $97$\cellcolor{olive!25} \\
 & \texttt{hand\_digits} & $78$ & $41$ & $95$\cellcolor{olive!25} & $92$\cellcolor{olive!25} & $95$\cellcolor{olive!25} & $95$\cellcolor{olive!25} & $93$\cellcolor{olive!25} & $92$\cellcolor{olive!25} & $93$\cellcolor{olive!25} & $95$\cellcolor{olive!25} \\
 & \texttt{image\_seg} & $61$ & $44$ & $85$ & $82$ & $89$ & $89$\cellcolor{olive!25} & $45$ & $56$ & $74$ & $74$ \\
 & \texttt{isolet} & $88$ & $55$ & $75$ & $71$ & $75$ & $75$ & $73$ & $66$ & $74$ & $72$ \\
 & \texttt{letter} & $5$ & $0$ & $74$ & $71$ & $73$ & $75$ & $63$ & $60$ & $62$ & $75$ \\
 & \texttt{mhr} & $3$ & $0$ & $67$ & $47$ & $87$ & $69$ & $67$ & $93$\cellcolor{olive!25} & $100$\cellcolor{olive!25} & $100$\cellcolor{olive!25} \\
 & \texttt{molecular} & $55$ & $21$ & $99$\cellcolor{olive!25} & $99$\cellcolor{olive!25} & $100$\cellcolor{olive!25} & $100$\cellcolor{olive!25} & $99$\cellcolor{olive!25} & $99$\cellcolor{olive!25} & $100$\cellcolor{olive!25} & $100$\cellcolor{olive!25} \\
 & \texttt{nursery} & $42$ & $21$ & $99$\cellcolor{olive!25} & $97$\cellcolor{olive!25} & $99$\cellcolor{olive!25} & $99$\cellcolor{olive!25} & $97$\cellcolor{olive!25} & $96$\cellcolor{olive!25} & $96$\cellcolor{olive!25} & $95$\cellcolor{olive!25} \\
 & \texttt{obesity} & $73$ & $65$ & $96$\cellcolor{olive!25} & $80$ & $69$ & $67$ & $66$ & $84$ & $69$ & $67$ \\
 & \texttt{page\_block} & $66$ & $16$ & $69$ & $64$ & $43$ & $35$ & $30$ & $64$ & $64$ & $53$ \\
 & \texttt{phishing} & $3$ & $1$ & $82$ & $75$ & $78$ & $75$ & $45$ & $91$\cellcolor{olive!25} & $92$\cellcolor{olive!25} & $91$\cellcolor{olive!25} \\
 & \texttt{satellite} & $12$ & $3$ & $97$\cellcolor{olive!25} & $92$\cellcolor{olive!25} & $98$\cellcolor{olive!25} & $90$\cellcolor{olive!25} & $97$\cellcolor{olive!25} & $93$\cellcolor{olive!25} & $98$\cellcolor{olive!25} & $90$\cellcolor{olive!25} \\
 & \texttt{shuttle} & $92$\cellcolor{olive!25} & $74$ & $100$\cellcolor{olive!25} & $98$\cellcolor{olive!25} & $100$\cellcolor{olive!25} & $100$\cellcolor{olive!25} & $100$\cellcolor{olive!25} & $97$\cellcolor{olive!25} & $99$\cellcolor{olive!25} & $100$\cellcolor{olive!25} \\
 & \texttt{waveform-v1} & $16$ & $6$ & $87$ & $98$\cellcolor{olive!25} & $99$\cellcolor{olive!25} & $100$\cellcolor{olive!25} & $54$ & $98$\cellcolor{olive!25} & $90$\cellcolor{olive!25} & $96$\cellcolor{olive!25} \\
 & \texttt{wine-quality} & $0$ & $0$ & $91$\cellcolor{olive!25} & $45$ & $28$ & $14$ & $91$\cellcolor{olive!25} & $85$ & $86$ & $95$\cellcolor{olive!25} \\
 & \texttt{yeast} & $0$ & $0$ & $63$ & $0$ & $0$ & $0$ & $27$ & $34$ & $92$\cellcolor{olive!25} & $92$\cellcolor{olive!25} \\
\midrule
\multirow{15}{*}{\begin{sideways}text\end{sideways}} & \texttt{gasp} & $37$ & $11$ & $96$\cellcolor{olive!25} & $90$\cellcolor{olive!25} & $95$\cellcolor{olive!25} & $96$\cellcolor{olive!25} & $96$\cellcolor{olive!25} & $79$ & $87$ & $87$ \\
 & \texttt{hcr} & $19$ & $3$ & $92$\cellcolor{olive!25} & $58$ & $76$ & $80$ & $97$\cellcolor{olive!25} & $86$ & $89$\cellcolor{olive!25} & $91$\cellcolor{olive!25} \\
 & \texttt{omd} & $28$ & $22$ & $92$\cellcolor{olive!25} & $67$ & $94$\cellcolor{olive!25} & $92$\cellcolor{olive!25} & $76$ & $67$ & $84$ & $79$ \\
 & \texttt{sanders} & $25$ & $4$ & $96$\cellcolor{olive!25} & $95$\cellcolor{olive!25} & $93$\cellcolor{olive!25} & $91$\cellcolor{olive!25} & $96$\cellcolor{olive!25} & $95$\cellcolor{olive!25} & $93$\cellcolor{olive!25} & $91$\cellcolor{olive!25} \\
 & \texttt{semeval13} & $28$ & $12$ & $91$\cellcolor{olive!25} & $76$ & $93$\cellcolor{olive!25} & $93$\cellcolor{olive!25} & $91$\cellcolor{olive!25} & $76$ & $93$\cellcolor{olive!25} & $93$\cellcolor{olive!25} \\
 & \texttt{semeval14} & $28$ & $13$ & $94$\cellcolor{olive!25} & $77$ & $94$\cellcolor{olive!25} & $95$\cellcolor{olive!25} & $94$\cellcolor{olive!25} & $77$ & $94$\cellcolor{olive!25} & $95$\cellcolor{olive!25} \\
 & \texttt{semeval15} & $21$ & $12$ & $73$ & $61$ & $76$ & $74$ & $73$ & $61$ & $76$ & $74$ \\
 & \texttt{semeval16} & $18$ & $4$ & $88$ & $64$ & $87$ & $87$ & $93$\cellcolor{olive!25} & $84$ & $87$ & $87$ \\
 & \texttt{sst} & $28$ & $15$ & $92$\cellcolor{olive!25} & $86$ & $92$\cellcolor{olive!25} & $93$\cellcolor{olive!25} & $92$\cellcolor{olive!25} & $86$ & $92$\cellcolor{olive!25} & $93$\cellcolor{olive!25} \\
 & \texttt{wa} & $43$ & $34$ & $91$\cellcolor{olive!25} & $88$ & $90$\cellcolor{olive!25} & $91$\cellcolor{olive!25} & $78$ & $73$ & $75$ & $77$ \\
 & \texttt{wb} & $43$ & $37$ & $95$\cellcolor{olive!25} & $94$\cellcolor{olive!25} & $96$\cellcolor{olive!25} & $96$\cellcolor{olive!25} & $95$\cellcolor{olive!25} & $88$ & $88$ & $89$ \\
 & \texttt{LeQua2022-T1A} & $32$ & $21$ & $96$\cellcolor{olive!25} & $96$\cellcolor{olive!25} & $98$\cellcolor{olive!25} & $97$\cellcolor{olive!25} & $96$\cellcolor{olive!25} & $89$\cellcolor{olive!25} & $93$\cellcolor{olive!25} & $91$\cellcolor{olive!25} \\
 & \texttt{LeQua2022-T1B} & $0$ & $0$ & $42$ & $1$ & $11$ & $13$ & $42$ & $22$ & $42$ & $41$ \\
 & \texttt{LeQua2024-T1} & $34$ & $18$ & $98$\cellcolor{olive!25} & $95$\cellcolor{olive!25} & $98$\cellcolor{olive!25} & $98$\cellcolor{olive!25} & $90$\cellcolor{olive!25} & $89$ & $91$\cellcolor{olive!25} & $92$\cellcolor{olive!25} \\
 & \texttt{LeQua2024-T2} & $0$ & $0$ & $40$ & $1$ & $17$ & $11$ & $40$ & $17$ & $42$ & $39$ \\
\midrule
\multirow{5}{*}{\begin{sideways}image\end{sideways}} & \texttt{CIFAR100} & $0$ & $0$ & $8$ & $4$ & $7$ & $6$ & $4$ & $4$ & $7$ & $6$ \\
 & \texttt{CIFAR10} & $90$\cellcolor{olive!25} & $54$ & $90$\cellcolor{olive!25} & $88$ & $90$\cellcolor{olive!25} & $90$\cellcolor{olive!25} & $88$ & $88$ & $89$ & $90$\cellcolor{olive!25} \\
 & \texttt{MNIST} & $100$\cellcolor{olive!25} & $90$\cellcolor{olive!25} & $90$\cellcolor{olive!25} & $90$\cellcolor{olive!25} & $90$\cellcolor{olive!25} & $90$\cellcolor{olive!25} & $90$\cellcolor{olive!25} & $90$\cellcolor{olive!25} & $90$\cellcolor{olive!25} & $90$\cellcolor{olive!25} \\
 & \texttt{fashionMNIST} & $65$ & $29$ & $90$\cellcolor{olive!25} & $86$ & $90$\cellcolor{olive!25} & $90$\cellcolor{olive!25} & $86$ & $82$ & $88$ & $88$ \\
 & \texttt{SVHN} & $71$ & $36$ & $89$\cellcolor{olive!25} & $89$ & $89$\cellcolor{olive!25} & $90$\cellcolor{olive!25} & $83$ & $88$ & $86$ & $89$ \\
\midrule
 & Mean & $35$ & $21$ & $82$ & $70$ & $76$ & $74$ & $77$ & $76$ & $83$ & $83$ \\
\bottomrule
\end{tabular}

    }%
    \resizebox{.48\textwidth}{!}{%
\begin{tabular}{ll|cc|cccc|cccc|}
\toprule
\multicolumn{2}{c}{} & \multicolumn{1}{c}{\begin{sideways}Boots-CC\;\end{sideways}} & \multicolumn{1}{c}{\begin{sideways}Boots-PCC\;\end{sideways}} & \multicolumn{1}{c}{\begin{sideways}Bayes-ACC\;\end{sideways}} & \multicolumn{1}{c}{\begin{sideways}Bayes-EMQ (MAPLS)\;\end{sideways}} & \multicolumn{1}{c}{\begin{sideways}Bayes-KDEy(Gau)\;\end{sideways}} & \multicolumn{1}{c}{\begin{sideways}Bayes-KDEy(Ait-$\lambda$)\;\end{sideways}} & \multicolumn{1}{c}{\begin{sideways}Bayes-ACC$^{cal}$\;\end{sideways}} & \multicolumn{1}{c}{\begin{sideways}Bayes-EMQ (MAPLS)$^{cal}$\;\end{sideways}} & \multicolumn{1}{c}{\begin{sideways}Bayes-KDEy(Gau)$^{cal}$\;\end{sideways}} & \multicolumn{1}{c}{\begin{sideways}Bayes-KDEy(Ait-$\lambda$)$^{cal}$\;\end{sideways}} \\
\cmidrule(lr){3-4}\cmidrule(lr){5-8}\cmidrule(lr){9-12}
\multicolumn{2}{c}{} & \multicolumn{2}{c}{} & \multicolumn{4}{c}{Temperature=1} & \multicolumn{4}{c}{Calibrated Temperature} \\
\midrule
\multirow{22}{*}{\begin{sideways}tabular\end{sideways}} & \texttt{abalone} & $21$ & $8$ & $58$ & $66$ & $34$ & $45$ & $95$\cellcolor{olive!25} & $88$ & $93$\cellcolor{olive!25} & $94$\cellcolor{olive!25} \\
 & \texttt{academic-success} & $17$ & $8$ & $75$ & $85$ & $87$ & $89$ & $53$ & $96$\cellcolor{olive!25} & $87$ & $89$ \\
 & \texttt{chess} & $23$ & $9$ & $96$\cellcolor{olive!25} & $79$ & $88$ & $83$ & $97$\cellcolor{olive!25} & $87$ & $95$\cellcolor{olive!25} & $93$\cellcolor{olive!25} \\
 & \texttt{cmc} & $10$ & $2$ & $80$ & $44$ & $48$ & $50$ & $80$ & $61$ & $65$ & $63$ \\
 & \texttt{connect-4} & $14$ & $6$ & $96$\cellcolor{olive!25} & $89$\cellcolor{olive!25} & $96$\cellcolor{olive!25} & $96$\cellcolor{olive!25} & $96$\cellcolor{olive!25} & $89$\cellcolor{olive!25} & $96$\cellcolor{olive!25} & $96$\cellcolor{olive!25} \\
 & \texttt{digits} & $97$\cellcolor{olive!25} & $95$\cellcolor{olive!25} & $99$\cellcolor{olive!25} & $98$\cellcolor{olive!25} & $99$\cellcolor{olive!25} & $99$\cellcolor{olive!25} & $98$\cellcolor{olive!25} & $99$\cellcolor{olive!25} & $99$\cellcolor{olive!25} & $99$\cellcolor{olive!25} \\
 & \texttt{dry-bean} & $86$ & $72$ & $100$\cellcolor{olive!25} & $99$\cellcolor{olive!25} & $100$\cellcolor{olive!25} & $100$\cellcolor{olive!25} & $98$\cellcolor{olive!25} & $98$\cellcolor{olive!25} & $98$\cellcolor{olive!25} & $98$\cellcolor{olive!25} \\
 & \texttt{hand\_digits} & $94$\cellcolor{olive!25} & $86$ & $99$\cellcolor{olive!25} & $98$\cellcolor{olive!25} & $99$\cellcolor{olive!25} & $99$\cellcolor{olive!25} & $98$\cellcolor{olive!25} & $98$\cellcolor{olive!25} & $99$\cellcolor{olive!25} & $99$\cellcolor{olive!25} \\
 & \texttt{image\_seg} & $89$ & $83$ & $91$\cellcolor{olive!25} & $92$\cellcolor{olive!25} & $96$\cellcolor{olive!25} & $96$\cellcolor{olive!25} & $81$ & $86$ & $91$\cellcolor{olive!25} & $91$\cellcolor{olive!25} \\
 & \texttt{isolet} & $98$\cellcolor{olive!25} & $96$\cellcolor{olive!25} & $99$\cellcolor{olive!25} & $97$\cellcolor{olive!25} & $99$\cellcolor{olive!25} & $99$\cellcolor{olive!25} & $98$\cellcolor{olive!25} & $97$\cellcolor{olive!25} & $98$\cellcolor{olive!25} & $98$\cellcolor{olive!25} \\
 & \texttt{letter} & $76$ & $48$ & $98$\cellcolor{olive!25} & $96$\cellcolor{olive!25} & $98$\cellcolor{olive!25} & $99$\cellcolor{olive!25} & $93$\cellcolor{olive!25} & $92$\cellcolor{olive!25} & $95$\cellcolor{olive!25} & $99$\cellcolor{olive!25} \\
 & \texttt{mhr} & $16$ & $6$ & $66$ & $59$ & $88$ & $76$ & $66$ & $96$\cellcolor{olive!25} & $100$\cellcolor{olive!25} & $100$\cellcolor{olive!25} \\
 & \texttt{molecular} & $69$ & $39$ & $99$\cellcolor{olive!25} & $99$\cellcolor{olive!25} & $100$\cellcolor{olive!25} & $100$\cellcolor{olive!25} & $99$\cellcolor{olive!25} & $99$\cellcolor{olive!25} & $100$\cellcolor{olive!25} & $100$\cellcolor{olive!25} \\
 & \texttt{nursery} & $70$ & $56$ & $100$\cellcolor{olive!25} & $99$\cellcolor{olive!25} & $100$\cellcolor{olive!25} & $99$\cellcolor{olive!25} & $98$\cellcolor{olive!25} & $97$\cellcolor{olive!25} & $96$\cellcolor{olive!25} & $97$\cellcolor{olive!25} \\
 & \texttt{obesity} & $91$\cellcolor{olive!25} & $89$ & $95$\cellcolor{olive!25} & $94$\cellcolor{olive!25} & $88$ & $88$ & $84$ & $96$\cellcolor{olive!25} & $88$ & $88$ \\
 & \texttt{page\_block} & $79$ & $46$ & $79$ & $78$ & $72$ & $68$ & $64$ & $78$ & $81$ & $77$ \\
 & \texttt{phishing} & $17$ & $10$ & $83$ & $78$ & $84$ & $80$ & $58$ & $90$\cellcolor{olive!25} & $92$\cellcolor{olive!25} & $90$\cellcolor{olive!25} \\
 & \texttt{satellite} & $60$ & $46$ & $97$\cellcolor{olive!25} & $96$\cellcolor{olive!25} & $98$\cellcolor{olive!25} & $95$\cellcolor{olive!25} & $97$\cellcolor{olive!25} & $97$\cellcolor{olive!25} & $98$\cellcolor{olive!25} & $95$\cellcolor{olive!25} \\
 & \texttt{shuttle} & $94$\cellcolor{olive!25} & $85$ & $100$\cellcolor{olive!25} & $99$\cellcolor{olive!25} & $100$\cellcolor{olive!25} & $100$\cellcolor{olive!25} & $100$\cellcolor{olive!25} & $97$\cellcolor{olive!25} & $100$\cellcolor{olive!25} & $100$\cellcolor{olive!25} \\
 & \texttt{waveform-v1} & $37$ & $22$ & $86$ & $99$\cellcolor{olive!25} & $98$\cellcolor{olive!25} & $100$\cellcolor{olive!25} & $64$ & $99$\cellcolor{olive!25} & $91$\cellcolor{olive!25} & $96$\cellcolor{olive!25} \\
 & \texttt{wine-quality} & $15$ & $4$ & $93$\cellcolor{olive!25} & $70$ & $63$ & $52$ & $93$\cellcolor{olive!25} & $93$\cellcolor{olive!25} & $93$\cellcolor{olive!25} & $96$\cellcolor{olive!25} \\
 & \texttt{yeast} & $23$ & $12$ & $85$ & $30$ & $44$ & $46$ & $70$ & $77$ & $95$\cellcolor{olive!25} & $95$\cellcolor{olive!25} \\
\midrule
\multirow{15}{*}{\begin{sideways}text\end{sideways}} & \texttt{gasp} & $54$ & $28$ & $97$\cellcolor{olive!25} & $92$\cellcolor{olive!25} & $97$\cellcolor{olive!25} & $96$\cellcolor{olive!25} & $97$\cellcolor{olive!25} & $85$ & $90$\cellcolor{olive!25} & $89$\cellcolor{olive!25} \\
 & \texttt{hcr} & $40$ & $17$ & $94$\cellcolor{olive!25} & $71$ & $84$ & $86$ & $98$\cellcolor{olive!25} & $91$\cellcolor{olive!25} & $93$\cellcolor{olive!25} & $94$\cellcolor{olive!25} \\
 & \texttt{omd} & $47$ & $38$ & $92$\cellcolor{olive!25} & $76$ & $96$\cellcolor{olive!25} & $93$\cellcolor{olive!25} & $80$ & $76$ & $86$ & $84$ \\
 & \texttt{sanders} & $46$ & $22$ & $97$\cellcolor{olive!25} & $96$\cellcolor{olive!25} & $95$\cellcolor{olive!25} & $95$\cellcolor{olive!25} & $97$\cellcolor{olive!25} & $96$\cellcolor{olive!25} & $95$\cellcolor{olive!25} & $95$\cellcolor{olive!25} \\
 & \texttt{semeval13} & $48$ & $32$ & $94$\cellcolor{olive!25} & $81$ & $95$\cellcolor{olive!25} & $96$\cellcolor{olive!25} & $94$\cellcolor{olive!25} & $81$ & $95$\cellcolor{olive!25} & $96$\cellcolor{olive!25} \\
 & \texttt{semeval14} & $50$ & $32$ & $96$\cellcolor{olive!25} & $84$ & $95$\cellcolor{olive!25} & $96$\cellcolor{olive!25} & $96$\cellcolor{olive!25} & $84$ & $95$\cellcolor{olive!25} & $96$\cellcolor{olive!25} \\
 & \texttt{semeval15} & $42$ & $30$ & $81$ & $72$ & $82$ & $83$ & $81$ & $72$ & $82$ & $83$ \\
 & \texttt{semeval16} & $39$ & $18$ & $89$ & $73$ & $89$ & $89$ & $93$\cellcolor{olive!25} & $89$\cellcolor{olive!25} & $89$ & $89$ \\
 & \texttt{sst} & $48$ & $35$ & $94$\cellcolor{olive!25} & $91$\cellcolor{olive!25} & $94$\cellcolor{olive!25} & $95$\cellcolor{olive!25} & $94$\cellcolor{olive!25} & $91$\cellcolor{olive!25} & $94$\cellcolor{olive!25} & $95$\cellcolor{olive!25} \\
 & \texttt{wa} & $61$ & $54$ & $94$\cellcolor{olive!25} & $90$\cellcolor{olive!25} & $93$\cellcolor{olive!25} & $93$\cellcolor{olive!25} & $84$ & $80$ & $84$ & $83$ \\
 & \texttt{wb} & $59$ & $56$ & $97$\cellcolor{olive!25} & $96$\cellcolor{olive!25} & $98$\cellcolor{olive!25} & $98$\cellcolor{olive!25} & $97$\cellcolor{olive!25} & $91$\cellcolor{olive!25} & $91$\cellcolor{olive!25} & $91$\cellcolor{olive!25} \\
 & \texttt{LeQua2022-T1A} & $32$ & $21$ & $96$\cellcolor{olive!25} & $96$\cellcolor{olive!25} & $98$\cellcolor{olive!25} & $97$\cellcolor{olive!25} & $96$\cellcolor{olive!25} & $89$\cellcolor{olive!25} & $93$\cellcolor{olive!25} & $91$\cellcolor{olive!25} \\
 & \texttt{LeQua2022-T1B} & $52$ & $30$ & $93$\cellcolor{olive!25} & $66$ & $84$ & $84$ & $93$\cellcolor{olive!25} & $87$ & $94$\cellcolor{olive!25} & $94$\cellcolor{olive!25} \\
 & \texttt{LeQua2024-T1} & $34$ & $18$ & $98$\cellcolor{olive!25} & $95$\cellcolor{olive!25} & $98$\cellcolor{olive!25} & $98$\cellcolor{olive!25} & $90$\cellcolor{olive!25} & $89$ & $91$\cellcolor{olive!25} & $92$\cellcolor{olive!25} \\
 & \texttt{LeQua2024-T2} & $41$ & $22$ & $93$\cellcolor{olive!25} & $69$ & $84$ & $82$ & $93$\cellcolor{olive!25} & $86$ & $95$\cellcolor{olive!25} & $94$\cellcolor{olive!25} \\
\midrule
\multirow{5}{*}{\begin{sideways}image\end{sideways}} & \texttt{CIFAR100} & $82$ & $60$ & $96$\cellcolor{olive!25} & $92$\cellcolor{olive!25} & $95$\cellcolor{olive!25} & $94$\cellcolor{olive!25} & $89$ & $92$\cellcolor{olive!25} & $95$\cellcolor{olive!25} & $94$\cellcolor{olive!25} \\
 & \texttt{CIFAR10} & $97$\cellcolor{olive!25} & $89$\cellcolor{olive!25} & $99$\cellcolor{olive!25} & $98$\cellcolor{olive!25} & $99$\cellcolor{olive!25} & $99$\cellcolor{olive!25} & $97$\cellcolor{olive!25} & $98$\cellcolor{olive!25} & $98$\cellcolor{olive!25} & $98$\cellcolor{olive!25} \\
 & \texttt{MNIST} & $100$\cellcolor{olive!25} & $99$\cellcolor{olive!25} & $99$\cellcolor{olive!25} & $98$\cellcolor{olive!25} & $99$\cellcolor{olive!25} & $99$\cellcolor{olive!25} & $99$\cellcolor{olive!25} & $99$\cellcolor{olive!25} & $99$\cellcolor{olive!25} & $99$\cellcolor{olive!25} \\
 & \texttt{fashionMNIST} & $91$\cellcolor{olive!25} & $82$ & $98$\cellcolor{olive!25} & $97$\cellcolor{olive!25} & $99$\cellcolor{olive!25} & $99$\cellcolor{olive!25} & $97$\cellcolor{olive!25} & $96$\cellcolor{olive!25} & $98$\cellcolor{olive!25} & $98$\cellcolor{olive!25} \\
 & \texttt{SVHN} & $92$\cellcolor{olive!25} & $83$ & $98$\cellcolor{olive!25} & $98$\cellcolor{olive!25} & $99$\cellcolor{olive!25} & $99$\cellcolor{olive!25} & $96$\cellcolor{olive!25} & $97$\cellcolor{olive!25} & $97$\cellcolor{olive!25} & $98$\cellcolor{olive!25} \\
\midrule
 & Mean & $56$ & $43$ & $92$ & $85$ & $89$ & $89$ & $89$ & $90$ & $93$ & $93$ \\
\bottomrule
\end{tabular}

    }%
\end{table} %
\begin{table}[h]
    \caption{Amplitude 
    before and after temperature calibration. Highlighted cells indicate values $\leq5\%$.}
    \label{tab:amplitude}
    \centering
    \resizebox{.8\textwidth}{!}{%
\begin{tabular}{ll|cc|cccc|cccc|}
\toprule
\multicolumn{2}{c}{} & \multicolumn{1}{c}{\begin{sideways}Boots-CC\;\end{sideways}} & \multicolumn{1}{c}{\begin{sideways}Boots-PCC\;\end{sideways}} & \multicolumn{1}{c}{\begin{sideways}Bayes-ACC\;\end{sideways}} & \multicolumn{1}{c}{\begin{sideways}Bayes-EMQ (MAPLS)\;\end{sideways}} & \multicolumn{1}{c}{\begin{sideways}Bayes-KDEy(Gau)\;\end{sideways}} & \multicolumn{1}{c}{\begin{sideways}Bayes-KDEy(Ait-$\lambda$)\;\end{sideways}} & \multicolumn{1}{c}{\begin{sideways}Bayes-ACC$^{cal}$\;\end{sideways}} & \multicolumn{1}{c}{\begin{sideways}Bayes-EMQ (MAPLS)$^{cal}$\;\end{sideways}} & \multicolumn{1}{c}{\begin{sideways}Bayes-KDEy(Gau)$^{cal}$\;\end{sideways}} & \multicolumn{1}{c}{\begin{sideways}Bayes-KDEy(Ait-$\lambda$)$^{cal}$\;\end{sideways}} \\
\cmidrule(lr){3-4}\cmidrule(lr){5-8}\cmidrule(lr){9-12}
\multicolumn{2}{c}{} & \multicolumn{2}{c}{} & \multicolumn{4}{c}{Temperature=1} & \multicolumn{4}{c}{Calibrated Temperature} \\
\midrule
\multirow{22}{*}{\begin{sideways}tabular\end{sideways}} & \texttt{abalone} & $<0.1$\cellcolor{olive!25} & $<0.1$\cellcolor{olive!25} & $<0.1$\cellcolor{olive!25} & $<0.1$\cellcolor{olive!25} & $<0.1$\cellcolor{olive!25} & $<0.1$\cellcolor{olive!25} & $44.9$ & $3.3$\cellcolor{olive!25} & $26.0$ & $30.0$ \\
 & \texttt{academic-success} & $0.4$\cellcolor{olive!25} & $0.2$\cellcolor{olive!25} & $2.2$\cellcolor{olive!25} & $1.0$\cellcolor{olive!25} & $1.0$\cellcolor{olive!25} & $1.0$\cellcolor{olive!25} & $1.1$\cellcolor{olive!25} & $2.0$\cellcolor{olive!25} & $1.0$\cellcolor{olive!25} & $1.0$\cellcolor{olive!25} \\
 & \texttt{chess} & $<0.1$\cellcolor{olive!25} & $<0.1$\cellcolor{olive!25} & $0.1$\cellcolor{olive!25} & $<0.1$\cellcolor{olive!25} & $<0.1$\cellcolor{olive!25} & $<0.1$\cellcolor{olive!25} & $0.4$\cellcolor{olive!25} & $0.1$\cellcolor{olive!25} & $<0.1$\cellcolor{olive!25} & $<0.1$\cellcolor{olive!25} \\
 & \texttt{cmc} & $0.5$\cellcolor{olive!25} & $0.1$\cellcolor{olive!25} & $12.5$ & $3.2$\cellcolor{olive!25} & $3.0$\cellcolor{olive!25} & $2.9$\cellcolor{olive!25} & $12.5$ & $6.1$ & $5.9$ & $5.6$ \\
 & \texttt{connect-4} & $0.5$\cellcolor{olive!25} & $0.1$\cellcolor{olive!25} & $2.7$\cellcolor{olive!25} & $1.8$\cellcolor{olive!25} & $1.6$\cellcolor{olive!25} & $1.6$\cellcolor{olive!25} & $2.7$\cellcolor{olive!25} & $1.8$\cellcolor{olive!25} & $1.6$\cellcolor{olive!25} & $1.6$\cellcolor{olive!25} \\
 & \texttt{digits} & $<0.1$\cellcolor{olive!25} & $<0.1$\cellcolor{olive!25} & $<0.1$\cellcolor{olive!25} & $<0.1$\cellcolor{olive!25} & $<0.1$\cellcolor{olive!25} & $<0.1$\cellcolor{olive!25} & $<0.1$\cellcolor{olive!25} & $<0.1$\cellcolor{olive!25} & $<0.1$\cellcolor{olive!25} & $<0.1$\cellcolor{olive!25} \\
 & \texttt{dry-bean} & $<0.1$\cellcolor{olive!25} & $<0.1$\cellcolor{olive!25} & $<0.1$\cellcolor{olive!25} & $<0.1$\cellcolor{olive!25} & $<0.1$\cellcolor{olive!25} & $<0.1$\cellcolor{olive!25} & $<0.1$\cellcolor{olive!25} & $<0.1$\cellcolor{olive!25} & $<0.1$\cellcolor{olive!25} & $<0.1$\cellcolor{olive!25} \\
 & \texttt{hand\_digits} & $<0.1$\cellcolor{olive!25} & $<0.1$\cellcolor{olive!25} & $<0.1$\cellcolor{olive!25} & $<0.1$\cellcolor{olive!25} & $<0.1$\cellcolor{olive!25} & $<0.1$\cellcolor{olive!25} & $<0.1$\cellcolor{olive!25} & $<0.1$\cellcolor{olive!25} & $<0.1$\cellcolor{olive!25} & $<0.1$\cellcolor{olive!25} \\
 & \texttt{image\_seg} & $<0.1$\cellcolor{olive!25} & $<0.1$\cellcolor{olive!25} & $<0.1$\cellcolor{olive!25} & $<0.1$\cellcolor{olive!25} & $<0.1$\cellcolor{olive!25} & $<0.1$\cellcolor{olive!25} & $<0.1$\cellcolor{olive!25} & $<0.1$\cellcolor{olive!25} & $<0.1$\cellcolor{olive!25} & $<0.1$\cellcolor{olive!25} \\
 & \texttt{isolet} & $<0.1$\cellcolor{olive!25} & $<0.1$\cellcolor{olive!25} & $<0.1$\cellcolor{olive!25} & $<0.1$\cellcolor{olive!25} & $<0.1$\cellcolor{olive!25} & $<0.1$\cellcolor{olive!25} & $<0.1$\cellcolor{olive!25} & $<0.1$\cellcolor{olive!25} & $<0.1$\cellcolor{olive!25} & $<0.1$\cellcolor{olive!25} \\
 & \texttt{letter} & $<0.1$\cellcolor{olive!25} & $<0.1$\cellcolor{olive!25} & $<0.1$\cellcolor{olive!25} & $<0.1$\cellcolor{olive!25} & $<0.1$\cellcolor{olive!25} & $<0.1$\cellcolor{olive!25} & $<0.1$\cellcolor{olive!25} & $<0.1$\cellcolor{olive!25} & $<0.1$\cellcolor{olive!25} & $<0.1$\cellcolor{olive!25} \\
 & \texttt{mhr} & $0.4$\cellcolor{olive!25} & $0.1$\cellcolor{olive!25} & $6.4$ & $1.9$\cellcolor{olive!25} & $1.1$\cellcolor{olive!25} & $1.1$\cellcolor{olive!25} & $6.4$ & $15.3$ & $10.1$ & $9.9$ \\
 & \texttt{molecular} & $0.4$\cellcolor{olive!25} & $0.4$\cellcolor{olive!25} & $0.5$\cellcolor{olive!25} & $0.4$\cellcolor{olive!25} & $0.4$\cellcolor{olive!25} & $0.4$\cellcolor{olive!25} & $0.5$\cellcolor{olive!25} & $0.4$\cellcolor{olive!25} & $0.4$\cellcolor{olive!25} & $0.4$\cellcolor{olive!25} \\
 & \texttt{nursery} & $<0.1$\cellcolor{olive!25} & $<0.1$\cellcolor{olive!25} & $0.1$\cellcolor{olive!25} & $<0.1$\cellcolor{olive!25} & $<0.1$\cellcolor{olive!25} & $<0.1$\cellcolor{olive!25} & $<0.1$\cellcolor{olive!25} & $<0.1$\cellcolor{olive!25} & $<0.1$\cellcolor{olive!25} & $<0.1$\cellcolor{olive!25} \\
 & \texttt{obesity} & $<0.1$\cellcolor{olive!25} & $<0.1$\cellcolor{olive!25} & $<0.1$\cellcolor{olive!25} & $<0.1$\cellcolor{olive!25} & $<0.1$\cellcolor{olive!25} & $<0.1$\cellcolor{olive!25} & $<0.1$\cellcolor{olive!25} & $<0.1$\cellcolor{olive!25} & $<0.1$\cellcolor{olive!25} & $<0.1$\cellcolor{olive!25} \\
 & \texttt{page\_block} & $<0.1$\cellcolor{olive!25} & $<0.1$\cellcolor{olive!25} & $0.2$\cellcolor{olive!25} & $<0.1$\cellcolor{olive!25} & $<0.1$\cellcolor{olive!25} & $<0.1$\cellcolor{olive!25} & $0.1$\cellcolor{olive!25} & $<0.1$\cellcolor{olive!25} & $0.1$\cellcolor{olive!25} & $0.1$\cellcolor{olive!25} \\
 & \texttt{phishing} & $0.4$\cellcolor{olive!25} & $0.2$\cellcolor{olive!25} & $3.4$\cellcolor{olive!25} & $0.9$\cellcolor{olive!25} & $0.8$\cellcolor{olive!25} & $0.8$\cellcolor{olive!25} & $1.7$\cellcolor{olive!25} & $1.4$\cellcolor{olive!25} & $1.1$\cellcolor{olive!25} & $1.2$\cellcolor{olive!25} \\
 & \texttt{satellite} & $<0.1$\cellcolor{olive!25} & $<0.1$\cellcolor{olive!25} & $<0.1$\cellcolor{olive!25} & $<0.1$\cellcolor{olive!25} & $<0.1$\cellcolor{olive!25} & $<0.1$\cellcolor{olive!25} & $<0.1$\cellcolor{olive!25} & $<0.1$\cellcolor{olive!25} & $<0.1$\cellcolor{olive!25} & $<0.1$\cellcolor{olive!25} \\
 & \texttt{shuttle} & $0.3$\cellcolor{olive!25} & $0.3$\cellcolor{olive!25} & $0.3$\cellcolor{olive!25} & $0.3$\cellcolor{olive!25} & $0.3$\cellcolor{olive!25} & $0.3$\cellcolor{olive!25} & $0.2$\cellcolor{olive!25} & $0.2$\cellcolor{olive!25} & $0.2$\cellcolor{olive!25} & $0.2$\cellcolor{olive!25} \\
 & \texttt{waveform-v1} & $0.4$\cellcolor{olive!25} & $0.3$\cellcolor{olive!25} & $0.9$\cellcolor{olive!25} & $0.5$\cellcolor{olive!25} & $0.5$\cellcolor{olive!25} & $0.5$\cellcolor{olive!25} & $0.4$\cellcolor{olive!25} & $0.5$\cellcolor{olive!25} & $0.2$\cellcolor{olive!25} & $0.2$\cellcolor{olive!25} \\
 & \texttt{wine-quality} & $<0.1$\cellcolor{olive!25} & $<0.1$\cellcolor{olive!25} & $5.7$ & $0.8$\cellcolor{olive!25} & $1.6$\cellcolor{olive!25} & $0.2$\cellcolor{olive!25} & $5.7$ & $7.9$ & $15.3$ & $59.0$ \\
 & \texttt{yeast} & $<0.1$\cellcolor{olive!25} & $<0.1$\cellcolor{olive!25} & $0.3$\cellcolor{olive!25} & $<0.1$\cellcolor{olive!25} & $<0.1$\cellcolor{olive!25} & $<0.1$\cellcolor{olive!25} & $<0.1$\cellcolor{olive!25} & $0.4$\cellcolor{olive!25} & $28.2$ & $29.9$ \\
\midrule
\multirow{15}{*}{\begin{sideways}text\end{sideways}} & \texttt{gasp} & $4.3$\cellcolor{olive!25} & $2.2$\cellcolor{olive!25} & $10.3$ & $6.9$ & $7.8$ & $7.5$ & $10.3$ & $3.6$\cellcolor{olive!25} & $4.0$\cellcolor{olive!25} & $3.9$\cellcolor{olive!25} \\
 & \texttt{hcr} & $4.5$\cellcolor{olive!25} & $1.2$\cellcolor{olive!25} & $21.0$ & $8.2$ & $12.7$ & $12.9$ & $36.0$ & $30.7$ & $22.9$ & $23.3$ \\
 & \texttt{omd} & $4.4$\cellcolor{olive!25} & $3.8$\cellcolor{olive!25} & $17.4$ & $7.3$ & $12.8$ & $12.4$ & $9.6$ & $7.3$ & $6.8$ & $6.6$ \\
 & \texttt{sanders} & $4.3$\cellcolor{olive!25} & $1.6$\cellcolor{olive!25} & $12.0$ & $8.6$ & $8.6$ & $8.4$ & $12.0$ & $8.6$ & $8.6$ & $8.4$ \\
 & \texttt{semeval13} & $4.2$\cellcolor{olive!25} & $3.1$\cellcolor{olive!25} & $11.2$ & $7.2$ & $9.0$ & $9.2$ & $11.2$ & $7.2$ & $9.0$ & $9.2$ \\
 & \texttt{semeval14} & $4.3$\cellcolor{olive!25} & $3.1$\cellcolor{olive!25} & $11.5$ & $7.2$ & $9.0$ & $9.2$ & $11.5$ & $7.2$ & $9.0$ & $9.2$ \\
 & \texttt{semeval15} & $4.2$\cellcolor{olive!25} & $3.1$\cellcolor{olive!25} & $12.0$ & $7.6$ & $9.6$ & $10.0$ & $12.0$ & $7.6$ & $9.6$ & $10.0$ \\
 & \texttt{semeval16} & $4.6$\cellcolor{olive!25} & $1.4$\cellcolor{olive!25} & $28.3$ & $12.3$ & $20.9$ & $21.1$ & $36.6$ & $39.5$ & $20.9$ & $21.1$ \\
 & \texttt{sst} & $4.5$\cellcolor{olive!25} & $4.2$\cellcolor{olive!25} & $17.3$ & $9.9$ & $10.5$ & $10.5$ & $17.3$ & $9.9$ & $10.5$ & $10.5$ \\
 & \texttt{wa} & $4.2$\cellcolor{olive!25} & $4.1$\cellcolor{olive!25} & $8.7$ & $6.5$ & $6.3$ & $6.8$ & $4.5$\cellcolor{olive!25} & $3.4$\cellcolor{olive!25} & $3.2$\cellcolor{olive!25} & $3.5$\cellcolor{olive!25} \\
 & \texttt{wb} & $4.2$\cellcolor{olive!25} & $4.1$\cellcolor{olive!25} & $8.7$ & $6.1$ & $6.3$ & $6.2$ & $8.7$ & $3.2$\cellcolor{olive!25} & $3.2$\cellcolor{olive!25} & $3.2$\cellcolor{olive!25} \\
 & \texttt{LeQua2022-T1A} & $11.3$ & $8.8$ & $18.1$ & $14.5$ & $14.9$ & $14.5$ & $18.1$ & $10.3$ & $10.6$ & $10.3$ \\
 & \texttt{LeQua2022-T1B} & $<0.1$\cellcolor{olive!25} & $<0.1$\cellcolor{olive!25} & $<0.1$\cellcolor{olive!25} & $<0.1$\cellcolor{olive!25} & $<0.1$\cellcolor{olive!25} & $<0.1$\cellcolor{olive!25} & $<0.1$\cellcolor{olive!25} & $<0.1$\cellcolor{olive!25} & $<0.1$\cellcolor{olive!25} & $<0.1$\cellcolor{olive!25} \\
 & \texttt{LeQua2024-T1} & $11.1$ & $8.3$ & $16.6$ & $13.0$ & $13.2$ & $13.2$ & $11.8$ & $9.3$ & $9.4$ & $9.4$ \\
 & \texttt{LeQua2024-T2} & $<0.1$\cellcolor{olive!25} & $<0.1$\cellcolor{olive!25} & $<0.1$\cellcolor{olive!25} & $<0.1$\cellcolor{olive!25} & $<0.1$\cellcolor{olive!25} & $<0.1$\cellcolor{olive!25} & $<0.1$\cellcolor{olive!25} & $<0.1$\cellcolor{olive!25} & $<0.1$\cellcolor{olive!25} & $<0.1$\cellcolor{olive!25} \\
\midrule
\multirow{5}{*}{\begin{sideways}image\end{sideways}} & \texttt{CIFAR100} & $<0.1$\cellcolor{olive!25} & $<0.1$\cellcolor{olive!25} & $<0.1$\cellcolor{olive!25} & $<0.1$\cellcolor{olive!25} & $<0.1$\cellcolor{olive!25} & $<0.1$\cellcolor{olive!25} & $<0.1$\cellcolor{olive!25} & $<0.1$\cellcolor{olive!25} & $<0.1$\cellcolor{olive!25} & $<0.1$\cellcolor{olive!25} \\
 & \texttt{CIFAR10} & $<0.1$\cellcolor{olive!25} & $<0.1$\cellcolor{olive!25} & $<0.1$\cellcolor{olive!25} & $<0.1$\cellcolor{olive!25} & $<0.1$\cellcolor{olive!25} & $<0.1$\cellcolor{olive!25} & $<0.1$\cellcolor{olive!25} & $<0.1$\cellcolor{olive!25} & $<0.1$\cellcolor{olive!25} & $<0.1$\cellcolor{olive!25} \\
 & \texttt{MNIST} & $<0.1$\cellcolor{olive!25} & $<0.1$\cellcolor{olive!25} & $<0.1$\cellcolor{olive!25} & $<0.1$\cellcolor{olive!25} & $<0.1$\cellcolor{olive!25} & $<0.1$\cellcolor{olive!25} & $<0.1$\cellcolor{olive!25} & $<0.1$\cellcolor{olive!25} & $<0.1$\cellcolor{olive!25} & $<0.1$\cellcolor{olive!25} \\
 & \texttt{fashionMNIST} & $<0.1$\cellcolor{olive!25} & $<0.1$\cellcolor{olive!25} & $<0.1$\cellcolor{olive!25} & $<0.1$\cellcolor{olive!25} & $<0.1$\cellcolor{olive!25} & $<0.1$\cellcolor{olive!25} & $<0.1$\cellcolor{olive!25} & $<0.1$\cellcolor{olive!25} & $<0.1$\cellcolor{olive!25} & $<0.1$\cellcolor{olive!25} \\
 & \texttt{SVHN} & $<0.1$\cellcolor{olive!25} & $<0.1$\cellcolor{olive!25} & $<0.1$\cellcolor{olive!25} & $<0.1$\cellcolor{olive!25} & $<0.1$\cellcolor{olive!25} & $<0.1$\cellcolor{olive!25} & $<0.1$\cellcolor{olive!25} & $<0.1$\cellcolor{olive!25} & $<0.1$\cellcolor{olive!25} & $<0.1$\cellcolor{olive!25} \\
\midrule
 & Mean & $1.7$ & $1.2$ & $5.4$ & $3.0$ & $3.6$ & $3.6$ & $6.6$ & $4.5$ & $5.2$ & $6.4$ \\
\bottomrule
\end{tabular}

    }%
\end{table}

Statistical significance comparisons are reported in Table~\ref{tab:ae-pairwise-symbols:bayes}. Notational conventions are as in Table~\ref{tab:ae-pairwise-symbols}.

\begin{table*}[t]
\centering
\caption{Pairwise symbolic matrix of Wilcoxon-Holm comparisons on AE.}
\label{tab:ae-pairwise-symbols:bayes}
\resizebox{\textwidth}{!}{%
\begin{tabular}{lcccccc}
\toprule
 & Boots-CC & Boots-PCC & Bayes-ACC & Bayes-EMQ (MAPLS) & Bayes-KDEy(Gau) & Bayes-KDEy(Ait-$\lambda$) \\
\midrule
Boots-CC & $=$ & $\ll$ & $\gg$ & $\gg$ & $\gg$ & $\gg$ \\
Boots-PCC & $\gg$ & $=$ & $\gg$ & $\gg$ & $\gg$ & $\gg$ \\
Bayes-ACC & $\ll$ & $\ll$ & $=$ & $>$ & $\gg$ & $\gg$ \\
Bayes-EMQ (MAPLS) & $\ll$ & $\ll$ & $<$ & $=$ & $>$ & $>$ \\
Bayes-KDEy(Gau) & $\ll$ & $\ll$ & $\ll$ & $<$ & $=$ & $\approx$ \\
Bayes-KDEy(Ait-$\lambda$) & $\ll$ & $\ll$ & $\ll$ & $<$ & $\approx$ & $=$ \\
\bottomrule
\end{tabular}
}%
\end{table*}

\subsection{Ablation experiment: unregularized geometry-aware KDE}
\label{app:ablation}


Table~\ref{tab:pointclr} reports the results of KDEy(Ait), the unregularized geometry-aware KDE variant without shrinkage regularization. The results show that, although the pure Aitchison-based model performs well on most datasets, it also exhibits a few severe failures (notably on \texttt{isolet}, \texttt{mhr}, \texttt{wine-quality}, \texttt{LeQua2022-T1B}, and \texttt{LeQua2024-T2}). These failures are effectively corrected by the regularized variant KDEy(Ait-$\lambda$). The remaining methods are included for reference, and the selected value of \(\lambda\) is reported in the last column. Notational conventions are as in Table~\ref{tab:point}.


\begin{table}[h]
    \caption{Ablation results in terms of AE and $W$ for the unregularized geometry-aware model KDEy(Ait) (i.e., $\lambda=0$). All methods included for reference. The method is competitive on many datasets, but suffers from instability in a small number of cases; these failure modes are mitigated by the proposed shrinkage regularization. The shrinkage parameter chosen by KDEy(Ait-$\lambda$) is shown in the last column.}
    \label{tab:pointclr}
    \centering
    \resizebox{\textwidth}{!}{%
\begin{tabular}{ll|ccccccc|ccccccc|c|}
\toprule
\multicolumn{2}{c}{} & \multicolumn{1}{c}{\begin{sideways}CC\;\end{sideways}} & \multicolumn{1}{c}{\begin{sideways}PCC\;\end{sideways}} & \multicolumn{1}{c}{\begin{sideways}BBSE/ACC\;\end{sideways}} & \multicolumn{1}{c}{\begin{sideways}MLLS/EMQ\;\end{sideways}} & \multicolumn{1}{c}{\begin{sideways}KDEy(Gau)\;\end{sideways}} & \multicolumn{1}{c}{\begin{sideways}KDEy(Ait)\;\end{sideways}} & \multicolumn{1}{c}{\begin{sideways}KDEy(Ait-$\lambda$)\;\end{sideways}} & \multicolumn{1}{c}{\begin{sideways}CC\;\end{sideways}} & \multicolumn{1}{c}{\begin{sideways}PCC\;\end{sideways}} & \multicolumn{1}{c}{\begin{sideways}BBSE/ACC\;\end{sideways}} & \multicolumn{1}{c}{\begin{sideways}MLLS/EMQ\;\end{sideways}} & \multicolumn{1}{c}{\begin{sideways}KDEy(Gau)\;\end{sideways}} & \multicolumn{1}{c}{\begin{sideways}KDEy(Ait)\;\end{sideways}} & \multicolumn{1}{c}{\begin{sideways}KDEy(Ait-$\lambda$)\;\end{sideways}} & \multicolumn{1}{c}{\begin{sideways}$\lambda$\;\end{sideways}} \\
\cmidrule(lr){3-9}\cmidrule(lr){10-16}\cmidrule(lr){17-17}
\multicolumn{2}{c}{} & \multicolumn{7}{c}{AE} & \multicolumn{7}{c}{W} & \multicolumn{1}{c}{} \\
\midrule
\multirow{22}{*}{\begin{sideways}tabular\end{sideways}} & \texttt{abalone} & $0.0450$\cellcolor{olive!28} & $\textbf{0.0424}$\cellcolor{olive!30} & $0.0966$\cellcolor{olive!0} & $0.0655$\cellcolor{olive!17} & $0.0496$\cellcolor{olive!26} & $0.0485$\cellcolor{olive!26} & $0.0505$\cellcolor{olive!25} & $7.0645$\cellcolor{olive!29} & $\textbf{6.2734}$\cellcolor{olive!30} & $96.8502$\cellcolor{olive!0} & $15.7999$\cellcolor{olive!26} & $9.7935$\cellcolor{olive!28} & $10.1485$\cellcolor{olive!28} & $9.8538$\cellcolor{olive!28} & 0.750 \\
 & \texttt{academic-success} & $0.0953$\cellcolor{olive!7} & $0.1211$\cellcolor{olive!0} & $0.0491$\cellcolor{olive!21} & $0.0260$\cellcolor{olive!28} & $0.0251$\cellcolor{olive!28} & $\textbf{0.0211}$\cellcolor{olive!30} & $0.0236$\cellcolor{olive!29} & $0.2541$\cellcolor{olive!12} & $0.4139$\cellcolor{olive!0} & $0.0796$\cellcolor{olive!25} & $0.0231$\cellcolor{olive!29} & $0.0210$\cellcolor{olive!29} & $\textbf{0.0152}$\cellcolor{olive!30} & $0.0188$\cellcolor{olive!29} & 0.250 \\
 & \texttt{chess} & $0.0413$\cellcolor{olive!3} & $0.0450$\cellcolor{olive!0} & $0.0409$\cellcolor{olive!4} & $0.0398$\cellcolor{olive!5} & $^\dag0.0165$\cellcolor{olive!29} & $0.0169$\cellcolor{olive!29} & $\textbf{0.0163}$\cellcolor{olive!30} & $1.9618$\cellcolor{olive!5} & $2.3372$\cellcolor{olive!0} & $1.1529$\cellcolor{olive!17} & $1.3512$\cellcolor{olive!14} & $\textbf{0.3125}$\cellcolor{olive!30} & $^\dag0.3125$\cellcolor{olive!29} & $0.3878$\cellcolor{olive!28} & 0.500 \\
 & \texttt{cmc} & $0.1560$\cellcolor{olive!7} & $0.1778$\cellcolor{olive!0} & $0.1065$\cellcolor{olive!25} & $0.0973$\cellcolor{olive!28} & $^\dag0.0937$\cellcolor{olive!29} & $0.0965$\cellcolor{olive!28} & $\textbf{0.0934}$\cellcolor{olive!30} & $0.3903$\cellcolor{olive!10} & $0.4937$\cellcolor{olive!0} & $0.2123$\cellcolor{olive!27} & $\textbf{0.1884}$\cellcolor{olive!30} & $^\dag0.1913$\cellcolor{olive!29} & $0.2046$\cellcolor{olive!28} & $^\dag0.1925$\cellcolor{olive!29} & 0.999 \\
 & \texttt{connect-4} & $0.1138$\cellcolor{olive!7} & $0.1468$\cellcolor{olive!0} & $0.0270$\cellcolor{olive!28} & $0.0301$\cellcolor{olive!27} & $^\dag0.0205$\cellcolor{olive!29} & $^\dag0.0209$\cellcolor{olive!29} & $\textbf{0.0201}$\cellcolor{olive!30} & $1.0574$\cellcolor{olive!11} & $1.6642$\cellcolor{olive!0} & $0.0742$\cellcolor{olive!29} & $0.1020$\cellcolor{olive!28} & $0.0412$\cellcolor{olive!29} & $^\dag0.0460$\cellcolor{olive!29} & $\textbf{0.0395}$\cellcolor{olive!30} & 0.750 \\
 & \texttt{digits} & $0.0035$\cellcolor{olive!10} & $0.0045$\cellcolor{olive!0} & $0.0030$\cellcolor{olive!15} & $0.0023$\cellcolor{olive!22} & $0.0024$\cellcolor{olive!21} & $0.0016$\cellcolor{olive!29} & $\textbf{0.0016}$\cellcolor{olive!30} & $0.0034$\cellcolor{olive!8} & $0.0045$\cellcolor{olive!0} & $0.0017$\cellcolor{olive!21} & $0.0012$\cellcolor{olive!25} & $0.0014$\cellcolor{olive!23} & $\textbf{0.0005}$\cellcolor{olive!30} & $^\dag0.0006$\cellcolor{olive!29} & 0.001 \\
 & \texttt{dry-bean} & $0.0088$\cellcolor{olive!10} & $0.0122$\cellcolor{olive!0} & $0.0040$\cellcolor{olive!26} & $0.0032$\cellcolor{olive!28} & $0.0032$\cellcolor{olive!28} & $\textbf{0.0028}$\cellcolor{olive!30} & $0.0030$\cellcolor{olive!29} & $0.0066$\cellcolor{olive!15} & $0.0131$\cellcolor{olive!0} & $0.0020$\cellcolor{olive!26} & $0.0012$\cellcolor{olive!28} & $0.0011$\cellcolor{olive!29} & $\textbf{0.0007}$\cellcolor{olive!30} & $0.0010$\cellcolor{olive!29} & 0.250 \\
 & \texttt{hand\_digits} & $0.0048$\cellcolor{olive!9} & $0.0063$\cellcolor{olive!0} & $0.0031$\cellcolor{olive!20} & $0.0023$\cellcolor{olive!25} & $0.0023$\cellcolor{olive!25} & $0.0020$\cellcolor{olive!27} & $\textbf{0.0016}$\cellcolor{olive!30} & $0.0054$\cellcolor{olive!13} & $0.0092$\cellcolor{olive!0} & $0.0018$\cellcolor{olive!25} & $0.0010$\cellcolor{olive!28} & $0.0011$\cellcolor{olive!27} & $0.0008$\cellcolor{olive!28} & $\textbf{0.0006}$\cellcolor{olive!30} & 0.001 \\
 & \texttt{image\_seg} & $0.0063$\cellcolor{olive!16} & $0.0083$\cellcolor{olive!4} & $0.0090$\cellcolor{olive!0} & $0.0051$\cellcolor{olive!24} & $\textbf{0.0041}$\cellcolor{olive!30} & $0.0054$\cellcolor{olive!22} & $0.0041$\cellcolor{olive!29} & $0.0098$\cellcolor{olive!15} & $0.0164$\cellcolor{olive!0} & $0.0114$\cellcolor{olive!11} & $0.0046$\cellcolor{olive!27} & $\textbf{0.0033}$\cellcolor{olive!30} & $0.0046$\cellcolor{olive!27} & $0.0034$\cellcolor{olive!29} & 0.900 \\
 & \texttt{isolet} & $0.0021$\cellcolor{olive!16} & $0.0023$\cellcolor{olive!12} & $0.0018$\cellcolor{olive!20} & $0.0014$\cellcolor{olive!27} & $0.0014$\cellcolor{olive!29} & $0.0030$\cellcolor{olive!0} & $\textbf{0.0013}$\cellcolor{olive!30} & $0.0077$\cellcolor{olive!19} & $0.0092$\cellcolor{olive!15} & $0.0051$\cellcolor{olive!24} & $0.0034$\cellcolor{olive!28} & $0.0029$\cellcolor{olive!29} & $0.0165$\cellcolor{olive!0} & $\textbf{0.0027}$\cellcolor{olive!30} & 0.250 \\
 & \texttt{letter} & $0.0077$\cellcolor{olive!9} & $0.0108$\cellcolor{olive!0} & $0.0048$\cellcolor{olive!19} & $0.0043$\cellcolor{olive!20} & $0.0025$\cellcolor{olive!26} & $0.0017$\cellcolor{olive!28} & $\textbf{0.0014}$\cellcolor{olive!30} & $0.0856$\cellcolor{olive!14} & $0.1631$\cellcolor{olive!0} & $0.0300$\cellcolor{olive!24} & $0.0253$\cellcolor{olive!25} & $0.0078$\cellcolor{olive!29} & $0.0041$\cellcolor{olive!29} & $\textbf{0.0025}$\cellcolor{olive!30} & 0.001 \\
 & \texttt{mhr} & $0.1216$\cellcolor{olive!5} & $0.1444$\cellcolor{olive!0} & $0.1113$\cellcolor{olive!8} & $0.0549$\cellcolor{olive!22} & $\textbf{0.0257}$\cellcolor{olive!30} & $0.1445$\cellcolor{olive!0} & $0.0340$\cellcolor{olive!27} & $0.2154$\cellcolor{olive!8} & $0.2776$\cellcolor{olive!2} & $0.1621$\cellcolor{olive!14} & $0.0443$\cellcolor{olive!26} & $\textbf{0.0104}$\cellcolor{olive!30} & $0.3019$\cellcolor{olive!0} & $0.0164$\cellcolor{olive!29} & 0.250 \\
 & \texttt{molecular} & $0.0188$\cellcolor{olive!18} & $0.0388$\cellcolor{olive!0} & $0.0089$\cellcolor{olive!27} & $0.0071$\cellcolor{olive!29} & $0.0072$\cellcolor{olive!29} & $\textbf{0.0064}$\cellcolor{olive!30} & $0.0068$\cellcolor{olive!29} & $0.0066$\cellcolor{olive!24} & $0.0294$\cellcolor{olive!0} & $0.0014$\cellcolor{olive!29} & $0.0013$\cellcolor{olive!29} & $0.0011$\cellcolor{olive!29} & $\textbf{0.0010}$\cellcolor{olive!30} & $0.0011$\cellcolor{olive!29} & 0.500 \\
 & \texttt{nursery} & $0.0179$\cellcolor{olive!12} & $0.0265$\cellcolor{olive!0} & $\textbf{0.0050}$\cellcolor{olive!30} & $^\dag0.0051$\cellcolor{olive!29} & $0.0061$\cellcolor{olive!28} & $0.0096$\cellcolor{olive!23} & $0.0069$\cellcolor{olive!27} & $0.1206$\cellcolor{olive!18} & $0.2875$\cellcolor{olive!0} & $\textbf{0.0139}$\cellcolor{olive!30} & $0.0256$\cellcolor{olive!28} & $0.0703$\cellcolor{olive!23} & $0.1424$\cellcolor{olive!15} & $0.0474$\cellcolor{olive!26} & 0.001 \\
 & \texttt{obesity} & $0.0070$\cellcolor{olive!16} & $0.0072$\cellcolor{olive!15} & $0.0095$\cellcolor{olive!0} & $\textbf{0.0050}$\cellcolor{olive!30} & $0.0084$\cellcolor{olive!6} & $0.0074$\cellcolor{olive!13} & $0.0085$\cellcolor{olive!6} & $0.0059$\cellcolor{olive!17} & $0.0065$\cellcolor{olive!15} & $0.0102$\cellcolor{olive!0} & $\textbf{0.0029}$\cellcolor{olive!30} & $0.0092$\cellcolor{olive!3} & $0.0059$\cellcolor{olive!17} & $0.0093$\cellcolor{olive!3} & 0.900 \\
 & \texttt{page\_block} & $\textbf{0.0141}$\cellcolor{olive!30} & $0.0271$\cellcolor{olive!0} & $0.0217$\cellcolor{olive!12} & $0.0170$\cellcolor{olive!23} & $0.0194$\cellcolor{olive!17} & $0.0163$\cellcolor{olive!24} & $0.0215$\cellcolor{olive!12} & $\textbf{0.3355}$\cellcolor{olive!30} & $1.4168$\cellcolor{olive!2} & $1.4289$\cellcolor{olive!2} & $0.5930$\cellcolor{olive!23} & $1.1913$\cellcolor{olive!8} & $0.5861$\cellcolor{olive!23} & $1.5181$\cellcolor{olive!0} & 0.999 \\
 & \texttt{phishing} & $0.0953$\cellcolor{olive!3} & $0.1062$\cellcolor{olive!0} & $0.0553$\cellcolor{olive!18} & $0.0281$\cellcolor{olive!27} & $\textbf{0.0221}$\cellcolor{olive!30} & $0.0549$\cellcolor{olive!18} & $0.0243$\cellcolor{olive!29} & $1.3672$\cellcolor{olive!1} & $1.4128$\cellcolor{olive!0} & $0.4169$\cellcolor{olive!22} & $0.1092$\cellcolor{olive!29} & $\textbf{0.0717}$\cellcolor{olive!30} & $0.5383$\cellcolor{olive!19} & $0.0872$\cellcolor{olive!29} & 0.999 \\
 & \texttt{satellite} & $0.0231$\cellcolor{olive!9} & $0.0312$\cellcolor{olive!0} & $0.0111$\cellcolor{olive!24} & $0.0076$\cellcolor{olive!28} & $0.0070$\cellcolor{olive!29} & $\textbf{0.0066}$\cellcolor{olive!30} & $0.0078$\cellcolor{olive!28} & $0.0767$\cellcolor{olive!14} & $0.1463$\cellcolor{olive!0} & $0.0194$\cellcolor{olive!27} & $0.0077$\cellcolor{olive!29} & $^\dag0.0059$\cellcolor{olive!29} & $\textbf{0.0055}$\cellcolor{olive!30} & $0.0078$\cellcolor{olive!29} & 0.250 \\
 & \texttt{shuttle} & $0.0073$\cellcolor{olive!8} & $0.0093$\cellcolor{olive!0} & $0.0021$\cellcolor{olive!28} & $0.0054$\cellcolor{olive!15} & $0.0019$\cellcolor{olive!29} & $0.0020$\cellcolor{olive!29} & $\textbf{0.0018}$\cellcolor{olive!30} & $0.0028$\cellcolor{olive!11} & $0.0043$\cellcolor{olive!0} & $0.0004$\cellcolor{olive!29} & $0.0014$\cellcolor{olive!21} & $0.0003$\cellcolor{olive!29} & $0.0004$\cellcolor{olive!29} & $\textbf{0.0002}$\cellcolor{olive!30} & 0.900 \\
 & \texttt{waveform-v1} & $0.0416$\cellcolor{olive!9} & $0.0582$\cellcolor{olive!0} & $0.0250$\cellcolor{olive!19} & $^\dag0.0070$\cellcolor{olive!29} & $0.0102$\cellcolor{olive!28} & $\textbf{0.0069}$\cellcolor{olive!30} & $0.0073$\cellcolor{olive!29} & $0.0234$\cellcolor{olive!14} & $0.0436$\cellcolor{olive!0} & $0.0094$\cellcolor{olive!23} & $\textbf{0.0007}$\cellcolor{olive!30} & $0.0015$\cellcolor{olive!29} & $^\dag0.0007$\cellcolor{olive!29} & $0.0008$\cellcolor{olive!29} & 0.001 \\
 & \texttt{wine-quality} & $0.0998$\cellcolor{olive!19} & $0.1144$\cellcolor{olive!13} & $0.0879$\cellcolor{olive!23} & $\textbf{0.0714}$\cellcolor{olive!30} & $0.1328$\cellcolor{olive!6} & $0.1511$\cellcolor{olive!0} & $^\dag0.0729$\cellcolor{olive!29} & $5.4520$\cellcolor{olive!20} & $7.6107$\cellcolor{olive!14} & $2.9600$\cellcolor{olive!26} & $\textbf{1.6965}$\cellcolor{olive!30} & $9.9969$\cellcolor{olive!8} & $13.4015$\cellcolor{olive!0} & $3.8726$\cellcolor{olive!24} & 0.250 \\
 & \texttt{yeast} & $0.0539$\cellcolor{olive!13} & $0.0557$\cellcolor{olive!11} & $0.0561$\cellcolor{olive!10} & $0.0647$\cellcolor{olive!0} & $0.0449$\cellcolor{olive!24} & $\textbf{0.0409}$\cellcolor{olive!30} & $0.0437$\cellcolor{olive!26} & $7.5748$\cellcolor{olive!14} & $6.3074$\cellcolor{olive!19} & $6.7092$\cellcolor{olive!17} & $11.1195$\cellcolor{olive!0} & $5.8082$\cellcolor{olive!20} & $\textbf{3.5281}$\cellcolor{olive!30} & $5.0089$\cellcolor{olive!24} & 0.999 \\
\midrule
\multirow{15}{*}{\begin{sideways}text\end{sideways}} & \texttt{gasp} & $0.0899$\cellcolor{olive!11} & $0.1199$\cellcolor{olive!0} & $0.0518$\cellcolor{olive!25} & $0.0415$\cellcolor{olive!29} & $0.0409$\cellcolor{olive!29} & $\textbf{0.0393}$\cellcolor{olive!30} & $^\dag0.0397$\cellcolor{olive!29} & $0.6271$\cellcolor{olive!16} & $1.2881$\cellcolor{olive!0} & $0.1704$\cellcolor{olive!28} & $^\dag0.1001$\cellcolor{olive!29} & $0.1149$\cellcolor{olive!29} & $\textbf{0.0973}$\cellcolor{olive!30} & $^\dag0.1011$\cellcolor{olive!29} & 0.250 \\
 & \texttt{hcr} & $0.1265$\cellcolor{olive!9} & $0.1524$\cellcolor{olive!0} & $0.0962$\cellcolor{olive!21} & $0.0811$\cellcolor{olive!27} & $0.0814$\cellcolor{olive!26} & $\textbf{0.0731}$\cellcolor{olive!30} & $0.0793$\cellcolor{olive!27} & $0.3777$\cellcolor{olive!9} & $0.4836$\cellcolor{olive!0} & $0.2078$\cellcolor{olive!23} & $0.1610$\cellcolor{olive!28} & $0.1737$\cellcolor{olive!26} & $\textbf{0.1382}$\cellcolor{olive!30} & $0.1635$\cellcolor{olive!27} & 0.500 \\
 & \texttt{omd} & $0.1062$\cellcolor{olive!6} & $0.1182$\cellcolor{olive!0} & $0.0901$\cellcolor{olive!14} & $0.0688$\cellcolor{olive!26} & $^\dag0.0620$\cellcolor{olive!29} & $^\dag0.0630$\cellcolor{olive!29} & $\textbf{0.0618}$\cellcolor{olive!30} & $0.1844$\cellcolor{olive!9} & $0.2356$\cellcolor{olive!0} & $0.1412$\cellcolor{olive!17} & $0.0876$\cellcolor{olive!27} & $\textbf{0.0715}$\cellcolor{olive!30} & $^\dag0.0720$\cellcolor{olive!29} & $^\dag0.0745$\cellcolor{olive!29} & 0.500 \\
 & \texttt{sanders} & $0.1114$\cellcolor{olive!7} & $0.1341$\cellcolor{olive!0} & $0.0587$\cellcolor{olive!25} & $^\dag0.0470$\cellcolor{olive!29} & $\textbf{0.0459}$\cellcolor{olive!30} & $^\dag0.0464$\cellcolor{olive!29} & $^\dag0.0463$\cellcolor{olive!29} & $0.5178$\cellcolor{olive!7} & $0.6633$\cellcolor{olive!0} & $0.1318$\cellcolor{olive!27} & $0.0930$\cellcolor{olive!29} & $\textbf{0.0839}$\cellcolor{olive!30} & $^\dag0.0855$\cellcolor{olive!29} & $0.0892$\cellcolor{olive!29} & 0.900 \\
 & \texttt{semeval13} & $0.1115$\cellcolor{olive!9} & $0.1365$\cellcolor{olive!0} & $0.0758$\cellcolor{olive!23} & $0.0739$\cellcolor{olive!23} & $\textbf{0.0580}$\cellcolor{olive!30} & $0.0597$\cellcolor{olive!29} & $^\dag0.0589$\cellcolor{olive!29} & $0.2692$\cellcolor{olive!16} & $0.5073$\cellcolor{olive!0} & $0.1111$\cellcolor{olive!26} & $0.1040$\cellcolor{olive!27} & $\textbf{0.0642}$\cellcolor{olive!30} & $0.0676$\cellcolor{olive!29} & $^\dag0.0658$\cellcolor{olive!29} & 0.001 \\
 & \texttt{semeval14} & $0.1038$\cellcolor{olive!11} & $0.1332$\cellcolor{olive!0} & $0.0668$\cellcolor{olive!25} & $0.0656$\cellcolor{olive!26} & $\textbf{0.0560}$\cellcolor{olive!30} & $0.0576$\cellcolor{olive!29} & $^\dag0.0568$\cellcolor{olive!29} & $0.2100$\cellcolor{olive!19} & $0.4723$\cellcolor{olive!0} & $0.0967$\cellcolor{olive!28} & $0.1070$\cellcolor{olive!27} & $\textbf{0.0776}$\cellcolor{olive!30} & $0.0834$\cellcolor{olive!29} & $0.0818$\cellcolor{olive!29} & 0.001 \\
 & \texttt{semeval15} & $0.1253$\cellcolor{olive!9} & $0.1425$\cellcolor{olive!0} & $0.0958$\cellcolor{olive!24} & $^\dag0.0865$\cellcolor{olive!29} & $\textbf{0.0863}$\cellcolor{olive!30} & $0.0915$\cellcolor{olive!27} & $0.0913$\cellcolor{olive!27} & $0.2880$\cellcolor{olive!14} & $0.4533$\cellcolor{olive!0} & $0.1585$\cellcolor{olive!25} & $\textbf{0.1124}$\cellcolor{olive!30} & $0.1520$\cellcolor{olive!26} & $0.1734$\cellcolor{olive!24} & $0.1715$\cellcolor{olive!24} & 0.001 \\
 & \texttt{semeval16} & $0.1389$\cellcolor{olive!19} & $0.1612$\cellcolor{olive!10} & $0.1891$\cellcolor{olive!0} & $\textbf{0.1110}$\cellcolor{olive!30} & $0.1370$\cellcolor{olive!20} & $0.1354$\cellcolor{olive!20} & $0.1390$\cellcolor{olive!19} & $0.5124$\cellcolor{olive!26} & $0.6856$\cellcolor{olive!19} & $1.1685$\cellcolor{olive!0} & $\textbf{0.4215}$\cellcolor{olive!30} & $0.6520$\cellcolor{olive!20} & $0.6326$\cellcolor{olive!21} & $0.6749$\cellcolor{olive!19} & 0.500 \\
 & \texttt{sst} & $0.1051$\cellcolor{olive!11} & $0.1352$\cellcolor{olive!0} & $0.0749$\cellcolor{olive!22} & $^\dag0.0539$\cellcolor{olive!29} & $0.0541$\cellcolor{olive!29} & $0.0575$\cellcolor{olive!28} & $\textbf{0.0538}$\cellcolor{olive!30} & $0.1687$\cellcolor{olive!13} & $0.2659$\cellcolor{olive!0} & $0.0904$\cellcolor{olive!24} & $^\dag0.0487$\cellcolor{olive!29} & $0.0490$\cellcolor{olive!29} & $0.0560$\cellcolor{olive!28} & $\textbf{0.0484}$\cellcolor{olive!30} & 0.999 \\
 & \texttt{wa} & $0.0783$\cellcolor{olive!2} & $0.0819$\cellcolor{olive!0} & $0.0517$\cellcolor{olive!22} & $0.0500$\cellcolor{olive!23} & $^\dag0.0425$\cellcolor{olive!29} & $\textbf{0.0412}$\cellcolor{olive!30} & $0.0432$\cellcolor{olive!28} & $0.0873$\cellcolor{olive!2} & $0.0938$\cellcolor{olive!0} & $0.0393$\cellcolor{olive!24} & $0.0370$\cellcolor{olive!25} & $^\dag0.0281$\cellcolor{olive!29} & $\textbf{0.0264}$\cellcolor{olive!30} & $0.0297$\cellcolor{olive!28} & 0.001 \\
 & \texttt{wb} & $0.0804$\cellcolor{olive!0} & $0.0799$\cellcolor{olive!0} & $0.0439$\cellcolor{olive!23} & $\textbf{0.0333}$\cellcolor{olive!30} & $0.0343$\cellcolor{olive!29} & $0.0348$\cellcolor{olive!29} & $0.0346$\cellcolor{olive!29} & $0.1001$\cellcolor{olive!0} & $0.0990$\cellcolor{olive!0} & $0.0316$\cellcolor{olive!25} & $\textbf{0.0183}$\cellcolor{olive!30} & $0.0190$\cellcolor{olive!29} & $0.0197$\cellcolor{olive!29} & $0.0194$\cellcolor{olive!29} & 0.250 \\
 & \texttt{LeQua2022-T1A} & $0.0933$\cellcolor{olive!7} & $0.1181$\cellcolor{olive!0} & $0.0314$\cellcolor{olive!27} & $^\dag0.0240$\cellcolor{olive!29} & $0.0242$\cellcolor{olive!29} & $^\dag0.0252$\cellcolor{olive!29} & $\textbf{0.0234}$\cellcolor{olive!30} & $0.1260$\cellcolor{olive!11} & $0.2024$\cellcolor{olive!0} & $0.0170$\cellcolor{olive!28} & $^\dag0.0099$\cellcolor{olive!29} & $^\dag0.0097$\cellcolor{olive!29} & $^\dag0.0114$\cellcolor{olive!29} & $\textbf{0.0093}$\cellcolor{olive!30} & 0.001 \\
 & \texttt{LeQua2022-T1B} & $0.0140$\cellcolor{olive!17} & $0.0165$\cellcolor{olive!3} & $0.0130$\cellcolor{olive!23} & $^\dag0.0119$\cellcolor{olive!29} & $0.0118$\cellcolor{olive!29} & $0.0171$\cellcolor{olive!0} & $\textbf{0.0118}$\cellcolor{olive!30} & $\textbf{91.1316}$\cellcolor{olive!30} & $118.7259$\cellcolor{olive!13} & $121.7791$\cellcolor{olive!11} & $92.8496$\cellcolor{olive!28} & $109.3588$\cellcolor{olive!19} & $141.3354$\cellcolor{olive!0} & $109.3062$\cellcolor{olive!19} & 0.999 \\
 & \texttt{LeQua2024-T1} & $0.0862$\cellcolor{olive!9} & $0.1167$\cellcolor{olive!0} & $0.0270$\cellcolor{olive!27} & $^\dag0.0205$\cellcolor{olive!29} & $\textbf{0.0204}$\cellcolor{olive!30} & $0.0232$\cellcolor{olive!29} & $^\dag0.0206$\cellcolor{olive!29} & $0.1114$\cellcolor{olive!13} & $0.1978$\cellcolor{olive!0} & $0.0130$\cellcolor{olive!29} & $^\dag0.0078$\cellcolor{olive!29} & $\textbf{0.0078}$\cellcolor{olive!30} & $0.0103$\cellcolor{olive!29} & $^\dag0.0079$\cellcolor{olive!29} & 0.250 \\
 & \texttt{LeQua2024-T2} & $0.0171$\cellcolor{olive!21} & $0.0190$\cellcolor{olive!16} & $0.0170$\cellcolor{olive!21} & $\textbf{0.0137}$\cellcolor{olive!30} & $0.0160$\cellcolor{olive!24} & $0.0256$\cellcolor{olive!0} & $0.0156$\cellcolor{olive!25} & $137.5911$\cellcolor{olive!24} & $163.6366$\cellcolor{olive!21} & $231.3429$\cellcolor{olive!12} & $\textbf{100.8582}$\cellcolor{olive!30} & $186.1994$\cellcolor{olive!18} & $318.9042$\cellcolor{olive!0} & $196.3228$\cellcolor{olive!16} & 0.900 \\
\midrule
\multirow{5}{*}{\begin{sideways}image\end{sideways}} & \texttt{CIFAR100} & $0.0025$\cellcolor{olive!8} & $0.0028$\cellcolor{olive!0} & $0.0022$\cellcolor{olive!15} & $\textbf{0.0016}$\cellcolor{olive!30} & $0.0018$\cellcolor{olive!25} & $0.0021$\cellcolor{olive!18} & $0.0017$\cellcolor{olive!28} & $0.1375$\cellcolor{olive!8} & $0.1695$\cellcolor{olive!0} & $0.1009$\cellcolor{olive!18} & $\textbf{0.0576}$\cellcolor{olive!30} & $0.0712$\cellcolor{olive!26} & $0.1118$\cellcolor{olive!15} & $0.0643$\cellcolor{olive!28} & 0.250 \\
 & \texttt{CIFAR10} & $0.0065$\cellcolor{olive!9} & $0.0080$\cellcolor{olive!0} & $0.0044$\cellcolor{olive!23} & $0.0036$\cellcolor{olive!29} & $0.0038$\cellcolor{olive!27} & $\textbf{0.0035}$\cellcolor{olive!30} & $0.0036$\cellcolor{olive!29} & $0.0092$\cellcolor{olive!10} & $0.0127$\cellcolor{olive!0} & $0.0037$\cellcolor{olive!26} & $0.0025$\cellcolor{olive!29} & $0.0027$\cellcolor{olive!28} & $\textbf{0.0023}$\cellcolor{olive!30} & $0.0024$\cellcolor{olive!29} & 0.001 \\
 & \texttt{MNIST} & $0.0016$\cellcolor{olive!13} & $0.0018$\cellcolor{olive!7} & $0.0017$\cellcolor{olive!10} & $0.0012$\cellcolor{olive!28} & $0.0012$\cellcolor{olive!27} & $0.0020$\cellcolor{olive!0} & $\textbf{0.0011}$\cellcolor{olive!30} & $0.0007$\cellcolor{olive!3} & $0.0006$\cellcolor{olive!7} & $0.0005$\cellcolor{olive!13} & $0.0003$\cellcolor{olive!29} & $0.0003$\cellcolor{olive!26} & $0.0008$\cellcolor{olive!0} & $\textbf{0.0003}$\cellcolor{olive!30} & 0.001 \\
 & \texttt{fashionMNIST} & $0.0083$\cellcolor{olive!9} & $0.0103$\cellcolor{olive!0} & $0.0047$\cellcolor{olive!25} & $0.0041$\cellcolor{olive!28} & $\textbf{0.0038}$\cellcolor{olive!30} & $^\dag0.0039$\cellcolor{olive!29} & $0.0038$\cellcolor{olive!29} & $0.0201$\cellcolor{olive!9} & $0.0274$\cellcolor{olive!0} & $0.0052$\cellcolor{olive!27} & $0.0041$\cellcolor{olive!29} & $\textbf{0.0036}$\cellcolor{olive!30} & $^\dag0.0036$\cellcolor{olive!29} & $0.0036$\cellcolor{olive!29} & 0.999 \\
 & \texttt{SVHN} & $0.0098$\cellcolor{olive!3} & $0.0106$\cellcolor{olive!0} & $0.0061$\cellcolor{olive!22} & $0.0053$\cellcolor{olive!26} & $0.0049$\cellcolor{olive!28} & $0.0046$\cellcolor{olive!29} & $\textbf{0.0045}$\cellcolor{olive!30} & $0.0235$\cellcolor{olive!6} & $0.0284$\cellcolor{olive!0} & $0.0091$\cellcolor{olive!24} & $0.0072$\cellcolor{olive!27} & $0.0060$\cellcolor{olive!28} & $\textbf{0.0051}$\cellcolor{olive!30} & $^\dag0.0052$\cellcolor{olive!29} & 0.001 \\
\midrule
 & Rank & $5.4$\cellcolor{olive!7} & $6.5$\cellcolor{olive!0} & $5.0$\cellcolor{olive!10} & $3.1$\cellcolor{olive!23} & $2.6$\cellcolor{olive!26} & $3.2$\cellcolor{olive!22} & $\textbf{2.2}$\cellcolor{olive!30} \ & $5.3$\cellcolor{olive!8} & $6.3$\cellcolor{olive!0} & $4.8$\cellcolor{olive!12} & $3.0$\cellcolor{olive!26} & $2.7$\cellcolor{olive!29} & $3.3$\cellcolor{olive!24} & $\textbf{2.6}$\cellcolor{olive!30} \ & ---  \\
\bottomrule
\end{tabular}

    }%
\end{table}

\end{document}